%% file: main.tex
\pdfoutput=1
\documentclass{article}

\usepackage[nonatbib, main, final]{neurips_2026}

\usepackage[utf8]{inputenc} %
\usepackage[T1]{fontenc}    %
\usepackage{amsfonts}
\usepackage{amsmath}
\usepackage{mathtools}
\usepackage{amsbsy}
\usepackage{amssymb}
\usepackage[dvipsnames, table]{xcolor}
\usepackage{graphicx}
\usepackage[numbers,sort,compress]{natbib}
\usepackage{enumerate}
\usepackage{caption}
\usepackage{xurl}
\usepackage[breaklinks=true]{hyperref}
\usepackage{bm}
\usepackage{wrapfig}

\usepackage{algorithm}
\usepackage{algpseudocode}
\usepackage{enumitem}
\usepackage{booktabs}
\usepackage{subcaption}
\usepackage{tikz}
\usetikzlibrary{shapes,decorations,arrows,calc,arrows.meta,fit,positioning, backgrounds}
\tikzset{
    -Latex,auto,node distance =1 cm and 1 cm,semithick,
    state/.style ={ellipse, draw, minimum width = 0.7 cm},
    point/.style = {circle, draw, inner sep=0.04cm,fill,node contents={}},
    bidirected/.style={Latex-Latex,dashed},
    el/.style = {inner sep=2pt, align=left, sloped}
}
\usepackage{chemformula}
\usepackage{tablefootnote}
\usepackage{adjustbox}
\usepackage{multirow}
\usepackage{tocloft}

\usepackage{amsthm}
\newtheorem{theorem}{Theorem}
\newtheorem{lemma}[theorem]{Lemma}

\newtheorem{proposition}[theorem]{Proposition}
\newtheorem{definition}[theorem]{Definition}
\newtheorem{example}[theorem]{Example}

\newtheorem{remark}[theorem]{Remark}
\newtheorem{assumption}{Assumption}

\usepackage{cleveref}
\crefname{figure}{Fig.}{Figs.}
\crefname{definition}{Defn.}{Defns.}
\crefname{corollary}{Corollary}{Corollaries}
\crefname{proposition}{Prop.}{Props.}
\crefname{theorem}{Thm.}{Thms.}
\crefname{remark}{Remark}{Remarks}
\crefname{lemma}{Lemma}{Lemmata}
\crefname{claim}{Claim}{Claims}
\crefname{table}{Table}{Tables}
\crefname{section}{\S}{\S\S}
\crefname{subsection}{\S}{\S\S}
\crefname{subsubsection}{\S}{\S\S}
\crefname{assumption}{Assm.}{Assms.}
\crefname{appendix}{Appx.}{Appxs.}
\crefname{equation}{Eq.}{Eqs.}
\crefname{example}{Ex.}{Exs.}

\newcommand{\appendixtocname}{Appendix contents}
\newlistof{appendices}{app}{\appendixtocname}

\newcommand{\appsection}[1]{%
  \section{#1}%
  \addcontentsline{app}{section}{\protect\numberline{\thesection}#1}%
}

\definecolor{mygray}{gray}{0.5}
\definecolor{myblue}{RGB}{87,144,252}
\definecolor{myorange}{RGB}{248,156,32}
\definecolor{mygreen}{rgb}{0.133, 0.545, 0.133}
\definecolor{mypurple}{rgb}{0.58, 0.34, 0.92}

\newcommand{\Env}{\mathcal{E}}
\newcommand{\Envtr}{\mathcal{E}_\mathrm{tr}}

\newcommand{\Rlead}{R^{\mathrm{lead}}}
\newcommand{\Vfoll}{V^{\mathrm{foll}}}

\newcommand{\G}{\mathcal{G}}

\newcommand{\C}{\mathbb{C}}

\newcommand{\DEYforb}{\DE(Y)_\mathrm{forb}}
\newcommand{\SB}{\mathrm{SB}}
\newcommand{\MB}{\mathrm{MB}}
\DeclareMathOperator{\CH}{CH}
\DeclareMathOperator{\PA}{PA}

\DeclareMathOperator{\ND}{ND}
\DeclareMathOperator{\DE}{DE}
\DeclareMathOperator{\DIS}{DIS}
\newcommand{\SBg}[1]{\mathrm{SB}_{#1}(Y)}
\newcommand{\DEforbg}[1]{\mathrm{DE}_{#1}(Y)_{\mathrm{forb}}}

\newcommand{\Q}{\mathbb{Q}}
\renewcommand{\P}{\mathbb{P}}
\newcommand{\E}{\mathbb{E}}
\DeclareMathOperator*{\argmin}{argmin}
\DeclareMathOperator*{\argmax}{argmax}
\DeclareMathOperator{\Var}{Var}
\DeclareMathOperator{\Cov}{Cov}
\newcommand{\indep}{\mbox{${}\perp\mkern-9mu\perp{}$}}
\newcommand{\notindep}{\mbox{${}\not\!\perp\mkern-9mu\perp{}$}}
\newcommand{\given}{\,|\,}
\newcommand{\dsep}[1]{\mathrel{\perp_{#1}}}
\newcommand{\msep}{\mathrel{\perp_{m}}}

\title{
Prediction-Intervention Games and Invariant Sets
}

\author{%
    Linus K\"uhne \\
    Seminar for Statistics and ETH AI Center\\
    ETH Zurich\\
    Zurich, Switzerland \\
    \texttt{linus.kuehne@ai.ethz.ch} \\
    \And
    Felix Schur \\
    Seminar for Statistics\\
    ETH Zurich\\
    Zurich, Switzerland \\
    \texttt{felix.schur@stat.math.ethz.ch} \\
    \And
    Jonas Peters \\
    Seminar for Statistics\\
    ETH Zurich\\
    Zurich, Switzerland \\
    \texttt{jonas.peters@stat.math.ethz.ch}
  }

\begin{document}

\maketitle

\begin{abstract}%

We consider the following two-player game: using observational data, the leader chooses a prediction function for a response variable $Y$ from given covariates. The follower then reacts with an intervention on some covariates in the underlying structural causal model to maximize their own objective. The leader knows the intervention targets, but may have limited knowledge of the follower's objective. We call this setup a prediction-intervention game, a special case of a Stackelberg game. Finding an optimal strategy for the leader is generally difficult. To avoid severe performance loss, the leader may base their prediction on the causal parents of $Y$, or more generally on an invariant subset of covariates. We prove,
for two common classes of follower objectives,
that predictors based on the stable blanket, an
invariant subset, 
are 
at least as
good as
those based on the causal parents.
We further upper bound the leader's post-intervention risk by a worst-case risk over allowed interventions and strengthen existing distribution generalization results to analyze this bound: we give sufficient conditions under which stable-blanket predictors are worst-case optimal, and show by examples that these conditions cannot in general be dropped. Finally, we discuss practical strategies %
for settings with known and unknown graph, and test them on simulated and real-world data.
\end{abstract}

\input{body}

\end{document}

%% file: body.tex
\definecolor{softsage}{RGB}{190,214,224}
\definecolor{softrose}{RGB}{232,199,190}

\section{Introduction} \label{sec:intro}

We consider a two-player Stackelberg game for predicting a real-valued or binary
response $Y \in \mathcal{Y}$ from covariates $X \in \mathcal{X}$: 
based on  a training distribution (or data),
a leader commits to a prediction function $f$, and a follower (an individual subject to the prediction) observes or queries $f$ and intervenes on some covariates to optimize their own, possibly non-adversarial, objective. The test distribution thus emerges only after the
prediction function has been deployed, and can depend on that
function. In particular, training and test distributions need not coincide.

\paragraph{Example.}
Consider an insurer who prices policies with a function $f$ predicting 
the probability that an individual will submit a claim.
Suppose that $f$ uses two covariates that are correlated with $Y$,
but only one is causal: the car's number of safety features
and the car's color.
An individual $i$ with access to $f$ (such as through querying a website for a quote)
may change their covariates from $X_i$ to $\tilde{X}_i$ (yielding  $\tilde{Y}_i$) to minimize
$f(\tilde{X}_i)$ and obtain a smaller premium. Installing safety features changes both the predicted and true probability of filing a claim: $f(\tilde{X}_i) \neq f({X_i})$ and $\P(\tilde{Y}_i = 1) \neq \P(Y_i = 1)$. 
Choosing a car with a different color %
may instead yield $f(\tilde{X}_i) \neq f({X_i})$ 
without changing the probability of filing a claim: 
$\P(\tilde{Y}_i = 1) = \P(Y_i = 1)$.
A similar phenomenon occurs in usage-based car insurance, where \citet{leeRoleMonitoringEffect2025} show that temporarily monitoring drivers to price policies changes their behavior during that period, worsening prediction performance.

\paragraph{Prediction-intervention games.}
We formalize the problem as a causal version of a Stackelberg game \citep{stackelbergMarktformGleichgewicht1934}, which we call a prediction-intervention game.
The leader provides a prediction
function~$f$, and the follower reacts by choosing an intervention $e^*(f)$
in a structural causal model (SCM) that maximizes their objective:
$e^*(f) \in \argmax_{e} \Vfoll_e(f)$. We model the allowed intervention targets as the children of an intervention node in the causal graph. Knowing the causal structure helps the leader design a prediction function
that is robust under the follower's interventions, that is, find $\argmin_{f} \Rlead_{e^*(f)}(f)$. 
In practice, the leader estimates $\Rlead_{e}$ from data and may have only partial knowledge about the follower's objective and possible interventions.

\paragraph{Related work.}

\looseness=-1
Strategic classification \citep{hardtStrategicClassification2016, brucknerStackelbergGamesAdversarial2011} and its causal extensions \citep{miller2020strategic, shavitCausalStrategicLinear2022, horowitzCausalStrategicClassification2023} model followers changing their covariates at a cost to obtain favorable predictions.
Unlike strategic classification, the causal extensions and our formulation model how these changes affect $Y$, using an SCM.
In these causal extensions, the leader knows how the follower reacts to $f$, through a known cost function \citep{miller2020strategic, horowitzCausalStrategicClassification2023} or by deploying predictors and observing the resulting data \citep{shavitCausalStrategicLinear2022}.
Our guarantees do not need this knowledge: they hold for all interventions that $\G_0$ and $A$ allow for and for classes of follower objectives; the causal extensions provide no such guarantees.
Performative prediction \citep{perdomo2020performative} studies the phenomenon that deployed models shift the data distribution, without specifying what causes the shift. Our game is an instance of this broader framework in which we specify how the follower causes the shift, namely by intervening in an SCM.
In our game, the leader chooses $f$ without observing the follower's response, so the 
retraining procedures analyzed in performative prediction cannot be used; our specification lets us derive guarantees before deployment.
We adopt the leader-follower structure from Stackelberg games \citep{stackelbergMarktformGleichgewicht1934} but to the best of our knowledge, no methodology other than the one stated above is directly applicable to our setting.
Our focus on invariant subsets connects prediction-intervention games to causal distribution generalization \citep{petersCausalInferenceUsing2015, rojas-carullaInvariantModelsCausal2018, christiansenCausalFrameworkDistribution2021}; in particular, we analyze stable-blanket predictors
\citep{pfisterStabilizingVariableSelection2021}, whose previous optimality guarantee concerns average risk among invariant subsets.
We 
provide 
a tight bound, connecting our setting to worst-case distribution generalization 
\citep{Kuhn_Shafiee_Wiesemann_2025, sagawa_groupDRO, freni2025maximum}, which specifies test distributions without a causal model.
We also give conditions for the worst-case optimality of the stable-blanket predictor among all measurable predictors, over a family of distributions that, unlike in \citet[][Thm.~4]{rojas-carullaInvariantModelsCausal2018}, does not depend on the invariant subset.
The stable-blanket predictor belongs to the family considered by existing methods \citep{subbaswamyPreventingFailuresDue2019, leeFindingInvariantPredictors2023}, which rely on conditional independencies in interventional distributions \citep{shpitserDormantIndependence2008}.
Their selected predictor need not be worst-case optimal in our setting.
\Cref{app:extended_related_works} gives further details on %
existing work, e.g., on performative prediction, %
and lists additional work that we consider more distant than the above.

\paragraph{Contributions.}
(i) We formally introduce prediction-intervention games and show how the
intervention view connects to modeling a group of self-optimizing individuals
(\Cref{sec:problem_setup}).
(ii) 
We prove that, for any fixed intervention, stable-blanket prediction is optimal among predictors that do not use a specific set of descendants of $Y$ (namely, $\DEYforb$, see~\cref{sec:stable_blanket}), and show that, for two common classes of follower objectives, it weakly dominates parent-based prediction for the leader's game objective
(\Cref{sec:strategiesinvsets}).
(iii) The worst-case risk over the possible interventions upper-bounds the leader's post-intervention risk and does not require knowledge of the follower's objective. This connects to
distribution generalization and worst-case prediction. We strengthen existing
results to give sufficient conditions under which predictors based on the stable blanket 
are worst-case optimal among all measurable predictors, 
and 
prove
that in general these conditions cannot  be dropped (\Cref{sec:optimality}).
(iv) We propose 
learning methods that the leader can use when they do
not know the causal graph (\Cref{sec:learning_subset_predictors}).
(v) We test our methodology on simulated and real data (\Cref{sec:experiments}). The source code to reproduce all experiments is available under \href{https://github.com/LinusKuehne/prediction-intervention-games}{\texttt{github.com/LinusKuehne/prediction-intervention-games}}.

\section{
Prediction-intervention games
} \label{sec:problem_setup}

\paragraph{The game.}
We assume that the leader's task is to perform prediction; in this paper, we
consider both regression and binary classification.
Let $Y \in \mathcal{Y}$ be a binary or continuous response with covariates $X = (X_1, \ldots, X_d) \in \mathcal{X}$.
We use $\mathcal{Y} \coloneqq \{0,1\}$ and $\mathcal{A} \coloneqq [0,1]$ for classification, and $\mathcal{Y} \coloneqq \mathbb{R}$ and $\mathcal{A} \coloneqq \mathbb{R}$ for regression. The training data are generated by a structural causal model
\citep[SCM;][]{pearlCausalityModelsReasoning2009, bongersFoundationsStructuralCausal2021}
$\C_0$ over $(X, Y, E)$ with induced distribution $\P^0$ and DAG $\G_0$, in which $E$ is a source node indexing environments. Each $e \in \Env_0 \coloneqq \mathrm{supp}_{\C_0}(E)$
corresponds to a distribution $\P^0_e \coloneq \P^0(\cdot \given E=e)$ over $(X, Y)$. The leader has access
to one or more training distributions $\P^0_e$ (or i.i.d.\ samples thereof) for $e \in \Envtr \subseteq \Env_0$, and chooses a
prediction function $f \colon \mathcal{X} \to \mathcal{A}$. We assume that
$\G_0$ is known unless stated otherwise. After observing $f$, the follower reacts by strategically choosing an
environment $e \in \Env$ that maximizes a utility $\Vfoll_e(f)$. Here,
$\Env$ 
is a set of
environments indexing distributions $\P_e$ over
$(X, Y)$ 
that can occur after deployment of $f$; 
these environments correspond to intervention distributions
and we even allow for 
$\Env \supsetneq \Env_0$
(below, we provide more details on the follower's interventions). After the follower has chosen $e \in \Env$, the leader incurs risk $\Rlead_e(f)$. This defines a prediction-intervention game:\footnote{Instead of the
supremum, one could also average $\Rlead_e(f)$ over the follower's best
responses $e \in \Env^*(f)$.}
\begin{equation} \label{eq:pred_intervention_game}
    f^* \in \argmin_{f \colon \mathcal{X} \to \mathcal{A}}\;
    \sup_{e \in \Env^*(f)} \Rlead_e(f),
    \qquad
    \Env^*(f) \coloneqq \argmax_{e \in \Env} \Vfoll_e(f).
\end{equation}
We assume throughout that $\Env^*(f) \neq \varnothing$ and that the follower can evaluate $\Vfoll_e(f)$ for all $e \in \Env$, for example by knowing $\C_0$ or by sampling those variables from each\footnote{An individual follower may also condition on covariates that cannot be changed, such as a birth year; see \Cref{sec:individuals}.
} $\P_e$ that are required to evaluate $\Vfoll_e(f)$, possibly not containing $Y$.
We detail $\Rlead$, $\Vfoll$, and the construction of $\Env$ below.

\paragraph{Notation.}
For a node $j$ in a DAG $\mathcal{H}$, let $\PA_{\mathcal{H}}(j), \CH_{\mathcal{H}}(j), \DE_{\mathcal{H}}(j), \ND_{\mathcal{H}}(j)$ denote its parents, children, descendants, and non-descendants (all excluding $j$), respectively. We drop the subscript if it is clear from the context. 
We sometimes write $k \to j$ as a shorthand for an edge $X_k \to X_j$. For
$S \subseteq \{1, \ldots, d, Y\}$, we define $X_S \coloneqq (X_j)_{j \in S}$ with 
$X_Y := Y$. 
Slightly abusing notation, 
$S^C$ is always defined as 
$\{1, \ldots, d\} \setminus S$.
Expectations under $\P^0_e$ and $\P_e$ are denoted by $\E^0_e$ and~$\E_e$.

\paragraph{The leader's evaluation loss.}
We assume that the leader evaluates a predictor $f \colon \mathcal{X} \to \mathcal{A}$
using a loss $\ell \colon \mathcal{Y} \times \mathcal{A} \to [0, \infty]$ strictly consistent for the mean functional \citep{gneitingMakingEvaluatingPoint2011}: for every distribution $F$ on $\mathcal{Y}$ with finite mean $\E_F[Y]$ and for which the expectations below are finite, and for all $q \in \mathcal{A}$, $\E_F[\ell(Y, \E_F[Y])] \leq \E_F[\ell(Y, q)]$, with equality iff $q = \E_F[Y]$. Examples include the squared loss $\ell(y, q) = (y-q)^2$ for regression and, for binary classification, any strictly proper scoring rule such as the
Brier score $\ell(y, q) = (y-q)^2$ or the
log loss \citep{gneitingStrictlyProperScoring2007}. Strict consistency ensures that the risk $\Rlead_e(f) \coloneqq \E_e[\ell(Y, f(X))]$ is uniquely minimized by $f(x) = \E_e[Y \given X = x]$ (thus, it plays a similar role to the orthogonality of the conditional expectation, but extends to other loss functions). We assume throughout that $Y$ has finite mean under each $\P_e$ and that all risks in our results are finite.

\paragraph{Follower's objective.}
The follower's objective need not coincide with the leader's: while the
leader evaluates $f$ by $\Rlead_e(f)$ and prefers low risk, the follower selects $e \in \Env$ to maximize a separate utility $\Vfoll_e(f) \in \mathbb{R}$. For instance, the leader may use $\Rlead_e(f) = \E_e[(Y-f(X))^2]$, while the follower maximizes an objective such as $\E_e[Y - f(X)]$ or $\E_e[f(X)]$.

\paragraph{Follower's interventions.}
The follower reacts to $f$ by strategically intervening on the children of
$E$ in $\C_0$, possibly adding new edges to $\G_0$. This yields a
deployment SCM $\C$ over $(X, Y, E)$ that extends $\C_0$, with induced
distribution $\P$ and DAG $\G$. 
Each $e \in \Env \coloneqq
\mathrm{supp}_{\C}(E)$ yields $\P_e \coloneqq \P(\cdot \given E = e)$
over $(X, Y)$. The source node $E$ indicates the potential intervention targets: an edge $E \to X_j$ in $\G_0$ means the follower may intervene on
the mechanism of $X_j$. We assume no edge $E \to Y$. Compared to training, for each $j \in \CH_{\G_0}(E)$, the follower may let $X_j$ depend on its parents $X_{\PA_{\G_0}(j)}$ in $\G_0$, which may include $Y$,
and on additional variables from an action set $A_j \subseteq \{1, \ldots, d\}$ of permitted new parents; we write
$A \coloneqq (A_1, \ldots, A_d)$.
Thus, for $k \in A_j$, the new mechanism of $X_j$ may use $X_k$ as an input (so that $\G$ contains the additional edge $X_k \to X_j$)
but does not alter the mechanism of $X_k$.
Formally, $\C$ is obtained from $\C_0$ by replacing, for each
$j \in \CH_{\G_0}(E)$, the structural assignment of $X_j$ with a
measurable function of $X_{\PA_{\G_0}(j) \cup S_j}$ for some
$S_j \subseteq A_j$, and independent exogenous noise.
Since the new assignment may be constant, hard interventions are included in this intervention model, as are randomized and soft interventions.\footnote{
The results in \Cref{sec:strategiesinvsets} and \Cref{sec:optimality}
concern population predictors that are conditional means of $Y$ given a subset of the covariates. If these are continuous functions and the support of $X$ under every $e \in \Env$ lies in the union of the training supports, the training distributions identify them. 
When predictors are learned from data, interventions can move the covariates outside the training support, where learning requires extrapolation assumptions that we do not address.
}
All other structural assignments are unchanged. The deployment DAG $\G$ is the DAG
of $\C$, which we assume to be acyclic. Choosing $S_j = \varnothing$ for
all $j$ together with the original assignments recovers $\C_0$, so
$\Envtr \subseteq \Env_0 \subseteq \Env$ and $\P_e = \P^0_e$ for
$e \in \Env_0$. Larger $A_j$ allow for more flexible interventions.
\Cref{fig:synthetic_dag} shows our running example of $\G_0$ and $\G$, with $A_1 = \varnothing$ and $A_4 = \{1, 3\}$.
\begin{figure}[t]
    \begin{minipage}[c]{0.34\textwidth}
        \centering
        \begin{tikzpicture}[
            >=Stealth, xscale=0.8, yscale=0.75,
            var/.style={draw, circle, minimum size=6.5mm, inner sep=0.5pt, font=\footnotesize},
            pa/.style={var, fill=softsage},
            sbadd/.style={var, fill=softrose},
            forb/.style={var, fill=gray!30},
            env/.style={draw, rectangle, rounded corners, minimum size=6.5mm, inner sep=1pt, font=\footnotesize},
            arr/.style={->, semithick},
            addarr/.style={->, semithick, dashed}
        ]
            \node[env]   (E)  at (1.5, 2.5) {$E$};
            \node[pa]    (X1) at (0.0, 1.6) {$X_1$};
            \node[forb]  (X4) at (3.0, 1.6) {$X_4$};
            \node[forb]  (X5) at (4.6, 1.6) {$X_5$};
            \node[pa]    (X2) at (1.5, 0.9) {$X_2$};
            \node[var]   (Y)  at (1.5, -0.5) {$Y$};
            \node[sbadd] (X3) at (3.0, -0.5) {$X_3$};
            \draw[arr] (E) -- (X1);
            \draw[arr] (E) -- (X4);
            \draw[arr] (X2) -- (X1);
            \draw[arr] (X1) -- (Y);
            \draw[arr] (X2) -- (Y);
            \draw[arr] (Y) -- (X3);
            \draw[arr] (Y) -- (X4);
            \draw[arr] (X2) -- (X4);
            \draw[arr] (X4) -- (X5);
            \draw[addarr] (X1) -- (X4);
            \draw[addarr] (X3) -- (X4);
        \end{tikzpicture}
    \end{minipage}\hfill
    \begin{minipage}[c]{0.62\textwidth}
        \caption{Running example. The solid edges form the training DAG $\G_0$, in which the children $X_1$ and $X_4$ of $E$ are the intervention targets.
        During training,
        $X_4$ is a function of $X_2$, $Y$, and noise; the follower may replace this assignment by a function of 
        $X_1$, $X_2$, $X_3$, $Y$, and new noise, which adds the dashed edges and yields the deployment DAG $\G$. Colors: parents $\PA_{\G}(Y)$ (blue), stable blanket $\SBg{\G}$ (blue and red), and forbidden descendants $\DEforbg{\G}$ (gray), defined in \Cref{sec:stable_blanket}. 
        }
        \label{fig:synthetic_dag}
    \end{minipage}
\end{figure}

\paragraph{Unknown graph.}
In \Cref{sec:learning_subset_predictors}, we drop the assumption that the graph
is known to the leader. In that case, 
we assume access to training data 
from several training environments $\Envtr \subseteq \Env_0$ with $|\Envtr| > 1$. 
For each $e \in \Envtr$, we observe $n_e$ i.i.d.\ draws $(X_i, Y_i)_{i=1}^{n_e} \sim \P_e^0$.

As far as we know, the formulation~\eqref{eq:pred_intervention_game} is novel and determining 
its value 
$\sup_{e \in \Env^*(f^*)} \Rlead_e(f^*)$ 
is relevant. \Cref{app:real_world_examples} gives two real-world examples.
Prediction-intervention games are special instances of bilevel optimization problems \citep{zhangIntroductionBilevelOptimization2024}, which are known to be hard to solve in general \citep{bolteGeometricComputationalHardness2026}.
In \Cref{sec:strategiesinvsets} and~\Cref{sec:optimality}
we 
make first steps by connecting the problem to invariant sets and distribution generalization, respectively, and strengthening and applying results from these fields.

\subsection{Individuals and population interventions} \label{sec:individuals}

The follower of \Cref{sec:problem_setup} chooses mechanisms for a whole population; we now show that 
the framework still applies
when considering
individual followers who each choose values for the intervention targets given their own realization of the observed variables. In 
this section,
all graph quantities are taken with respect to $\G_0$.
An individual follower may observe a realization of
$X_O$ for some
$O \subseteq \ND(\CH(E))$, the variables among $X_1, \ldots, X_d, Y$ that are non-descendants of every target
(in particular, $O \cap \CH(E) = \emptyset$),
and is allowed to
choose deterministic values 
for $\CH(E)$, taking into account the observation $X_O=x_O$. For simplicity, let us assume that for all $j \in \CH(E)$, we have 
$\PA(j)\setminus \{E\} \subseteq O \cup \CH(E)$, so the follower has access to all parents of $j$ when setting $j$'s value, and $O \setminus \PA(j) \subseteq \{1, \ldots, d\}$, that is, $Y$ is observed if and only if it is a parent of $j$.
Let us further assume that the follower maximizes a utility of the form 
$
\E_e[k(f,X,Y) \given X_O=x_O]$.
We now consider a solution 
$(c_j)_{j \in \CH(E)}$ maximizing the individual follower's objective 
\begin{equation} \label{eq:individual}
    \E_{(c_j)_{j \in \CH(E)}}[k(f,
    X,Y)\given X_O=x_O],
\end{equation}
where the leftmost subscript %
indicates that the  $X_{\CH(E)}$ are intervened on and set to $c_j$, $j \in \CH(E)$. 
The optimal constants may vary with $x_O$, so there are functions 
$(g_j)_{j \in \CH(E)}$
such that for all $x_O$, they maximize
$
\E_{(g_j)_{j \in \CH(E)}}[k(f,
X,Y) \given X_O=x_O]$;
here the subscript indicates that the structural assignments of $X_{\CH(E)}$ are 
set to the
$g_j$'s.
Thus, the same $(g_j)_{j \in \CH(E)}$ maximize %
$
\E_{(g_j)_{j \in \CH(E)}}[k(f,
X,Y)],
$
that is, they solve the problem of a population follower with
$A_j \coloneqq O \setminus \PA(j)$.\footnote{Since $\PA(j) \setminus \{E\} \subseteq O \cup \CH(E)$, any population strategy with these $A_j$ sets $X_{\CH(E)}$ as a function of $X_O$ and independent noise; given $X_O = x_O$, its objective is therefore an average over constant choices, which the $g_j$ maximize.}
Thus, taking the average over many individual followers optimizing~\eqref{eq:individual} has,
from the perspective of the leader and the game's value, the same effect as considering a single population follower in an augmented graph.
In \Cref{app:individual_followers_theory}, we relax the assumption of a single $O$ to 
different $O_j$ for $j \in \CH(E)$,
still satisfying 
$O_j \subseteq \ND(j)$.
This relaxation is relevant for
considering individual followers in \Cref{fig:synthetic_dag} and
the experiments in \Cref{sec:synthetic_experiments} and \Cref{app:individual_followers}.

\section{Strategies 
based on invariant sets
} 
\label{sec:strategiesinvsets}

A natural strategy is to restrict the predictor $f$ to a subset
of covariates $S \subseteq \{1, \ldots, d\}$.
We first discuss the common choice $S= \PA(Y)$, which is a special case of choosing $f$ based on an invariant set. 
Throughout this section, we work with the augmented SCM $\C$, DAG $\G$,
and test family $\Env$ from \Cref{sec:problem_setup}. Quantities such as $\PA(Y)$ and $\CH(Y)$
(see~\cref{sec:problem_setup}), and the sets $\DEYforb$ and $\SB(Y)$ defined below,
are computed in $\G$. \Cref{sec:augmentation} relates these to the training-stage DAG $\G_0$. In the remainder of the paper, we write $R_e$ for the leader's risk $\Rlead_e$.

\subsection{Strategies based on causal parents} \label{sec:parents}
Using predictors $f$ based on the causal parents of $Y$ 
is common in causal strategic prediction \citep{yanDiscoveringOptimalScoring2023}
and has the following attractive property:
if the follower's utility depends on the intervention only through the
prediction $f(X)$, then a parent-based predictor can be affected only
by interventions changing the distribution of $X_{\PA(Y)}$.
`Gaming'
in the sense of manipulating descendants of $Y$ to change the prediction of $f$ without inducing a change in $Y$ is therefore impossible. 
Restricting $f$ to $X_{\PA(Y)}$ 
can thus incentivize followers to genuinely improve their outcomes.

Predictors based on $\PA(Y)$ further satisfy a simple robustness guarantee. Since $Y \notin \CH(E)$, the conditional distribution $Y \given X_{\PA(Y)}$ is the same in all $e \in \Env$. We refer to such subsets as invariant (see Def.~\ref{def:inv_subset} below). Thus, 
$f_{\PA(Y)}(x) \coloneqq \E[Y \given X_{\PA(Y)} = x_{\PA(Y)}]$
is environment-independent. Hence, if we restrict $f$ to the parents, the follower can induce a covariate shift in $X_{\PA(Y)}$; the optimal regression function using $X_{\PA(Y)}$ does not change.
In the special additive-noise case $Y = g(X_{\PA(Y)}) + \varepsilon$ with $\varepsilon\indep X_{\PA(Y)}$ and $\E[\varepsilon]=0$, the follower's interventions are inconsequential for parent-based prediction under the squared loss: for all $e \in \Env$, $R_e(f_{\PA(Y)})=\Var(\varepsilon)$. Without this additive-noise structure, the risk may still vary across environments through the marginal distribution of $X_{\PA(Y)}$, but the variation is only due to covariate shift under a fixed conditional $Y\given X_{\PA(Y)}$.
 
The above arguments do not apply exclusively
to $X_{\PA(Y)}$. Similar reasoning 
extends to other
invariant subsets that 
contain 
children of $Y$ that the follower cannot manipulate. We now
define invariant sets and
show that there is an invariant set, the stable blanket, that improves upon the parents.

\subsection{Stable blankets
are never worse than
the causal parents} \label{sec:stable_blanket}
\begin{definition}[Invariant subsets and predictors] \label{def:inv_subset}
A subset of covariates $S \subseteq \{1, \ldots, d\}$ is called \emph{invariant} if $Y \indep E \given X_S$ under $\P$,
or, equivalently, if 
$Y \given X_S$ is the same under all $\mathbb{P}_e$, $e \in \Env$.
For all invariant  $S$, we define the 
predictor $f_S: \mathcal{X} \to \mathcal{A}$ as $ f_S(x) \coloneq \E[Y \given X_S = x_S]$.
\end{definition}

The stable blanket of $Y$ is an invariant subset capturing as
much information about $Y$ as possible while excluding irrelevant covariates. 
In this sense, it resembles the Markov blanket\footnote{
In $\G\setminus\{E\}$ (with $E$ and its incident edges removed), the \emph{Markov blanket} (also called Markov boundary in other literature; we follow the terminology of \citet{pfisterStabilizingVariableSelection2021}) 
$\MB(Y) = \PA(Y) \cup \CH(Y) \cup (\PA(\CH(Y))\setminus \{Y\})$ is the smallest covariate set $S$ d-separating $Y$ from all covariates outside $S$ \citep[§~3.3]{pearlProbabilisticReasoningIntelligent1988}. %
} $\MB(Y)$:
$f_{\MB(Y)}$ minimizes the risk in a single environment, but $\MB(Y)$ may contain intervened children of $Y$
and need not be invariant. To define the stable blanket, let
$\CH_\mathrm{int}(Y) \coloneqq \CH(Y) \cap \CH(E)$ denote the intervened children of $Y$, and
define
$$\DEYforb \coloneqq \CH_\mathrm{int}(Y) \cup \DE(\CH_\mathrm{int}(Y)).$$
The reader may think of these nodes as `forbidden descendants' of $Y$.
The \emph{stable blanket} \citep[Def.~3.4]{pfisterStabilizingVariableSelection2021}
is a set satisfying
$\PA(Y) \subseteq \SB(Y) \subseteq \MB(Y)$; it is defined as
\begin{equation} \label{eq:stable_blanket_decomp}
    \SB(Y) \coloneqq \PA(Y) \cup \big(\CH(Y) \setminus \DEYforb \big) \cup \big(\PA\big(\CH(Y) \setminus \DEYforb \big)\setminus\{Y\}\big).
\end{equation}
\looseness=-1
The d-separation $E \dsep{\G} Y \given \SB(Y)$
implies that $\SB(Y)$ is invariant by the global Markov property.
In the deployment DAG $\G$ in \Cref{fig:synthetic_dag}, $\DEYforb = \{4, 5\}$ (gray) and $\SB(Y) = \{1, 2, 3\}$ (blue and red).

In previous work, optimality of the stable blanket has been discussed only as an average optimality among invariant sets \citep[Thm.~3.5]{pfisterStabilizingVariableSelection2021}. 
We now prove a stronger result:
among predictors that do not use information from $X_{\DEYforb}$, $f_{\SB(Y)}$ minimizes the risk in all environments.
\begin{lemma} \label{lem:stable_region_opt}
For all $\sigma(X_{\DEYforb^C})$-measurable $f \colon \mathcal{X} \to \mathcal{A}$ and all $e \in \Env$, we have $R_e(f) \geq R_e(f_{\SB(Y)}).$ In particular, for all $e \in \Env$, $R_e(f_{\PA(Y)}) \geq R_e(f_{\SB(Y)})$.
\end{lemma}
Proofs of all results are in \Cref{app:proofs}. 
Lem.~\ref{lem:stable_region_opt} implies 
a subset-level formulation,
recovering \citep[][Thm.~3.5]{pfisterStabilizingVariableSelection2021} under faithfulness; see \Cref{app:subset_formulation}.
The leader's question is which $f$ minimizes
$\sup_{e \in \Env^*(f)} R_e(f)$ in the prediction-intervention game \eqref{eq:pred_intervention_game}. Lem.~\ref{lem:stable_region_opt} does not immediately answer this since $\Env^*(f)$ depends on $f$. 
But it shows that for two common follower objectives, stable blankets beat
parents.
\begin{theorem}[Stable blanket vs.\ parents] \label{thm:PA_vs_SB_game}
If 
(i) $\forall e \in \Env$ $\Vfoll_e(f) = R_e(f)$ or
(ii) $\forall e \in \Env$  $\Vfoll_e(f) = \E_e[\,\alpha + \beta\, f(X) - c_e(X)\,]$ for
$\alpha, \beta \in \mathbb{R}$ and measurable $c_e\colon \mathcal{X} \to \mathbb{R}$ 
not 
depending
on
$f$, then
$$
\sup_{e \in \Env^*(f_{\SB(Y)})} R_e(f_{\SB(Y)})
   \leq \sup_{e \in \Env^*(f_{\PA(Y)})} R_e(f_{\PA(Y)}).
   $$
\end{theorem}

Case~(i) describes an adversarial follower
maximizing the leader's risk. Case~(ii) is
common in strategic prediction \citep{hardtStrategicClassification2016}: the
follower maximizes $f(X)$ while incurring a cost $c_e(X)$ for the
shift. Thus, while the optimal $f$ for the leader is unknown, $f_{\SB(Y)}$ weakly dominates $f_{\PA(Y)}$ for these two utility classes.
For other utilities the ordering can reverse (Ex.~\ref{ex:parents_beat_sb}).

\subsection{
Conditions for equality of
stable blanket during
training and deployment} \label{sec:augmentation}

The previous results use the DAG $\G$, which extends the training-time DAG $\G_0$ by including edges $X_k \to X_j$ for $j \in \CH(E)$ and $k \in A_j$ that encode interventions by followers. 
If the leader only knows $\G_0$, how can they compute $\SB(Y) = \SB_{\G}(Y)$?
Would they need to anticipate which edges strategic followers add once $f$ is deployed? 
The next lemma shows that this is not necessary: 
under weak assumptions on the $A_j$'s (the sets of permitted new parents, see~\Cref{sec:problem_setup})
the stable blanket is the same in $\G_0$ and $\G$, so $\SB(Y)$ can be read off from the training-stage DAG $\G_0$ alone.

\begin{lemma}%
\label{lem:augmentation}
Assume that for all $j \in \CH_{\G_0}(E)$, $A_j \subseteq \DEforbg{\G_0}^C \setminus \{Y, E\}$, and that $\G$ is acyclic. Then $\DEforbg{\G} = \DEforbg{\G_0}$ and $\SBg{\G} = \SBg{\G_0}$. In particular, $Y \dsep{\G} E \given \SBg{\G_0}$.
\end{lemma}

The set $A_j$, $j \in \CH_{\G_0}(E)$, restricts only the new parents that the follower may add to $X_j$ in an intervention; the parents of $X_j$ in $\G_0$, which may include $Y$, are allowed to remain parents in $\G$.
Acyclicity of $\G$ is a timing condition: it rules out interventions that depend, directly or indirectly, on their own outputs.
\Cref{app:assumptions_discussion} discusses the assumptions of  Lem.~\ref{lem:augmentation} further.
In \Cref{fig:synthetic_dag}, for example, $A_1 = \varnothing$ and $A_4 = \{1, 3\}$ satisfy the assumption, since $\DEforbg{\G_0} = \{4, 5\}$ and $\G$ is acyclic, so the dashed edges leave $\SB(Y)$ unchanged.
Lem.~\ref{lem:augmentation} 
is the only result in this paper that 
uses this assumption; the other results in \Cref{sec:strategiesinvsets} and \Cref{sec:optimality} hold for a fixed DAG $\G$.

\Cref{sec:optimality} discusses the question whether $f_{\SB(Y)}$ is worst-case
optimal among all measurable predictors and connects 
this question to
standard distribution generalization.
This worst-case setting corresponds to an adversarial follower, see \Cref{thm:PA_vs_SB_game}~(i). Lem.~\ref{lem:stable_region_opt}, which holds in every environment, and \Cref{thm:PA_vs_SB_game}~(ii), in which the follower is not adversarial, are not worst-case results.

\section{Distribution generalization and an upper bound} \label{sec:optimality}

\Cref{sec:strategiesinvsets} shows that no $\sigma(X_{\DEYforb^C})$-measurable predictor improves on $f_{\SB(Y)}$ in any environment, %
and that, for two natural classes of follower objectives $\Vfoll$, $f_{\SB(Y)}$ is a better choice than $f_{\PA(Y)}$ for the leader in the prediction-intervention game \eqref{eq:pred_intervention_game} (see Lem.~\ref{lem:stable_region_opt} and \cref{thm:PA_vs_SB_game}).
However, in general, %
the optimal leader strategy depends on $\Vfoll$ (\cref{app:parents_beat_sb}),
which may be unknown to the leader and may vary across followers.
In this section, we instead
analyze
an upper bound that comes with two benefits: it holds 
for all
$\Vfoll$ and %
allows for easier optimization than the game objective. 
For all $f$, we have
\begin{equation} \label{eq:worst_case_obj}
    \sup_{e \in \Env^*(f)} R_e(f) \leq \sup_{e \in \Env} R_e(f).
\end{equation}
The bound is attained when the follower
maximizes $\Vfoll_e(f) = R_e(f)$, and can be loose
when the objectives are different (see Rmk.~\ref{rem:bound_vs_game}). Minimizing the right-hand side is a 
distribution generalization (DG) problem. For the rest of the section, we focus on the worst-case problem
\begin{equation} \label{eq:worst-case}
    f^* \in \argmin_{f \colon \mathcal{X} \to \mathcal{A}} \sup_{e \in \Env} R_e(f).
\end{equation}

This differs from the standard DG objective \citep{rojas-carullaInvariantModelsCausal2018, arjovskyInvariantRiskMinimization2020, christiansenCausalFrameworkDistribution2021} in the uncertainty set: 
in our work,
$\Env$ is generated by strategic follower interventions after deployment, possibly including additional edges into intervened variables, which subsumes the standard DG problem.
Still, the bound~\eqref{eq:worst_case_obj} makes
prediction-intervention games concrete application scenarios for existing worst-case methods. In the following, we extend and strengthen existing DG results to make them applicable in prediction-intervention games.
Our new results also hold under the augmentation construction in \Cref{sec:problem_setup} in that they apply to $(\C, \G, \Env)$ and to $(\C_0, \G_0, \Env_0)$ alike.
We formulate our results using the former notation.

\subsection{Optimality of the stable blanket} \label{sec:worst_case}

By Lem.~\ref{lem:stable_region_opt}, we know that the stable blanket solves the worst-case problem \eqref{eq:worst-case} over predictors that do not use $X_{\DEYforb}$. We now ask %
whether $f_{\SB(Y)}$ is worst-case optimal among all measurable predictors $f\colon \mathcal{X} \to \mathcal{A}$.\footnote{
The answer does not follow directly from existing work:
e.g.,
\citet[][Thm.~4]{rojas-carullaInvariantModelsCausal2018} 
(or, more precisely, Prop.~\ref{prop:adversarial}, the extension to our setting)
prove that 
for all invariant $S$, $f_S$ is a worst-case optimal predictor over distributions sharing the conditional $Y \given X_S$; this result is insufficient, as, for us, the considered distributions do not change with $S$.
}
We will prove in 
\Cref{sec:sbnotalwaysopt}
that
this is not always the case: $\SB(Y)$
excludes descendants of $\CH_\mathrm{int}(Y)$ but there are examples in which these covariates
carry substantial information about $Y$
and are only weakly affected by the intervention.
We first introduce two assumptions under which worst-case optimality %
does hold (and which prevent
such counterexamples).
\begin{assumption}[Strong interventions on $\DE(E)$] \label{as:A1}
    For all $e \in \Env$ there exists $e' \in \Env$ such that (i) 
 $X_{\SB(Y)}$ has the same marginal distribution under $\P_e$ and $\P_{e'}$, and (ii)
        $$
        Y \indep X_{\DEYforb} \given X_{\DEYforb^C} \text{ under } \P_{e'}.$$
\end{assumption}
Asm.~\ref{as:A1} is inspired by the confounding-removing interventions of \citet{christiansenCausalFrameworkDistribution2021} (see also \citealp{saengkyongamInvariantPolicyLearning2022}). Intuitively, (ii) asks for an environment $e'$ in which $\E_{e'}[Y \given X] = f_{\SB(Y)}(X)$.
This 
is not always satisfied:
it requires the environment to strongly
influence descendants of intervened children that are not directly 
affected 
by interventions
(see Ex.~\ref{ex:star_counterexample} for a setting violating this condition and \Cref{app:assumptions_discussion} for a discussion of Asm.~\ref{as:A1} and~\ref{as:A2}). 
A sufficient condition for Asm.~\ref{as:A1} is that
no such nodes exist and interventions on $\CH_\mathrm{int}(Y)$ are strong.
\begin{assumption}[Strong interventions on $\CH_\mathrm{int}(Y)$ with graphical condition] \label{as:A2} 
In $\G$, we have
\begin{equation} \label{eq:star}
    \CH(Y) \cap \DE(\CH_\mathrm{int}(Y)) \subseteq \CH_\mathrm{int}(Y)
\end{equation}
\looseness=-1
and for all $e \in \Env$ there exists $e' \in \Env$ for which the assignment of every $X_j$, $j \in \CH_\mathrm{int}(Y)$, is replaced by exogenous noise independent of all other noise variables, with all other assignments unchanged.
\end{assumption}
Either assumption ensures worst-case optimality of the stable blanket.
\begin{theorem}[Worst-case optimality of $\SB(Y)$] \label{thm:worst_case}
If either Asm.~\ref{as:A1} or~\ref{as:A2} holds, then for all
measurable $f\colon \mathcal{X}\to\mathcal{A}$, $f_{\SB(Y)}$ solves Problem~\eqref{eq:worst-case}:
\begin{equation} \label{eq:worst_case_statement}
    \sup_{e \in \Env} R_e(f) \geq \sup_{e \in \Env} R_e(f_{\SB(Y)}).
\end{equation}
\end{theorem}

Worst-case optimality of $f_{\SB(Y)}$ still holds if strong interventions are achieved only asymptotically, which can be relevant in some linear SCMs.
\begin{proposition}[Asymptotic relaxation of Asm.~\ref{as:A1}]
\label{prop:asymptotic}
For all $e \in \Env$, define the predictor $\mu_e(x) \coloneq \E_e[Y \given X = x]$. If for
all $e \in \Env$ there exists a sequence
$(e'_j)_{j \geq 1} \subseteq \Env$ such that (i) $X_{\SB(Y)}$ has the same marginal distribution under $\P_e$ and $\P_{e'_j}$ for all $j$, and
(ii) $R_{e'_j} (f_{\SB(Y)})
        - R_{e'_j}(\mu_{e'_j})
        \xrightarrow{j \to \infty} 0,$
then \eqref{eq:worst_case_statement} holds for all measurable
$f \colon \mathcal{X} \to \mathcal{A}$.
\end{proposition}

\Cref{thm:worst_case} is a function-level statement: under
Asm.~\ref{as:A1} or~\ref{as:A2}, no measurable predictor improves
on $f_{\SB(Y)}$ in worst-case risk. Our methodology itself is
subset-based:
in strategic
settings, the choice of subset shapes follower incentives, as in the example of \Cref{sec:intro}. Furthermore, subset-based predictors are interpretable 
and easy to audit, which matters in regulated settings such as lending \citep{chen2018fair, cfpbCircular2022_03}.
\Cref{app:worst-case-opt-functions} contrasts the subset view
with function-based invariance criteria such as IRM
\citep{arjovskyInvariantRiskMinimization2020} and HSIC-X
\citep{saengkyongamExploitingIndependentInstruments2022}.

\subsection{
Counterexamples} \label{sec:sbnotalwaysopt}
We now show 
that, in general, the graphical condition~\eqref{eq:star} in Asm.~\ref{as:A2} cannot be dropped.

\begin{example} \label{ex:star_counterexample}
Consider an SCM with independent noises $\varepsilon_1 \sim \mathrm{Ber}(p)$ for $p \in (0,1/2)$ and $\varepsilon_2$:\\[0pt]
\begin{minipage}{0.32\textwidth}
\begin{align*}
    Y &\sim \mathrm{Ber}(1/2), \\
    X_2 &\coloneq h(Y,E,\varepsilon_2) \in \{0,1\},
\end{align*}
\end{minipage}
\hfill
\begin{minipage}{0.32\textwidth}
\begin{align*}
    X_1  &\coloneq 
\begin{cases}
Y & \text{if } X_2 = 0,\\
Y \oplus \varepsilon_1 & \text{if } X_2 = 1,
\end{cases}
\end{align*}
\end{minipage}
\hfill
\begin{minipage}{0.32\textwidth}
\centering
\input{dag_3}
\end{minipage}\\[1pt]
where $\oplus$ is addition mod $2$, and %
the function $h(\cdot,e,\cdot)$ may vary with $e$, so that
different environments correspond to different mechanisms for $X_2$.
Here, $\SB(Y) = \varnothing$, and the associated DAG violates Asm.~\ref{as:A2}. There exists a classifier $f\colon \mathcal{X} \to [0,1]$ achieving a
Brier score risk 
that is uniformly lower than the one of
$f_{\SB(Y)} \equiv 1/2$: for all $e \in \Env$, we have
$R_e(f) \leq p(1-p) < \tfrac{1}{4} = R_e(f_{\SB(Y)})$ (\Cref{app:proof_ex:star_counterexample}). 
\end{example}
In the regression setting with squared loss, for every DAG in which \eqref{eq:star} fails, there exists an SCM with that DAG for which $f_{\SB(Y)}$ is not worst-case optimal (proved formally as Prop.~\ref{prop:star_necessary}).

\section{
Learning subset-based predictors
} \label{sec:learning_subset_predictors}
\looseness=-1
When the causal graph is known, the stable blanket $\SB(Y)$ is identified and $f_{\SB(Y)}$ can be estimated by pooling data across training environments. 
In practice, the causal graph is often unknown.
In the population case, under faithfulness, every subset invariant on $\Env$ is contained in $\DEYforb^C$ (\Cref{app:subset_formulation}), so by Lem.~\ref{lem:stable_region_opt}, $\SB(Y)$ is a most predictive invariant subset.
With finite samples from $\Envtr$, we therefore propose to select subsets $S \subseteq \{1, \ldots, d\}$ that are invariant on $\Envtr$ and predictive of $Y$, and estimate $f_S$ by pooling data across $\Envtr$.
Since invariance on $\Envtr$ need not extend to $\Env$, this does not guarantee recovering $\SB(Y)$ (even for infinite data), so the methods below can be regarded as finding proxies for $f_{\SB(Y)}$.
We build on existing tests and methods. 
Throughout, let $\hat{f}_S: \mathcal{X} \to \mathcal{A}$ estimate $f_S(x) = \E[Y \given X_S = x_S]$, trained on data pooled across $\Envtr$; any base model (such as random forests or neural networks) can be used.
There are various ways of testing invariance, i.e.,
$H_{0,S}\colon Y \indep E \given X_S$, %
which we detail
in \Cref{app:invariance_tests}. We take the resulting
$p$-value as an invariance score $s_\mathrm{inv}(S)$, yielding
$\hat{\mathcal{I}}_{\Envtr} \coloneq \{S \given s_\mathrm{inv}(S) \geq
\alpha_\mathrm{inv} \}$ for a tuning parameter
$\alpha_\mathrm{inv} \in (0,1)$. 
To reduce the exponential cost of %
computing $s_\mathrm{inv}$ for 
all $2^d$ subsets, one can perform %
a screening step (e.g., $\ell_1$-penalization; see \Cref{app:SC_alg}). We denote by $s_\mathrm{pred}(S)$ a prediction score of $\hat{f}_S$, such as the out-of-sample negative MSE for regression or the out-of-sample negative binary cross-entropy for classification on the training environments.

\paragraph{Invariant most-predictive subsets.}
Given $\hat{\mathcal{I}}_{\Envtr}$, one option is to predict using the invariant subset with the highest prediction score, that is, $\hat{S}_\mathrm{max} \coloneq \argmax_{S \in \hat{\mathcal{I}}_{\Envtr}} s_\mathrm{pred}(S)$ \citep{rojas-carullaInvariantModelsCausal2018}. 
We call this approach IMP (invariant most-predictive) \citep{Gnecco2026jmlr}.

\paragraph{Stabilized regression and classification.}
Since IMP is based
on a single subset, 
errors in the invariance test may have a large impact on IMP's performance.
Stabilized regression \citep{pfisterStabilizingVariableSelection2021} addresses this by using an ensemble
$\hat{f}^{\mathrm{SR}}(x) \coloneq \sum_{S \subseteq \{1, \ldots, d\}} \hat{\omega}_S \hat{f}_S(x)$, 
where the 
non-negative weights sum to one.
Non-zero weights are assigned only to subsets that are both invariant and sufficiently predictive, with the predictiveness threshold calibrated by a bootstrap. To the best of our knowledge, we are the first to port this idea to classification by replacing the base regressors, the prediction score (cross-entropy instead of MSE), and the invariance tests with their classification counterparts, yielding stabilized classification (SC). Details on SC are given in \Cref{app:SC_alg}.

\section{Experiments} \label{sec:experiments}
\subsection{Synthetic prediction-intervention games}
\label{sec:synthetic_experiments}

\looseness=-1
We validate our theoretical results 
on synthetic data.
For this we consider two data-generating processes: 
a linear-Gaussian SCM and a nonlinear SCM. %
Both have the same induced training DAG
$\G_0$ in \cref{fig:synthetic_dag}.
We compare three predictors (either in a population or finite-sample version): 
based on $\PA(Y)$,
$\SB(Y)$, or all variables.
The follower is represented by a neural intervention mechanism with intervention
bound $b$ and is trained either to maximize the leader's MSE or to minimize the
prediction value (for details see \Cref{sec:sdfds}). All results are averaged over 10 independent runs, with shaded regions indicating $95\%$ confidence intervals for the mean. Further details are provided in \Cref{app:dgp}.

\looseness=-1
We first consider the linear-Gaussian SCM (\Cref{fig:synthetic_main}, left). In this setting the relevant conditional
expectations are available in closed form, so the predictors use population
regression quantities rather than estimates from finite training data. This
isolates the effect of the follower's intervention from statistical estimation
errors.
Both $\PA(Y)$ and $\SB(Y)$ yield
constant deployment MSE, 
as predicted for $\PA(Y)$ in \cref{sec:parents}; for any invariant subset $S$, in particular $\SB(Y)$,
the deployment MSE of $f_S$
in $e$
equals $\E_e[\Var(Y \given X_S)]$, and $\Var(Y \given X_S)$ is constant in linear-Gaussian SCMs.
$\SB(Y)$ improves over $\PA(Y)$ because the non-intervened child $X_3$ carries additional stable
information about $Y$. By contrast, the all-variable predictor is 
unstable: it performs well for weak interventions, but its MSE increases sharply
as the intervention bound grows, since it relies on the manipulable variable $X_4$.

The nonlinear SCM shows the same qualitative ordering (\Cref{fig:synthetic_main}, right). The stable blanket and
the parent set remain more stable than the all-variable predictor, and the
stable blanket again improves over the parents. Unlike in the linear-Gaussian
experiment, however, the stable-blanket and parent predictors are not perfectly
flat. 
This is expected: the predictors are now learned from finite training data, and for the stable blanket the population MSE $\E_e[\Var(Y \given X_{\SB(Y)})]$ is no longer constant, since $\Var(Y \given X_{\SB(Y)})$ is not constant in $X_{\SB(Y)}$ in the nonlinear SCM.
The follower can exploit finite-sample regression errors by shifting the
test distribution toward regions in which the fitted predictor differs from the
population regression function. \Cref{app:synthetic_regression_error}
investigates this effect directly: as the leader's training sample size grows,
the MSE curves for $\PA(Y)$ and $\SB(Y)$ become more stable, while the
all-variable predictor remains unstable.

\looseness=-1
Finally, we also allow the follower's intervention on $X_4$ to depend on
$X_1$ and $X_3$ in addition to its usual inputs (captured by the dashed edges in the deployment DAG $\G$ in \Cref{fig:synthetic_dag}). 
The qualitative results do not
change; see \Cref{app:synthetic_details}.
\Cref{app:individual_followers} presents a similar experiment with individual followers as in %
\Cref{app:individual_followers_theory}.

\begin{figure*}[t]
    \centering
    \includegraphics[width=\textwidth]{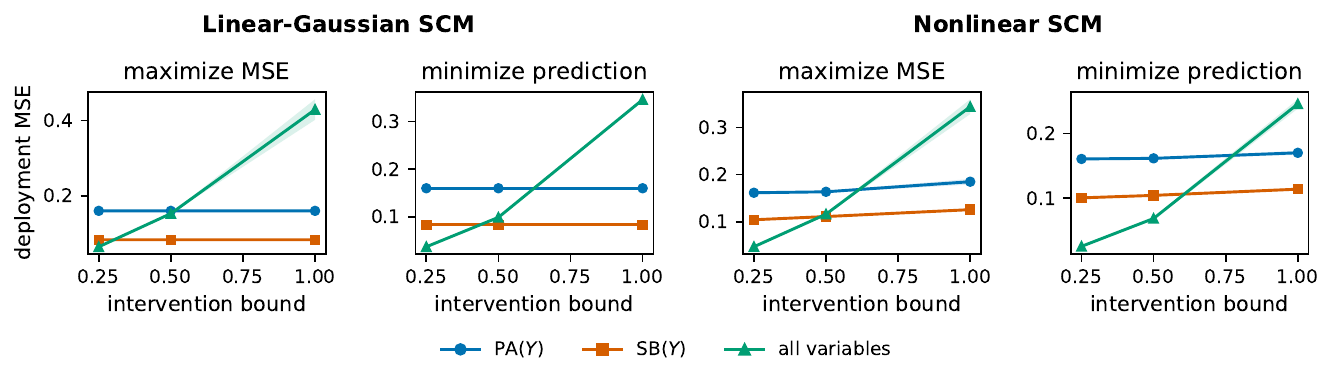}
    \caption{Synthetic prediction-intervention games. Left two panels:
    linear-Gaussian SCM with population regression predictors. Right two panels:
    nonlinear SCM with predictors estimated from finite data. The stable blanket
    is uniformly better than the parent set and both are substantially more
    stable than the all-variable predictor.
  In the nonlinear SCM, the small increase for $\PA(Y)$ is due to regression error alone; that for $\SB(Y)$ is mostly due to variation in population risk (shown in \Cref{app:synthetic_regression_error}).
    }
    \label{fig:synthetic_main}
\end{figure*}

\subsection{Prediction-intervention games on Causal Chambers} \label{sec:new_chamber_experiments}
We complement the synthetic experiments of \Cref{sec:synthetic_experiments} with a prediction-intervention game on real data from the Causal Chambers\footnote{The Causal Chambers are devices controlled through an API and collect measurements from a physical system
for which we have partial causal-graph information.
} light tunnel (\cref{fig:causal_chambers_dags} left) of \citet{gamellaCausalChambersRealworld2025}.

\paragraph{Setup.} 
The chamber takes a light source setting $\mathtt{RGB} \coloneq (\mathtt{red}, \mathtt{green}, \mathtt{blue}) \in \mathbb{R}^3$ and returns infrared and visible-light measurements ($\mathtt{ir\_j}$, $\mathtt{vis\_j}$) at positions $\mathtt{j} \in \{1,2,3\}$ along the chamber.
We aim to predict $Y \in \{0,1\}$ (constructed %
as a function of $\mathtt{ir\_1}$) 
and observe the system under different environments. 
The leader's loss is the MSE; the follower's objective is to minimize the population mean prediction $\Vfoll_e(\hat{f}) = -\E_e[\hat{f}(X)]$, as considered in \Cref{thm:PA_vs_SB_game}.
Given $\hat{f}$, the follower can query the Causal Chambers API to receive $500$ observations of $X$ under different interventions on $\mathtt{ir\_3}$ and $\mathtt{vis\_3}$. We consider a fixed collection of $49$ candidate environments the follower can choose from. We restrict their choice with an intervention bound $b \in [0, 1]$: only the $\max\{1, \mathrm{round}(49b)\}$
environments closest to a reference environment $e_\mathrm{ref}$ are available.
See \Cref{app:additional_details_causal_chambers}.

\begin{figure}[t]
    \centering
    \begin{minipage}{0.25\textwidth}
        \centering
        \includegraphics[width=\linewidth]{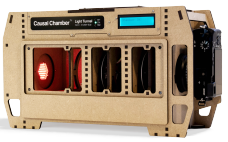}
    \end{minipage}
    \hfill
    \begin{minipage}{0.35\textwidth}
        \centering
        \input{dag_2}
    \end{minipage} \hfill
    \begin{minipage}{0.35\textwidth}
        \centering
        \input{dag_6}
    \end{minipage}
    \caption{
    Left: light tunnel 
    (adapted from \citet{gamellaCausalChambersRealworld2025}, licensed under CC BY 4.0).
    Middle: inaccurate graph, 
    inferred from the partial causal-graph information by
    \citet{gamellaCausalChambersRealworld2025}; 
    the follower is allowed to intervene on 
    $\mathtt{ir\_3}, \mathtt{vis\_3}$;
    the stable blanket of $Y$ is $\{\mathtt{RGB}\}$. Right: potentially correct graph with a hypothesized hidden confounder $H$ pointing to all $\mathtt{ir\_j}$ and $\mathtt{vis\_j}$; the stable blanket of $Y$ is $\{\mathtt{RGB}, \mathtt{ir\_2}, \mathtt{vis\_2}\}$. Both graphs omit $\mathtt{ir\_1}$, which is only used to define $Y$.}
    \label{fig:causal_chambers_dags}
\end{figure}

\paragraph{Inaccurate causal graph.}
\citet{gamellaCausalChambersRealworld2025} provide partial causal-graph
information for the light tunnel. 
They note that absent edges do not necessarily exclude the existence of
causal effects or hidden confounding.
Ignoring this note and adding the edges from $Y$ and $E$ to
$(\mathtt{ir\_3}, \mathtt{vis\_3})$, one might be tempted to consider the graph in \cref{fig:causal_chambers_dags} (middle) as
correct.
However, 
there is empirical evidence that this
is not the correct graph for our experiment:
conditional independence tests on training data reject\footnote{Specifically, $\mathtt{ir\_i} \indep \mathtt{ir\_j} \given \mathtt{RGB}$ and $\mathtt{vis\_i} \indep \mathtt{vis\_j} \given \mathtt{RGB}$ for $\mathtt{i} \ne \mathtt{j} \in \{1, 2, 3\}$ are rejected in the majority of training environments at level $0.01$; see \Cref{app:hidden_confounding}.} some of the implied independencies between sensors. %
We
hypothesize that this is due to a hidden confounder (such as background
lighting) acting on the sensors. Adding it results in \cref{fig:causal_chambers_dags} (right). Under this extended graph with hidden variables the stable blanket of $Y$
can be defined analogously \citep{xiangStableBlanketHidden2026} and
equals $\{\mathtt{RGB}, \mathtt{ir\_2}, \mathtt{vis\_2}\}$; \Cref{app:stable_blanket_hidden} provides more details.

\begin{figure}[t]
    \centering
    \includegraphics[width=0.94\linewidth]{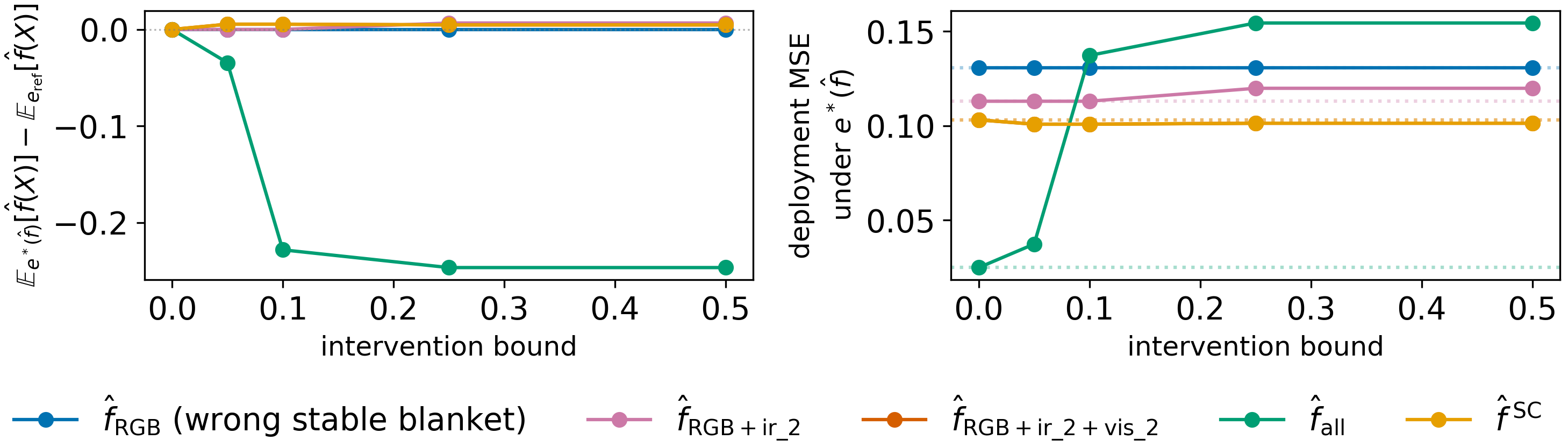}
    \caption{%
    \looseness=-1
        Prediction-intervention game on Causal Chambers data (\Cref{sec:new_chamber_experiments}). %
        Left: follower-induced shift in the population mean prediction relative to a reference environment $e_\mathrm{ref}$. 
        Right:
        MSE of the leader's deployed predictor %
        in $e^*$ (chosen by the follower). 
 The follower has a strong influence on their objective only for $\hat{f}_\text{all}$.
The proposed $\hat{f}^\mathrm{SC}$ (SC for stabilized classification) achieves the lowest MSE
for $b \geq 0.1$, tied with $\hat{f}_{\mathtt{RGB}, \mathtt{ir\_2}, \mathtt{vis\_2}}$
(curves overlap); 
$\hat{f}^\mathrm{SC}$ does not require graph knowledge. %
}
    \label{fig:follower_chambers}
\end{figure}
 
\paragraph{Results.}
\cref{fig:follower_chambers} 
shows 
the degree to which the follower can improve their objective (left), and the leader's MSE under the environment chosen by the follower (right), both computed on $500$ new evaluation observations.
Only under $\hat{f}_\text{all}$ (using all variables) can the follower improve their objective (left), and for large intervention bounds, this predictor performs worst for the leader. 
$\hat{f}_{\mathtt{RGB}}$ is the stable-blanket predictor under the inaccurate graph and 
has a higher MSE than the larger subsets that include $\mathtt{ir\_2}$ and $\mathtt{vis\_2}$ on the reference environment. This is consistent with the hidden-confounding hypothesis that the graph is incorrect. 
The SC predictor $\hat{f}^\mathrm{SC}$, which is purely based on data from three training environments and not the graph, selects $\{\mathtt{RGB}, \mathtt{ir\_2}, \mathtt{vis\_2}\}$ (the stable blanket under the hidden-confounder graph) as the unique subset passing both the
invariance and predictiveness filters, and therefore matches the performance
of $\hat{f}_{\mathtt{RGB}, \mathtt{ir\_2}, \mathtt{vis\_2}}$.

\section{Discussion and future work} \label{sec:conclusion}

\looseness=-1
We have introduced prediction-intervention games, connected them to
invariant sets and worst-case distribution generalization, and proved
optimality results for strategies based on the stable blanket. We proposed practical methods that the leader may use for settings 
with known and unknown causal graph, and have successfully tested them on data.
Many of our results concern the worst-case problem~\eqref{eq:worst-case}, such as \Cref{thm:PA_vs_SB_game}~(i), \Cref{thm:worst_case}, Prop.~\ref{prop:asymptotic}, and the examples in \Cref{sec:sbnotalwaysopt}. Exceptions are Lem.~\ref{lem:stable_region_opt}, which holds in every environment, and \Cref{thm:PA_vs_SB_game}~(ii), in which the follower is not adversarial. Our results do not characterize the optimal leader strategy in general.
One limitation of
the worst-case optimality results in \Cref{sec:worst_case} is that they
rely on
strong-intervention assumptions.
Analyzing optimal strategies under 
intervention-constrained followers
is an interesting direction for future research.
We hypothesize that solving prediction-intervention games is in
general a hard problem \citep[e.g.,][]{ConitzerS06}.
Finally, in practice, the leader uses a finite-sample estimate $\hat f_S$ rather than
the population predictor $f_S$, and the follower can shift covariates to high-error regions
(\Cref{app:synthetic_regression_error}); a finite-sample analysis is left for future work.

\begin{ack}
We thank Niklas Pfister and Sorawit Saengkyongam for helpful discussions, and Juan Gamella for the help with the Causal Chambers data generation. During this work, LK was
supported by the ETH AI Center through an ETH AI Center doctoral fellowship. We used Claude Opus 4.6, Fable 5, and Fable 5.1 to polish the language and proofread, to check and tighten proofs, to assist the literature search, to draft parts of the appendix (such as \Cref{app:stable_blanket_hidden}) from our notes, and to implement simulations and figures.
The authors take full responsibility for the content.

\end{ack}

\vskip 0.2in
\bibliographystyle{plainnat}
\bibliography{references}

\appendix
\clearpage

\addcontentsline{toc}{section}{Appendix}

\listofappendices
\vskip 0.75in

\appsection{Real-world examples} \label{app:real_world_examples}
Prediction-intervention games arise in US consumer credit scoring, where the score developer (such as FICO or VantageScore) takes the role of the leader and credit applicants are followers.
Credit scores predict future repayment from covariates derived from an applicant's credit record.
One covariate summarizes authorized-user accounts: revolving accounts (typically credit cards) that the applicant may use without being liable for the charges.
Regulation~B, which implements the 1974 US Equal Credit Opportunity Act, requires lenders to consider accounts an applicant shares with their spouse \citep{cfpb2026regulationB}.
Since credit records do not reveal which authorized users are spouses, scoring models have traditionally treated all authorized-user accounts like the applicant's own accounts \citep{brevoort2013credit}.
This enables a strategic intervention known as `piggybacking': so-called `credit repair' companies add credit applicants as authorized users (in name only, without receiving a card) on a stranger's old, well-maintained account for a fee.\footnote{According to \citet[Footnote~2]{brevoort2013credit}, applicants pay between \$1{,}000 and \$2{,}000 for this service, and the holder of the rented account can earn about \$200 per month.}
The account's favorable history then enters the applicant's credit record, while their repayment behavior remains unchanged.
In our terms, the applicant intervenes at a cost on a covariate that is predictive of $Y$ but leaves $Y$ unchanged, in order to increase $f(X)$, as in \cref{thm:PA_vs_SB_game}~(ii).
\citet{brevoort2013credit} estimate that adding one such account moves $28\%$ of `subprime' and $35\%$ of `near-prime' borrowers into a higher credit-quality segment.
FICO announced in 2007 that its next model would exclude authorized-user accounts; after warnings that exclusion might violate Regulation~B, FICO instead reduced the weight of these accounts \citep{brevoort2013credit}.
This is an example of a prediction-intervention game in which the leader initially chose to exclude a covariate from its model to protect itself against strategic interventions by followers, even though the covariate is predictive in some subgroups: excluding it leaves overall predictive accuracy statistically unchanged, but significantly reduces accuracy for applicants with two or fewer accounts of their own \citep{brevoort2013credit}.

A second example is web page ranking, where a search engine scores the relevance of a page from covariates including its content and metadata.
The \texttt{keywords} meta tag lists terms the page owner places in the HTML.
It is invisible to visitors and can be changed independently of the page's content.
After widespread `keyword stuffing' (filling the tag with irrelevant terms to attract traffic), Google excluded it from its web ranking algorithm \citep{cutts2009keywords}.

\appsection{Extended related work} \label{app:extended_related_works}

We expand on the related work discussion from \Cref{sec:intro}.
 
\paragraph{Strategic classification and performative prediction.}
Strategic classification \citep{brucknerStackelbergGamesAdversarial2011, hardtStrategicClassification2016} models followers who respond to a deployed classifier $f$ by modifying their features with some cost $c$ to obtain a more favorable prediction, typically through the pointwise best response $x^f \in \argmax_{x'} f(x') - c(x, x')$. The leader's goal is to learn an $f$ that is robust under the induced distribution shift. The leader's and follower's objectives are typically asymmetric in this Stackelberg formulation; \citet{ehrenbergAdversariesIncentivesStrategic2024} make this explicit and study a game in which the leader minimizes a worst-case loss over a set of different utility-maximizing opponents, without a causal structure over $X$ and $Y$. A related question is how the choice of predictor creates different incentives for the followers and distinguishes gaming the system 
from genuine improvement:
\citet{kleinbergHowClassifiersInduce2019} model followers as investing effort into different actions that affect the observed features, and ask which decision rules incentivize followers to invest in actions a leader views as valuable rather than in gaming; \citet{chenLinearClassifiersThat2021} distinguish improvable, manipulable, and immutable features, with gaming defined as changes to manipulable features that by assumption do not affect $Y$. They do not specify a full SCM over $(X, Y, E)$ that determines how the follower's actions propagate through the whole system. Performative prediction \citep{perdomo2020performative, hardtPerformativePredictionFuture2025} considers the broader phenomenon that deployed models shift the data distribution; \citet{mendler-dunnerAnticipatingPerformativityPredicting2022} address identifiability of the causal effect of predictions. These works do not generally specify the follower's moves through an SCM.
\citet{perdomo2020performative} formalize the dependence of the data 
on the deployed model by an abstract function $D(\cdot)$ from models to distributions and evaluate $f$ by its performative risk, the risk of $f$ under $D(f)$.
In our game, $D(f)$ is the distribution $\P_e$ induced by a best response $e \in \Env^*(f)$, and the leader's objective in \eqref{eq:pred_intervention_game} is the performative risk, with a supremum over best responses.
$D(\cdot)$ is generally unknown to the decision maker, and the results of \citet{perdomo2020performative} concern retraining procedures such as repeated risk minimization, which retrains on the distribution induced by the previously deployed model and, under regularity conditions, converges to a fixed point of retraining.
Such a fixed point need not minimize the performative risk, which accounts for the change of the distribution with the predictor.
These retraining procedures 
deploy a sequence of predictors and observe the induced distributions, which the leader in our game cannot do, since $f$ is chosen before the follower responds.
We instead specify how $D(\cdot)$ arises, through the follower's interventions on the children of $E$ in an SCM, and use the invariances implied by this structure. 
Thus, the methods proposed in this paper can compare predictors before deployment (possibly using a single dataset), whereas retraining methods in performative prediction
require a sequence of datasets.
Other game-theoretic interactions between predictors and strategic agents have been studied in market settings \citep{huIncentivizingHighQualityContent2023}.

\paragraph{Causal Stackelberg formulations.}
Several works model Stackelberg games with SCMs. 
\citet{shavitCausalStrategicLinear2022} formulate strategic linear regression in a linear SCM, and learn linear predictors targeting one of three specific objectives: minimizing the leader's risk under strategic follower responses, incentivizing followers to genuinely improve their outcomes, or estimating the parameters of the outcome model. Their algorithms deploy multiple predictors and observe the resulting follower responses to learn the final predictor. In contrast, our predictor is defined without observing follower interventions. We further allow a general SCM, cover regression and classification under any loss strictly consistent for the mean functional, and place no specific functional form on the follower's response: their setup fixes it through an effort conversion matrix and a quadratic cost, whereas ours constrains the follower only through the graph $\G$ and the action sets $A_j$.
\citet{miller2020strategic} develop a causal framework that distinguishes gaming from improvement and prove that designing classifiers which incentivize improvement is at least as hard as orienting the edges of the underlying causal graph. Their contribution is conceptual and concerns the difficulty of incentive design, while we focus on the leader's optimal choice of $f$ given an SCM.
\citet{yanDiscoveringOptimalScoring2023} assume an unknown causal graph and follower manipulations constrained by a fixed budget. Their algorithm iteratively deploys predictors and observes strategic responses to recover the graph, and their leader objective trades off prediction accuracy with outcome improvement. Our predictor does not require iterated deployment, and in the unknown-graph case (\Cref{sec:learning_subset_predictors}) we use multiple training environments instead of observed strategic responses.
\citet{horowitzCausalStrategicClassification2023} partition covariates into causal and non-causal predictive components. Their follower is an individual who modifies covariates to change the prediction in its favor, and only changes to causal variables propagate to $Y$. 
We, instead, allow for a more flexible model of the follower:  we may specify which variables the follower may intervene on, and which other variables their interventions may depend on. The latter can induce additional edges in the graph, so our setup accommodates settings in which the deployment graph extends the training graph. Unlike the work by \citet{horowitzCausalStrategicClassification2023}, we do not estimate the causal function of the response but work with invariant sets; furthermore, we admit regression in addition to classification, allow for more general losses 
and for more general follower utilities; finally, \citet{horowitzCausalStrategicClassification2023} do not provide optimality results in their setting.
\citet{kulynychCausalPredictionCan2022} extends an SCM over $(X,Y)$ by including the deployed predictor as an additional node $M$, with edges $M \to X$ but no edge to $Y$. Instead of modeling a follower optimizing its own utility, the deployment-induced shift in $X$ is encoded directly in the structural assignments of the SCM. In this setting, they study conditions under which the Markov-blanket predictor remains optimal under the shift it induces. In our formulation, the post-deployment distribution instead arises from a follower choosing interventions to maximize an objective distinct from the leader's risk, which the leader anticipates when choosing $f$. 
\citet{mishra2024} study interventions in causal games \citep{hammond2023}, 
graphical models of multi-agent decision-making, and give a set of primitive 
interventions from which arbitrary interventions can be composed; one 
application represents a leader's commitment to a policy, fixed before the 
follower acts, as an intervention on the leader's decision rule. As in our 
setting, a leader commits to a policy to which a follower responds, and 
interventions may add edges to the graph.
Our game could be represented as a causal game with the intervened variables as the follower's decision variables and the prediction as the leader's.
They study
mechanism design, 
whereas we study a specific prediction problem and analyze the predictive risk 
and worst-case optimality of predictors based on invariant subsets.

\paragraph{Causal distribution generalization and invariant subsets.}
Invariance-based causal prediction \citep{petersCausalInferenceUsing2015, magliacaneDomainAdaptation2018, christiansenCausalFrameworkDistribution2021} uses subsets $S$ for which $Y \given X_S$ is stable across environments. \citet{pfisterStabilizingVariableSelection2021} introduce the stable blanket.
For regression, invariant subsets have been motivated by worst-case performance: \citet{rojas-carullaInvariantModelsCausal2018} show that for each invariant subset $S$, $f_S$ minimizes the worst-case squared loss over the family $\mathcal{Q}_S$ of distributions sharing the conditional $Y \given X_S$. With Prop.~\ref{prop:adversarial} 
in \Cref{app:rojas}, we generalize this result to binary classification and regression with any loss strictly consistent for the mean functional. 
This result is not directly applicable to the optimization problem \eqref{eq:worst-case}, as
the family $\mathcal{Q}_S$ depends on $S$ and may strictly contain $\{\P_e\}_{e \in \Env}$, in which case $f_S$ guards against shifts that cannot occur under $\Env$.  The result
does not inform us about which invariant subset to use, either.
We instead analyze worst-case optimality directly over $\{\P_e\}_{e \in \Env}$ and give conditions under which the stable blanket is optimal among all measurable predictors.
Our worst-case analysis is related to a line of work taking a distributional-robustness view in which the uncertainty set comes from SCM interventions \citep{meinshausenCausalityDistributional2018, buhlmannInvarianceCausalityRobustness2020, subbaswamyUnifyingCausalFramework2022, shenCausalityorientedRobustness2023}; this connects to the broader DRO literature \citep{Kuhn_Shafiee_Wiesemann_2025, sagawa_groupDRO, freni2025maximum}.
Through the bound~\eqref{eq:worst_case_obj}, prediction-intervention games provide a concrete example of such a worst-case problem with the distribution family $\{\P_e\}_{e \in \Env}$ determined by the SCM. These works do not directly apply to our setting, however: their uncertainty sets (interventions of a fixed form on a fixed DAG or balls around a reference distribution in the case of DRO) generally do not coincide with $\{\P_e\}_{e \in \Env}$. Anchor regression \citep{rothenhauslerAnchorRegressionHeterogeneous2020} attains worst-case optimality over a class of bounded shift interventions. Other function-view criteria include IRM \citep{arjovskyInvariantRiskMinimization2020} and instrument-based independent residuals \citep{saengkyongamExploitingIndependentInstruments2022}. We do not use a function view but a subset view, as discussed in \Cref{app:worst-case-opt-functions}.

\paragraph{Comparison to \citet{subbaswamyPreventingFailuresDue2019} and \citet{leeFindingInvariantPredictors2023}.%
}
\leavevmode\citet{subbaswamyPreventingFailuresDue2019} assume knowledge of a causal graph, possibly with hidden confounders, in which the mechanisms of a known set $M$ of mutable variables may change arbitrarily; this defines an uncertainty set $\Gamma$ of distributions.
Their proposed full conditional surgery predictor $f_s(x) \coloneqq \E[Y \given X_{M^C} = x_{M^C}, \operatorname{do}(X_M = x_M)]$ intervenes on $M$ and conditions on the remaining covariates; it is not always identifiable from observational data.
Their Alg.~2 therefore searches over predictors $\E[Y \given X_Z = x_Z, \operatorname{do}(X_M = x_M)]$ for conditioning sets $Z \subseteq M^C$, retains those that are identifiable from the observational distribution, and selects by the lowest held-out source loss.
Our setting maps to this framework with $M = \CH(E)$.
Since $\SB(Y)$ is a stable conditioning set in their terminology, $f_{\SB(Y)}$ is among the predictors their method considers \citep[Lem.~1 and Cor.~1]{subbaswamyPreventingFailuresDue2019}.
However, as the following example in our fully observed setting shows, their selection rule, evaluated with population source risk, can return a predictor with higher worst-case risk than $f_{\SB(Y)}$ in a setting that satisfies Asm.~\ref{as:A1}, where $f_{\SB(Y)}$ is worst-case optimal by \Cref{thm:worst_case}.
\begin{example} \label{ex:surgery_selection}
Consider the DAG of Ex.~\ref{ex:star_counterexample}, the squared loss, and the SCM\\[0pt]
\begin{minipage}{0.49\textwidth}
\begin{align*}
    Y &\sim \mathrm{Unif}\{-1, 0, 1\}, \\
    X_2 &\coloneqq h(Y, E, \varepsilon_2) \in \{0, 1\}, \\
    X_1 &\coloneqq (1 - X_2)\varepsilon_1 + X_2 \bm{1}\{Y \geq 0\},
\end{align*}
\end{minipage}
\begin{minipage}{0.49\textwidth}
\centering
\vspace*{0.7em}
\input{dag_3}
\end{minipage}\\[6pt]
where $\varepsilon_1 \sim \mathrm{Ber}(1/2)$ is independent of $(Y, \varepsilon_2)$.
$\Env$ consists of a training environment $e_\mathrm{tr}$ with $\P(X_2 = 1 \given Y = \pm 1) = 3/4$ and $\P(X_2 = 1 \given Y = 0) = 1/4$, and two deployment environments $e_1$ and $e_2$ with $X_2 \coloneqq \bm{1}\{Y = 0\}$ and $X_2 \coloneqq 0$, respectively.
Here, $M = \{2\}$ and $\SB(Y) = \varnothing$, so $f_{\SB(Y)} \equiv \E[Y] = 0$.

The two identifiable candidates are $f_{\SB(Y)}$, with the empty conditioning set, and $f_s$, with conditioning set $\{1\}$: under $\operatorname{do}(X_2 = 0)$, $X_1 = \varepsilon_1$ is independent of $Y$, and under $\operatorname{do}(X_2 = 1)$, $X_1 = \bm{1}\{Y \geq 0\}$, so
$$
f_s(x) = \E[Y \given X_1 = x_1, \operatorname{do}(X_2 = x_2)] = x_2 \big(-1 + \tfrac{3}{2} x_1\big) \quad \text{for } x \in \{0, 1\}^2.
$$
The MSEs in $(e_\mathrm{tr}, e_1, e_2)$ are $(1/4, 3/4, 2/3)$ for $f_s$ and $(2/3, 2/3, 2/3)$ for $f_{\SB(Y)}$.
The selection rule therefore returns $f_s$ because of its lower source MSE, although its worst-case MSE over $\Env$ is $3/4$, compared to $2/3$ for $f_{\SB(Y)}$.
The assumptions of \citet[Thms.~3 and~4]{subbaswamyPreventingFailuresDue2019} do not hold here ($f_s$ depends on $x_2$, so $\P(Y \given X_1, \operatorname{do}(X_2))$ would be a function of their $M$).
\end{example}
A further difference concerns the uncertainty set: the guarantees of \citet[Thms.~3 and~4]{subbaswamyPreventingFailuresDue2019} are worst-case guarantees for the squared error over $\Gamma$, whereas our uncertainty set $\{\P_e\}_{e \in \Env}$ may be a strict subset of $\Gamma$, and interventions that add permitted parent edges can yield distributions outside $\Gamma$.
Moreover, $\Gamma$ contains every change of the mechanisms of $M$ and thus strong interventions (as in Asm.~\ref{as:A1} or~\ref{as:A2}) by construction, whereas $\{\P_e\}_{e \in \Env}$ need not; our results in \Cref{sec:strategiesinvsets} assume neither Asm.~\ref{as:A1} nor~\ref{as:A2}.
Besides using a different uncertainty set, \citet{subbaswamyPreventingFailuresDue2019} restrict their guarantees to the squared error and to the worst case, whereas all our results 
apply to
any loss strictly consistent for the mean, and Lem.~\ref{lem:stable_region_opt} and \Cref{thm:PA_vs_SB_game} cover non-adversarial follower objectives.
In our setting, under condition~\eqref{eq:star}, $\SB(Y)$ is the Markov blanket of $Y$ after removing the edges into $M$, so $f_s = f_{\SB(Y)}$ and the condition of their Thm.~3 holds.
\citet{leeFindingInvariantPredictors2023} build on the framework of \citet{subbaswamyPreventingFailuresDue2019} with a polynomial-time greedy algorithm, but do not prove worst-case optimality for the returned predictor.
The methods of \citet{subbaswamyPreventingFailuresDue2019} and \citet{leeFindingInvariantPredictors2023}
rely on conditional independencies in interventional distributions, which \citet{shpitserDormantIndependence2008} study under the name of dormant independencies.

\appsection{Additional theoretical results}
\subsection{Individual followers with shift interventions and target-specific observations} \label{app:individual_followers_theory}

The argument of \Cref{sec:individuals} assumes for simplicity that the individual follower observes only variables $X_O$ with $O \subseteq \ND(\CH(E))$ and that $\PA(j) \setminus \{E\} \subseteq O \cup \CH(E)$ for every intervention target $X_j$, $j \in \CH(E)$.
Here, we extend the argument to individual followers 
like the ones
in the setting of \Cref{sec:synthetic_experiments}, which differs from \Cref{sec:individuals} in two ways.
First, the interventions are (bounded) shifts on $X_1$ and $X_4$, so the follower observes the values $X_1$ and $X_4$ would have taken before the intervention.
Second, the follower sets $X_4$ after observing $Y$, which is a descendant of the earlier intervention target $X_1$ and hence not in $\ND(\CH(E))$.
We therefore consider target-specific observation sets $O_j$ instead of the single $O$ of \Cref{sec:individuals}: the follower can use the realized $y$ to set $X_4$, but not to set $X_1$ (otherwise, this would create a cycle and a timing issue).
\Cref{app:individual_followers} compares individual and population followers empirically.

\paragraph{Observations.}
\begin{wrapfigure}[12]{r}{0.35\textwidth}
    \centering
    \begin{tikzpicture}[
        >=Stealth, xscale=1.05, yscale=0.75,
        var/.style={draw, circle, minimum size=6.5mm, inner sep=0.5pt, font=\footnotesize},
        aux/.style={inner sep=1pt, font=\footnotesize},
        env/.style={draw, rectangle, rounded corners, minimum size=6.5mm, inner sep=1pt, font=\footnotesize},
        arr/.style={->, semithick}
    ]
        \node[env] (E)   at (1.5, 2.5) {$E$};
        \node[var] (X1)  at (0.0, 1.6) {$X_1$};
        \node[aux] (X1t) at (0.75, 1.25) {$X'_1$};
        \node[var] (X4)  at (3.0, 1.6) {$X_4$};
        \node[var] (X5)  at (4, 1.6) {$X_5$};
        \node[var] (X2)  at (1.5, 0.9) {$X_2$};
        \node[aux] (X4t) at (2.3, 0.7) {$X'_4$};
        \node[var] (Y)   at (1.5, -0.5) {$Y$};
        \node[var] (X3)  at (3.0, -0.5) {$X_3$};
        \draw[arr] (E) -- (X1);
        \draw[arr] (E) -- (X4);
        \draw[arr] (X2) -- (X1t);
        \draw[arr] (X1t) -- (X1);
        \draw[arr] (X1) -- (Y);
        \draw[arr] (X2) -- (Y);
        \draw[arr] (Y) -- (X3);
        \draw[arr] (Y) -- (X4t);
        \draw[arr] (X2) -- (X4t);
        \draw[arr] (X4t) -- (X4);
        \draw[arr] (X4) -- (X5);
    \end{tikzpicture}
    \caption{$\G_0$ of \Cref{fig:synthetic_dag} with the auxiliary variables $X'_1$ and $X'_4$.}
    \label{fig:auxiliary_dag} %
\end{wrapfigure}
We consider an individual follower who sets the intervention targets $X_j$, $j \in \CH(E)$, along a fixed topological order of $\G_0$.
To allow for shift interventions (which we consider in \Cref{sec:synthetic_experiments} and \Cref{app:individual_followers}), 
we now introduce auxiliary variables:
for each $X_j$, 
$j \in \CH(E)$, we introduce 
$X'_j$, that is, the value of $X_j$ without the shift intervention on $X_j$, once the upstream targets have been set; e.g., if
$X_j := h(X_{\PA(j)}, \varepsilon_j) + \delta_E$, 
then 
$X'_j := h(X_{\PA(j)}, \varepsilon_j)$ and $X_j = X'_j + \delta_E$.
Before intervening on $X_j$, the follower observes 
the values $x_{O_j}$ and $x'_j$
for a set $O_j$ of variables.
\Cref{fig:auxiliary_dag} shows $\G_0$ from our running example in \Cref{fig:synthetic_dag} together with the auxiliary variables.
The leader does not observe the $X'_j$'s.
We impose two conditions on $O_j$: (i) its variables precede $X_j$ in the topological order and include $\PA(j) \setminus \{E\}$, and (ii) it contains the earlier intervention targets $X_i$ and their observation sets $O_i$.
By (i), the variables in $O_j$ are not affected by the intervention on $X_j$; (ii) means when the follower chooses an intervention, they have access to their earlier observations and choices.
For example, in \Cref{fig:auxiliary_dag}, the sets $O_1 = \{2\}$ and $O_4 = \{1, 2, 3, Y\}$ satisfy (i) and (ii). In this case, an individual follower may manipulate $X_1$ using the observations $x_2$ and $x'_1$, and $X_4$ using $x_1$, $x_2$, $x_3$, $y$, and $x'_4$.
\Cref{fig:auxiliary_dag} shows only $\G_0$ together with the auxiliary variables; using observations of $X_1$ and $X_3$ for the intervention on $X_4$ makes them parents of $X_4$ in the deployment DAG $\G$ (dashed in \Cref{fig:synthetic_dag}).

\paragraph{Backward induction.}
We derive the follower's optimal choices recursively, starting from the last target.
Consider first the last intervention target $X_j$ in the topological order.
Given their observation $(x_{O_j}, x'_j)$, the follower chooses the value $c_j$ for $X_j$ that maximizes $\E_{c_j}[k(f, X, Y) \given X_{O_j} = x_{O_j}, X'_j = x'_j]$, where, as in \eqref{eq:individual}, the subscript indicates that $X_j$ is set to $c_j$; the expectation is over the remaining noise variables.
The optimal value depends on the observation, so there is a function $g_j$ with $c_j = g_j(x_{O_j}, x'_j)$ (we assume throughout that the maximizers exist and can be chosen measurably).
Consider next the second-to-last intervention target $X_l$.
Given their observation $(x_{O_l}, x'_l)$, the follower chooses the value $c_l$ for $X_l$ that maximizes $\E_{c_l}[k(f, X, Y) \given X_{O_l} = x_{O_l}, X'_l = x'_l]$, where now the structural assignment of $X_j$ is replaced by $g_j(X_{O_j}, X'_j)$; since $X_{O_j}$ depends on $c_l$, this takes into account that the choice for $X_j$ adapts to $c_l$.
Again, the optimal value defines a function $g_l(x_{O_l}, x'_l)$.
Continuing backwards along the topological order defines $g_j$ for all $j \in \CH(E)$.\footnote{When the follower chooses some $c_j$, they have already set the intervention targets $X_i$ preceding $X_j$ in the topological order, using functions $g_i$ that the backward induction determines only afterwards (if $c_j$ depended on these functions, the backward induction would be circular).
Which functions these are does not matter for $c_j$: the observed $x_i$ equals $g_i(x_{O_i}, x'_i)$ with $x_{O_i}$ observed by (ii), so knowing $g_i$ gives information only about $x'_i$, that is, about $\varepsilon_i$, which enters only $X_i$, whose value is already known.}
The functions $(g_j)_{j \in \CH(E)}$ therefore describe the optimal behavior of the individual follower.
They can be computed before any value is realized, by maximizing, for every possible observation, the expectation over the noise variables not yet realized at that stage; the topological order only determines when their inputs become available.
As in \Cref{sec:individuals}, they also maximize the population objective $\E_{(g_j)_{j \in \CH(E)}}[k(f, X, Y)]$, where the subscript indicates that the structural assignment of every $X_j$ is replaced by $g_j(X_{O_j}, X'_j)$, because at every stage they maximize the conditional expectation for every value of the observation.

\paragraph{Population follower.}
A population of such individual followers hence acts like a population follower whose structural assignment of $X_j$ is $g_j(X_{O_j}, X'_j)$; since $X'_j$ is a function of $X_{\PA(j)}$ and $\varepsilon_j$, this is a function of $X_{\PA(j)}$, of $X_{A_j}$ with $A_j \coloneqq O_j \setminus \PA(j)$, and of $\varepsilon_j$.
As in \Cref{sec:problem_setup}, $A_j$ must not contain $Y$; that is, the follower may observe $Y$ before setting $X_j$ only if $Y \in \PA(j)$.
The new edges $X_k \to X_j$, $k \in A_j$, follow the topological order, so the augmented graph is acyclic, and $\varepsilon_j$, which is independent of all other noise variables, serves as the exogenous noise of the new assignment.
In conclusion, as in \Cref{sec:individuals}, a population of such individual followers has, from the perspective of the leader and the game's value, the same effect as a single population follower in an augmented graph.

\subsection{A follower objective under which the parents beat the stable blanket} \label{app:parents_beat_sb}
\Cref{thm:PA_vs_SB_game} covers two classes of follower objectives.
For other objectives, the ordering can reverse.
\begin{example}[Parents beat the stable blanket] \label{ex:parents_beat_sb}
Let $\ell$ be the Brier score and consider the following SCM with $\Env = \{0, 1\}$ and $\varepsilon_Y \sim \mathrm{Unif}([0,1])$:\\[0pt]
\begin{minipage}{0.49\textwidth}
\begin{align*}
    X_1 &\coloneqq E \in \{0, 1\}, \\
    Y &\coloneqq \bm{1}\{\varepsilon_Y \leq 1/2 - X_1/6\}, \\
    X_2 &\coloneqq X_1 Y.
\end{align*}
\end{minipage}
\begin{minipage}{0.49\textwidth}
\centering
\vspace*{0.7em}
\input{dag_7}
\end{minipage}\\[6pt]
Then $\PA(Y) = \{1\}$ and $\SB(Y) = \{1, 2\}$, with $f_{\PA(Y)}(x) = 1/2 - x_1/6$ and $f_{\SB(Y)}(x) = (1 - x_1)/2 + x_1 x_2$.
The risks in $e = 0$ and $e = 1$ are $(1/4, 2/9)$ for $f_{\PA(Y)}$ and $(1/4, 0)$ for $f_{\SB(Y)}$, respectively.
Let the follower maximize $\Vfoll_e(f) = -\E_e[f(X)^2]$.
Under $f_{\PA(Y)}$, the follower chooses $e = 1$ ($-1/9 > -1/4$); under $f_{\SB(Y)}$, $e = 0$ ($-1/4 > -1/3$).
The leader's game objective~\eqref{eq:pred_intervention_game} is therefore $2/9$ for $f_{\PA(Y)}$ and $1/4$ for $f_{\SB(Y)}$.
This objective lies outside Cases~(i) and~(ii) of \Cref{thm:PA_vs_SB_game}.
\end{example}
Ex.~\ref{ex:parents_beat_sb} also shows that the optimal leader strategy can depend on the follower utility $\Vfoll$. Since $\DEYforb = \varnothing$, Lem.~\ref{lem:stable_region_opt} yields $R_e(f) \geq R_e(f_{\SB(Y)})$ for all measurable $f$ and all $e \in \Env$. Therefore, for adversarial followers with $\Vfoll_e(f) = R_e(f)$, we know that $f_{\SB(Y)}$ is an optimal leader strategy (see also the proof of \Cref{thm:PA_vs_SB_game}~(i)). For $\Vfoll_e(f) = -\E_e[f(X)^2]$, however, $f_{\SB(Y)}$ is not optimal, since $f_{\PA(Y)}$ is strictly better. Hence, whether $f_{\SB(Y)}$ is optimal depends on $\Vfoll$.

\subsection{Subset-level formulation} \label{app:subset_formulation}

\begin{remark} \label{rem:subset_formulation}
The following subset-level result follows from Lem.~\ref{lem:stable_region_opt}: for all $S \subseteq \DEYforb^C$,
\begin{equation} \label{eq:rem8}
      \sup_{e_1, e_2 \in \Env}
  \E_{e_1} \Big[\ell \big(Y,  \E_{e_2}[Y \given X_S]\big)\Big]
  \geq
  \sup_{e \in \Env}
  R_e(f_{\SB(Y)}).
\end{equation}
By \Cref{thm:worst_case}, \Cref{eq:rem8}
holds for all
$S \subseteq \{1, \ldots, d\}$ under Asm.~\ref{as:A1} or~\ref{as:A2}.  Consequently, under faithfulness,
\citet[Thm.~3.5]{pfisterStabilizingVariableSelection2021}
follows as a corollary, since every subset $S$ d-separating $Y$ and $E$ satisfies $S \subseteq \DEYforb^C$ by part~(a) of their proof.
\end{remark}

\subsection{Worst-case optimal functions}\label{app:worst-case-opt-functions}
In this work, we study predictors $f_S(x) = \E[Y \given X_S = x_S]$ defined by an invariant subset $S$ (the `subset view'). \Cref{thm:worst_case} shows that, under Asm.~\ref{as:A1} or~\ref{as:A2}, $f_{\SB(Y)}$ is worst-case optimal over all measurable functions $f\colon \mathcal{X} \to \mathcal{A}$ (the `function view'). We outline three reasons why the subset view is the right choice in our setting.

(1)
To solve Problem~\eqref{eq:worst-case}, Invariant Risk Minimization (IRM) \citep{arjovskyInvariantRiskMinimization2020} takes a function view and jointly learns a representation $\Phi\colon \mathcal{X} \to \mathcal{H}$ and a predictor $w\colon \mathcal{H} \to \mathcal{A}$. The representation $\Phi$ is required to be invariant in the sense that the optimal predictor $w$ on top of $\Phi$ is the same across environments. While failure cases of IRM are known \citep{rosenfeldRisksInvariantRisk2021, kamathDoesInvariantRisk2021}, we use the $0$--$1$ loss in the following example to make the point that IRM's invariance criterion can be weaker than the conditional invariance $Y \indep E \given X_S$ of Def.~\ref{def:inv_subset}.

\begin{example}[Weak IRM invariance under $0$--$1$ loss] \label{ex:irm-vs-subset}
Consider the following SCM\\[0pt]
\begin{minipage}{0.49\textwidth}
\begin{align*}
    X_1 &\coloneq \bm{1}\{\varepsilon_1 \leq E/10\}, \\
    Y &\coloneq \bm{1}\{\varepsilon_Y \leq 1/4 + X_1/2\}, \\
    X_2 &\coloneq Y \oplus \bm{1}\{\varepsilon_2 \leq E/10\},
\end{align*}
\end{minipage}
\begin{minipage}{0.49\textwidth}
\centering
\vspace*{0.7em}
\input{dag_5}
\end{minipage}\\[6pt]
where $\varepsilon_1, \varepsilon_Y, \varepsilon_2 \sim \mathrm{Unif}([0,1])$, mutually independent, and $\oplus$ denotes addition modulo $2$. Set $\Envtr = \{1, 2, 3\}$ and $\Env = \Envtr \cup \{8\}$.

Subset method: only $S = \{1\}$ is invariant, with $\P_e(Y = 1 \given X_1 = 1) = 3/4$ and $\P_e(Y = 1 \given X_1 = 0) = 1/4$ for all $e \in \Env$. The classifier $f_{\{1\}}$ achieves $3/4$ accuracy in every environment.

IRM: consider the representation $\Phi(x) \coloneq x_2$. Although $\P_e(Y = 1 \given X_2 = 1)$ varies across $\Envtr$, it stays above $1/2$ in every training environment, so the classifier $w = \mathrm{id}$ is Bayes-optimal under the $0$--$1$ loss in all $e \in \Envtr$.
The resulting accuracy ranges from $9/10$ to $7/10$ across $\Envtr$, exceeding the $3/4$ of $f_{\{1\}}$ on average. In the test environment $e = 8$, the accuracy of $w \circ \Phi$ drops to $1/5$.
\end{example}

In Ex.~\ref{ex:irm-vs-subset}, IRM accepts $\Phi$ as invariant because the Bayes-optimal classifier under the $0$--$1$ loss does not change across $\Envtr$, although $Y \notindep E \given \Phi(X)$. This illustrates that function-based invariance criteria can be weaker than the invariance $Y \indep E \given X_S$ of Def.~\ref{def:inv_subset}. Here, $w \circ \Phi$ is outperformed by the subset-based $f_{\{1\}}$ in worst-case accuracy on $\Env$. However, one could tackle this by 
modifying the original IRM as follows: we could replace
the $0$--$1$ loss with a loss strictly consistent for the mean functional in IRM: under such a loss, IRM would reject $\Phi$ since $\E_e[Y \given \Phi(X)]$ varies across $\Envtr$. In this case, the representation $X \mapsto X_1$ would still satisfy the IRM criterion.

(2) The subset view and the function view differ in the class of test interventions they are robust to. For example, if $E$ has two children, the function view typically requires that the joint distribution of these children shifts with $E$ on $\Env$ 
as 
it does on $\Envtr$;
a subset method, by contrast, often comes with guarantees even when the two children of $E$ are intervened on separately; we make this precise with the following example, where the function view exploits a pattern on $\Envtr$ that does not persist on $\Env$.
\begin{example}[Worst-case suboptimality of an invariant $f$] \label{ex:strict-inequality}
Let $\mathcal{Y} = \mathbb{R}$, $\ell(y,q) = (y-q)^2$, and consider the following SCM over $(X_1, X_2, Y, E)$ whose structural assignments on $\Envtr$ are\\[0pt]
\begin{minipage}{0.49\textwidth}
\begin{align*}
    Y &\coloneq \varepsilon_Y, \\
    X_1 &\coloneq \delta_E + Y + \varepsilon_1, \\
    X_2 &\coloneq \delta_E + \varepsilon_2, \\
    \text{where }& \delta_E \in \mathbb{R} \text{ and } \varepsilon_Y, \varepsilon_1, \varepsilon_2 \stackrel{\text{i.i.d.}}{\sim} \mathcal{N}(0,1).
\end{align*}
\end{minipage}
\begin{minipage}{0.49\textwidth}
\centering
\vspace*{1.5em}
\input{dag_4}
\end{minipage}\\[6pt]
Here, $\SB(Y) = \varnothing$, so $f_{\SB(Y)} = \E[Y] = 0$. The representation $\Phi(X) \coloneq X_1-X_2 = Y + \varepsilon_1 - \varepsilon_2$ is independent of $E$ on $\Envtr$, and the predictor $f(x) \coloneq \E[Y \given \Phi = x_1-x_2] = (x_1-x_2)/3$ is optimal for the squared loss on top of $\Phi$. For all $e \in \Envtr$, $R_e(f) = \frac{2}{3}< 1 = R_e(f_{\SB(Y)})$.

On $\Envtr$, $f$ satisfies several invariance properties: (i) $Y \indep E \given f(X)$, and thus $\E_e[Y \given \Phi(X)]$ does not depend on $e \in \Envtr$, so $\Phi$ satisfies the IRM invariance criterion of \citet[Def.~3]{arjovskyInvariantRiskMinimization2020}; (ii) $E \indep f(X)$; (iii) invariant residuals $Y - f(X) \indep E$, the invariance notion used by \citet{saengkyongamExploitingIndependentInstruments2022}; and (iv) invariance of the risks.

We now extend the SCM to $\Env$: for all $e\in\Envtr$, let $\Env$ contain an environment $e'$ in which the structural assignment of $X_1$ is replaced by $X_1\coloneq\tilde\varepsilon_1\sim\mathcal{N}(0,1)$, independent of $(Y,\varepsilon_1,\varepsilon_2)$, while the assignment of $X_2$ is unchanged. Then Asm.~\ref{as:A2} holds. In environment $e'$, $f(X)=(\tilde\varepsilon_1-\delta_{e'}-\varepsilon_2)/3$ is independent of $Y$. Furthermore, 
$R_{e'}(f) = 1 + (2 + \delta_{e'}^2)/9 > 1 = R_{e'}(f_{\SB(Y)}), $
and hence, $\sup_{e \in \Env}R_e(f) > \sup_{e\in\Env} R_e(f_{\SB(Y)})$, that is, Inequality~\eqref{eq:worst_case_statement} is strict. 
\end{example}

In Ex.~\ref{ex:strict-inequality}, $f$ exploits a pattern on $\Envtr$ that does not extend to $\Env$. Our setup in \Cref{sec:problem_setup} makes few structural assumptions on the interventions, and in this regime the subset view is robust under the intervention class considered here. \Cref{thm:worst_case} formalizes this. If the interventions are restricted to a known class, the function view can instead be preferable: anchor regression \citep{rothenhauslerAnchorRegressionHeterogeneous2020}, for instance, is worst-case optimal over a specific class of bounded shift interventions. But in this paper, we focus on the more general case.

(3)
The subset view is arguably more interpretable and transparent: it combines feature selection with standard supervised learning, rather than an opaque function of all covariates. In regulated settings such as credit scoring or clinical decision support, this simplifies auditing and certifying that prohibited features are not used, and makes individual predictions easier to explain. Only the variables in the chosen invariant subset $S \subseteq \{1, \ldots, d\}$ need to be measured at deployment.
Finally, the subset view can make the follower's potential actions
more transparent, as the choice of which subset of covariates is used to construct $f$ influences the followers' behavior: as in \Cref{sec:intro}, using only causal features (such as safety equipment) means a follower can change their prediction only by genuinely changing $Y$, whereas including manipulable non-causal features (such as car color) lets them change the prediction without changing $Y$.

We return to Ex.~\ref{ex:strict-inequality} to show that the bound~\eqref{eq:worst_case_obj} and the game~\eqref{eq:pred_intervention_game} can order
predictors differently.
\begin{remark}[Bound~\eqref{eq:worst_case_obj} versus the game] \label{rem:bound_vs_game}
Consider the SCM of Ex.~\ref{ex:strict-inequality} with $\Env = \{e_0, e_1\}$, where $\delta_{e_0} = \delta_{e_1} = -1$ and $e_1$ additionally replaces the assignment of $X_1$ by $X_1 \coloneqq \tilde{\varepsilon}_1$ as above, so that Asm.~\ref{as:A2} holds, and let the follower maximize $\Vfoll_e(f) = -\E_e[f(X)]$.
For $f(x) = (x_1 - x_2)/3$, we have $\E_{e_0}[f(X)] = 0$ and $\E_{e_1}[f(X)] = 1/3$, so $\Env^*(f) = \{e_0\}$, whereas $\E_e[f_{\SB(Y)}(X)] = 0$ for both $e$, so $\Env^*(f_{\SB(Y)}) = \Env$.
The game objective and the bound are therefore
\begin{align*}
    \sup_{e \in \Env^*(f)} R_e(f) = R_{e_0}(f) &= \tfrac{2}{3},
    & \sup_{e \in \Env} R_e(f) = R_{e_1}(f) &= \tfrac{4}{3}, \\
    \sup_{e \in \Env^*(f_{\SB(Y)})} R_e(f_{\SB(Y)}) &= 1,
    & \sup_{e \in \Env} R_e(f_{\SB(Y)}) &= 1.
\end{align*}
The bound is thus smaller for $f_{\SB(Y)}$, in line with \Cref{thm:worst_case}, while the game objective is smaller for $f$.
\end{remark}

\subsection{Generalization to Ex.~\ref{ex:star_counterexample}} \label{sec:generalized_example}
The following proposition generalizes
Ex.~\ref{ex:star_counterexample}.
\begin{proposition} \label{prop:star_necessary}
Assume $\mathcal{Y} = \mathbb{R}$ and $\ell(y,q)=(y-q)^2$.
Let $\G$ be a DAG over $(X_1, \ldots, X_d, Y, E)$ with $E$ a source node and no edge $E \to Y$, and let $\bar{\G}$ be the induced subgraph of $\G$ on $(X_1, \ldots, X_d, Y)$.
If \eqref{eq:star} fails, then there exists an SCM with graph $\G$ and induced distribution $\P$, and a set of environments $\Env$
such that
(i) there exists $e_0 \in \Env$ under which $\P_{e_0}$ is faithful w.r.t.\ $\bar{\G}$,
(ii) the second part of Asm.~\ref{as:A2} holds (that is, for all $e \in \Env$ there exists an environment $e' \in \Env$ for which every $X_j$, $j \in \CH_\mathrm{int}(Y)$, %
is replaced by exogenous noise independent of all other exogenous noises, with all other structural assignments unchanged), and
(iii) there exists a measurable $f \colon \mathcal{X} \to \mathbb{R}$ with $R_e(f) < R_e(f_{\SB(Y)})$ for all $e \in \Env$.
\end{proposition}

A proof can be found in \Cref{app:proof_prop:star_necessary}.

\subsection[Generalizing the result by Rojas-Carulla et al. (2018)]{Generalizing the result by \citet{rojas-carullaInvariantModelsCausal2018}}
\label{app:rojas}

The following proposition generalizes Thm.~4 by \citet{rojas-carullaInvariantModelsCausal2018} to losses strictly consistent for the mean (and, in particular, to classification).

\begin{proposition}[Adversarial optimality] \label{prop:adversarial}
Let $S$ be an invariant subset of covariates, and let
$N \coloneq \{1, \ldots, d\} \setminus S$. Let $\mu$ be a dominating measure
on $\mathcal{X}$ such that $\mu = \mu_S \otimes \mu_N$ for some
measures $\mu_S$ on $\mathcal{X}_S$ and $\mu_N$ on $\mathcal{X}_N$
(such as Lebesgue or counting measures). Define $\mathcal{Q}_S$ to be
the family of all joint distributions $\Q$ of $(X,Y)$ on
$\mathcal{X} \times \mathcal{Y}$ such that
\begin{enumerate}[label=(\roman*)]
  \item %
  $\Q(Y\in\cdot\given X_S) = \P(Y\in\cdot\given X_S)$ \quad $\Q$-almost surely,
  \item the marginal distribution of $X$ under $\Q$ is absolutely
        continuous with respect to $\mu$.
\end{enumerate}
Assume $\sup_{\Q\in\mathcal{Q}_S}\E_\Q[\ell(Y,f_S(X))]<\infty$.
Then the predictor $f_S$ satisfies
$$
  f_S \in
  \argmin_{\substack{f:\,\mathcal{X}\to\mathcal{A}\\
                      \mathrm{measurable}}}\;
  \sup_{\Q \in \mathcal{Q}_S}\,
  \E_{(X,Y)\sim\Q}\!\big[\ell(Y, f(X))\big].
$$
\end{proposition}
This result is insufficient to answer the question in \Cref{sec:worst_case} whether $f_{\SB(Y)}$ is worst-case optimal among all measurable predictors, as, in our setting, the considered distributions do not change with $S$.
A proof can be found in \Cref{app:adversarial}.

\subsection{The stable blanket under hidden variables} \label{app:stable_blanket_hidden}

Our main results assume a fully observed DAG (\Cref{sec:problem_setup}).
For the light tunnel experiment (\Cref{sec:new_chamber_experiments}), we hypothesize a hidden common cause of the response and the sensor measurements (\cref{fig:causal_chambers_dags}, right).
Here, we use the stable frontier construction of \citet[Secs.~3 and~6]{xiangStableBlanketHidden2026} to extend Lem.~\ref{lem:stable_region_opt} and \Cref{thm:PA_vs_SB_game} to graphs with hidden variables and to justify the subset used as the stable blanket in our experiment.
Unlike in the fully observed case, the graph alone 
does not necessarily
determine the most predictive subset among those whose invariance follows from the Markov property.

\paragraph{Graphs with hidden variables.}
Consider an acyclic SCM over $(X, H, Y, E)$ with hidden variables $H$, where, as in \Cref{sec:problem_setup}, $E$ is a source node whose children are the intervention targets.
We assume that there is no edge from $E$ to $Y$ or into $H$.
For simplicity, we assume that the follower's interventions do not add edges to the graph (in the notation of \Cref{sec:problem_setup}, $A_j = \varnothing$ for all $j \in \CH(E)$).
We work with the latent projection of this graph onto $(X, Y, E)$ \citep[Sec.~3.1.3]{vermaGraphicalAspectsCausal1993}.
This is an acyclic directed mixed graph \citep[ADMG;][]{richardsonMarkovPropertiesAcyclic2003} $\mathcal{M}$ that contains the edge $i \to j$ if the DAG contains a directed path from $i$ to $j$ whose intermediate nodes, if any, are all hidden, and the edge $i \leftrightarrow j$ if the DAG contains a path $i \leftarrow \cdots \leftarrow H_k \to \cdots \to j$ whose intermediate nodes are all hidden.
In $\mathcal{M}$, a node $k$ is a collider on a path if both adjacent edges have an arrowhead at $k$ (as in $\to k \leftarrow$ or $\leftrightarrow k \leftrightarrow$), and a path is m-blocked by $S$ if it contains a non-collider in $S$ or a collider $k$ such that neither $k$ nor any of its descendants is in $S$ \citep[Sec.~2.1]{richardsonMarkovPropertiesAcyclic2003}.
We write $E \msep Y \given S$ if every path between $E$ and $Y$ in $\mathcal{M}$ is m-blocked by $S$ (m-separation).
For disjoint sets of observed nodes, m-separation in $\mathcal{M}$ is equivalent to d-separation in the DAG \citep[Prop.~1]{richardsonNestedMarkovProperties2023}, so the Markov property of the SCM yields
\begin{equation*}
    E \msep Y \given S \quad \Longrightarrow \quad E \indep Y \given X_S,
\end{equation*}
that is, a set $S$ that m-separates $E$ and $Y$ is invariant (Def.~\ref{def:inv_subset}).
We write $\DIS(j)$ for the district of $j$ in $\mathcal{M}$, that is, the set containing $j$ and all nodes connected to $j$ by paths of bidirected edges, and $\DIS(A) = \bigcup_{j \in A} \DIS(j)$ for a node set $A$.

\paragraph{Stable frontiers.}
Following \citet{xiangStableBlanketHidden2026}, we now construct an analogue of the stable blanket in $\mathcal{M}$. Intuitively, we want a set of covariates that m-separates $E$ and $Y$ and, like the Markov blanket, is informative about $Y$ without containing irrelevant covariates.
Without hidden variables, the stable blanket \eqref{eq:stable_blanket_decomp} satisfies these requirements: it equals the Markov blanket of $Y$ in the graph obtained by removing $\DEYforb$, that is, the intervened children of $Y$ and their descendants, and the removal ensures that the resulting set d-separates $E$ and $Y$ (Lem.~\ref{lemma:d-sep}).
With hidden variables, we proceed in the same way: we determine which paths from $E$ to $Y$ the Markov blanket of $Y$ in $\mathcal{M}$ does not block, and we remove covariates so that the Markov blanket of the remaining graph blocks them.
By \citet[Prop.~3.3]{xiangStableBlanketHidden2026}, in an ADMG the Markov blanket of $Y$ within a single environment (that is, for fixed $E$) is
\begin{equation} \label{eq:mb_admg_new}
    \Big( \DIS(Y) \cup \PA(\DIS(Y)) \cup \DIS(\CH(Y)) \cup \PA(\DIS(\CH(Y))) \Big) \setminus \{Y, E\}.
\end{equation}
As in \Cref{sec:stable_blanket}, the Markov blanket of $Y$ in $\mathcal{M}$ 
does not necessarily
m-separate $E$ and $Y$: the paths from $E$ to $Y$ that it does not block are exactly those of the two forms
\begin{equation} \label{eq:two_forms_new}
    E \to X_{j} \leftrightarrow \cdots \leftrightarrow Y
    \qquad \text{and} \qquad
    E \to X_{j} \leftrightarrow \cdots \leftrightarrow X_{k} \leftarrow Y,
\end{equation}
including the shortest cases $E \to X_j \leftrightarrow Y$ and $E \to X_j \leftarrow Y$; nodes of such a path may also be joined by directed edges, but the path uses the bidirected edges.
Every covariate on such a path is a collider on it and lies in $\DIS(Y)$ or $\DIS(\CH(Y))$, and hence in \eqref{eq:mb_admg_new}.
Conversely, if a path from $E$ to $Y$ contains a non-collider, then the non-collider closest to $Y$ lies in \eqref{eq:mb_admg_new} and blocks the path; otherwise, all intermediate nodes are colliders, so consecutive ones are joined by bidirected edges and the path has one of the two forms.

A covariate on a path of the form \eqref{eq:two_forms_new} is called an admissible collider if it has no directed edge to either of its neighbors on the path and no directed path to $Y$.\footnote{In the terminology of \citet{xiangStableBlanketHidden2026}, paths of the first and second form in \eqref{eq:two_forms_new} correspond to intervened districts (Def.~6.1) and intervened sub-districts (Def.~3.5), and admissible colliders on them to district colliders (Def.~6.2) and sub-district colliders (Def.~3.6); we use one term for both forms because we treat them alike.}
The first condition makes the covariate a collider on every path through the same nodes, and the second ensures that removing the covariate and its descendants keeps all ancestors of $Y$ in the graph.
Following \citet[Setting~6.1]{xiangStableBlanketHidden2026}, we assume that every path of the form \eqref{eq:two_forms_new} contains an admissible collider (which is satisfied in \cref{sec:new_chamber_experiments}).
This assumption is necessary for the existence of a subset that m-separates $E$ and $Y$ \citep[Props.~3.8 and~6.3]{xiangStableBlanketHidden2026}, and the proof of Prop.~\ref{prop:frontier_new} below shows that it is also sufficient.
It fails, for example, if $E \to X_1 \to Y$ and $X_1 \leftrightarrow Y$: the only covariate $X_1$ on the path $E \to X_1 \leftrightarrow Y$ has a directed edge to $Y$ and thus violates both conditions, and indeed no subset m-separates $E$ and $Y$.
\citet{xiangStableBlanketHidden2026} constructs the analogue of the stable blanket in three steps.
\begin{enumerate}
    \item Select an admissible collider set, that is, a set $O$ of admissible colliders that contains at least one admissible collider of every path of the form \eqref{eq:two_forms_new} \citep[Def.~6.4]{xiangStableBlanketHidden2026}.
    \item Remove $O$ and its descendants: let $N(O) \subseteq \{1, \ldots, d\}$ denote the retained covariates, that is, those that are neither in $O$ nor descendants of $O$, and let $\mathcal{M}_O$ denote the subgraph of $\mathcal{M}$ induced by $N(O) \cup \{Y, E\}$.
    \item Take the Markov blanket of $Y$ in $\mathcal{M}_O$: this is the stable frontier $S(O)$ for $O$, given by \eqref{eq:mb_admg_new} with all sets computed in $\mathcal{M}_O$ \citep[Def.~3.12, Thm.~3.14, and Sec.~6]{xiangStableBlanketHidden2026}.
\end{enumerate}
For a fixed $O$, $f_{S(O)}$ has the smallest risk in every environment among all predictors based on $X_{N(O)}$ (Prop.~\ref{prop:frontier_new}~(ii) below), but different choices of $O$ retain different covariates and can yield different frontiers.
Ex.~\ref{ex:frontiers_new} below goes through the three steps for two such choices.
If all choices yield the same frontier, \citet{xiangStableBlanketHidden2026} calls it the stable blanket.
Without hidden variables, the only paths of the form \eqref{eq:two_forms_new} are $E \to X_j \leftarrow Y$ with $j \in \CH_\mathrm{int}(Y)$, so $O = \CH_\mathrm{int}(Y)$ is the only admissible collider set and $S(O) = \SB(Y)$.

\paragraph{Prediction guarantees.}
The following result extends Lem.~\ref{lem:stable_region_opt} and \Cref{thm:PA_vs_SB_game} to graphs with hidden variables, for a fixed admissible collider set.
$\PA(Y)$ denotes the parents of $Y$ in $\mathcal{M}$.
\begin{proposition} \label{prop:frontier_new}
Let $O$ be an admissible collider set in $\mathcal{M}$.
\begin{enumerate}[label=(\roman*)]
    \item $N(O)$ and $S(O)$ are invariant, and $\PA(Y) \subseteq S(O) \subseteq N(O)$.
    \item For all $\sigma(X_{N(O)})$-measurable $f \colon \mathcal{X} \to \mathcal{A}$ and all $e \in \Env$, we have $R_e(f) \geq R_e(f_{S(O)})$.
    \item If $\PA(Y)$ is invariant, which holds if $E \msep Y \given \PA(Y)$, then \Cref{thm:PA_vs_SB_game} holds with $S(O)$ in place of $\SB(Y)$.
\end{enumerate}
\end{proposition}
In a DAG, $E \dsep{\G} Y \given \PA(Y)$ always holds because $E$ is a source node and not a parent of $Y$.
In an ADMG, this can fail, so the invariance of $\PA(Y)$ in (iii) is an assumption.\footnote{For example, in the ADMG with edges $E \to X_1 \to X_2 \leftrightarrow Y$ and $X_2 \to X_3 \to Y$, conditioning on $\PA(Y) = \{X_3\}$ opens the path $E \to X_1 \to X_2 \leftrightarrow Y$ at the collider $X_2$.}

If Asm.~\ref{as:A1} holds with $S(O)$, $O \cup \DE(O)$, and $N(O)$ in place of $\SB(Y)$, $\DEYforb$, and $\DEYforb^C$, then $f_{S(O)}$ solves Problem~\eqref{eq:worst-case}, that is, $\sup_{e \in \Env} R_e(f) \geq \sup_{e \in \Env} R_e(f_{S(O)})$ for all measurable $f$.
The proof of \Cref{thm:worst_case} applies unchanged, since it uses only the invariance of both sets and the identity $\E_e[Y \given X_{N(O)}] = f_{S(O)}(X)$.
We leave extensions of Asm.~\ref{as:A2} and Lem.~\ref{lem:augmentation} open.

\paragraph{The graph
does not necessarily determine the most predictive m-separating subset.}
Two admissible collider sets can yield frontiers that are not nested, and the graph alone need not determine the most predictive subset among those whose invariance follows from it.
\begin{example} \label{ex:frontiers_new}
Consider the SCM with parameter $a > 0$ and ADMG $\mathcal{M}$ shown on the right:\\[0pt]
\begin{minipage}{0.55\textwidth}
\begin{align*}
    Y &\coloneqq \varepsilon_Y, & X_1 &\coloneqq a Y + \varepsilon_1, \\
    X_2 &\coloneqq \delta_E + \varepsilon_2, & X_3 &\coloneqq X_2 + \varepsilon_3,
\end{align*}
\end{minipage}
\hfill
\begin{minipage}{0.4\textwidth}
\centering
\begin{tikzpicture}[>=Stealth, semithick, scale=1.15]
    \node (Y)  at (0,1) {$Y$};
    \node (X1) at (1.3,1) {$X_1$};
    \node (X2) at (2.6,1) {$X_2$};
    \node (E)  at (3.9,1) {$E$};
    \node (X3) at (2.6,0) {$X_3$};
    \draw[->] (Y) -- (X1);
    \draw[<->, dashed] (X1) -- (X2);
    \draw[->] (E) -- (X2);
    \draw[->] (X2) -- (X3);
    \draw[<->, dashed] (Y) -- (X3);
\end{tikzpicture}
\end{minipage}\\[1pt]
where $\delta_E \in \mathbb{R}$ shifts $X_2$ depending on $E$, and the noise variables are centered, jointly Gaussian, and independent of $E$, with unit variances and the only nonzero cross-covariances $\Cov(\varepsilon_1, \varepsilon_2) = \Cov(\varepsilon_Y, \varepsilon_3) = 1/2$ due to hidden common causes.
The only path of the form \eqref{eq:two_forms_new} is $E \to X_2 \leftrightarrow X_1 \leftarrow Y$.
Both $X_2$ and $X_1$ are admissible colliders on it.
Choosing $O_A \coloneqq \{2\}$ yields the frontier $S_A \coloneqq S(O_A) = \{1\}$; choosing $O_B \coloneqq \{1\}$ yields $S_B \coloneqq S(O_B) = \{2, 3\}$.
Both $S_A$ and $S_B$ m-separate $E$ and $Y$ and are thus invariant.
For the squared loss, the two predictors and their risks are
\begin{equation*}
\begin{aligned}
    f_{S_A}(X) &= \frac{a X_1}{1+a^2}, & R_e(f_{S_A}) &= \frac{1}{1+a^2}, \\
    f_{S_B}(X) &= \frac{X_3 - X_2}{2}, & R_e(f_{S_B}) &= \frac{3}{4}.
\end{aligned}
\end{equation*}
For $a = 2$, $f_{S_A}$ has the smaller risk, and for $a = 1/2$, $f_{S_B}$ has the smaller risk, while the graph is the same in both cases.
\end{example}
Of the eight subsets of $\{1, 2, 3\}$ in Ex.~\ref{ex:frontiers_new}, exactly $\varnothing$, $\{1\}$, $\{2\}$, and $\{2, 3\}$ m-separate $E$ and $Y$, and the predictors of $\varnothing$ and $\{2\}$ have risk $1$, so the most predictive m-separating subset is $S_A$ or $S_B$, depending on $a$.\footnote{The two frontiers in Ex.~\ref{ex:frontiers_new} cannot be merged into an m-separating subset, but an invariant function can use both.
With $Z \coloneqq X_3 - X_2$, the distribution of $(Y, X_1, Z)$ does not depend on $e$, and $g(X) \coloneqq \E[Y \given X_1, Z] = (3a X_1 + 2Z)/(3a^2 + 4)$ has risk $3/(3a^2 + 4)$, which is smaller than both $R_e(f_{S_A})$ and $R_e(f_{S_B})$ for every $a > 0$; see \Cref{app:worst-case-opt-functions} for a discussion of such function-based predictors.}

\begin{wrapfigure}{r}{0.4\textwidth}
    \vspace{-\baselineskip}
    \centering
    \input{dag_6}
    \vspace{-\baselineskip}
\end{wrapfigure}
\paragraph{The light tunnel.}
We now apply the above to the hypothesized graph of the light tunnel, reproduced on the right from \cref{fig:causal_chambers_dags}.
In the corresponding ADMG $\mathcal{M}$, the nodes $Y, \mathtt{ir\_2}, \mathtt{vis\_2}, \mathtt{ir\_3}, \mathtt{vis\_3}$ are pairwise connected by bidirected edges, $\mathtt{RGB}$ is a parent of each of them, and $Y$ and $E$ are parents of $\mathtt{ir\_3}$ and $\mathtt{vis\_3}$.
The assumption that every path of the form \eqref{eq:two_forms_new} contains an admissible collider holds: every such path starts with $E \to \mathtt{ir\_3}$ or $E \to \mathtt{vis\_3}$, and because $\mathtt{ir\_3}$ and $\mathtt{vis\_3}$ have no children, each is an admissible collider on every such path containing it.
Hence $O^* \coloneqq \{\mathtt{ir\_3}, \mathtt{vis\_3}\}$ is an admissible collider set.
Conversely, the paths $E \to \mathtt{ir\_3} \leftrightarrow Y$ and $E \to \mathtt{vis\_3} \leftrightarrow Y$ contain no other covariate, so every admissible collider set contains $O^*$.
The retained covariates are $N(O^*) = \{\mathtt{RGB}, \mathtt{ir\_2}, \mathtt{vis\_2}\}$, and in $\mathcal{M}_{O^*}$ we have
\begin{equation*}
    \DIS(Y) = \{Y, \mathtt{ir\_2}, \mathtt{vis\_2}\}, \qquad \PA(\DIS(Y)) = \{\mathtt{RGB}\}, \qquad \CH(Y) = \varnothing,
\end{equation*}
so $S(O^*) = \{\mathtt{RGB}, \mathtt{ir\_2}, \mathtt{vis\_2}\}$.
Other admissible collider sets yield smaller frontiers; for example, $O^* \cup \{\mathtt{ir\_2}\}$ yields $\{\mathtt{RGB}, \mathtt{vis\_2}\}$.
The frontier is thus not unique, so there is no stable blanket in the sense of \citet{xiangStableBlanketHidden2026}, but $S(O^*)$ contains every other frontier (in \Cref{sec:new_chamber_experiments}, we thus write $\SB(Y)$ for this largest frontier in a slight abuse of notation).
By Prop.~\ref{prop:frontier_new}, $f_{\SB(Y)}$ has the smallest risk in every environment among all predictors that do not use $\mathtt{ir\_3}$ or $\mathtt{vis\_3}$.
Moreover, every path from $E$ to $Y$ starts with $E \to \mathtt{ir\_3}$ or $E \to \mathtt{vis\_3}$ followed by an edge pointing into that node, so a set m-separates $E$ and $Y$ if and only if it contains neither $\mathtt{ir\_3}$ nor $\mathtt{vis\_3}$.
Hence, every m-separating subset is contained in $N(O^*)$, and $f_{\SB(Y)}$ also has the smallest risk in every environment among all predictors based on m-separating subsets.
In particular, $\PA(Y) = \{\mathtt{RGB}\}$ is invariant, so $f_{\SB(Y)}$ weakly dominates $f_{\mathtt{RGB}}$ in the game \eqref{eq:pred_intervention_game} for both follower classes of \Cref{thm:PA_vs_SB_game}.
Consistent with this, $\hat{f}_{\mathtt{RGB}, \mathtt{ir\_2}, \mathtt{vis\_2}}$ has a lower MSE than $\hat{f}_{\mathtt{RGB}}$ while the MSE of both predictors remains essentially constant under the follower's interventions, and SC recovers exactly this subset from data (\cref{fig:follower_chambers}).

\subsection{Discussion of assumptions} \label{app:assumptions_discussion}

\paragraph{The assumption on $A_j$ in Lem.~\ref{lem:augmentation}.}
Lem.~\ref{lem:augmentation} assumes $A_j \subseteq \DEforbg{\G_0}^C \setminus \{Y, E\}$ for all $j \in \CH_{\G_0}(E)$ and that $\G$ is acyclic.
The sets $A_j$ restrict only the new parents of an intervention target $X_j$: an added edge $X_k \to X_j$ means that the new mechanism of $X_j$ uses the realized value of $X_k$, so $A_j$ specifies which new inputs are available before $X_j$ is set.
In \Cref{fig:synthetic_dag}, the new mechanism of $X_4$ may use $X_1$ and $X_3$ (dashed edges), but not $X_5$, which is realized only after $X_4$ is set.
None of $X_3$, $X_4$, and $X_5$ can become a new parent of $X_1$, since they are consequences of $X_1$ and the added edge would create a cycle.
In this example, the restriction to $\DEforbg{\G_0}^C$ excludes nothing beyond acyclicity, because $X_4$ and $X_5$ are downstream of $X_1$, and $X_5$ of $X_4$.

The assumption is realistic for actions chosen before their consequences are observed, such as changes to an application before submitting it or a treatment chosen from baseline measurements, but less so for adaptive settings, such as changing one feature after observing the system's response to a previous change.
Excluding $Y$ as a new parent is plausible when $Y$ is latent or realized later, but not when followers act after observing the outcome.
If the restriction on $A_j$ fails, the results of \cref{sec:strategiesinvsets} and \cref{sec:optimality} other than Lem.~\ref{lem:augmentation} still apply to the acyclic $\G$ under their stated assumptions, and $\SBg{\G} \subseteq \SBg{\G_0}$ still holds because new edges only enter intervention targets.
If the inclusion is strict, however, $\SBg{\G_0}$ contains a descendant of an intervened child $X_i$ of $Y$, and conditioning on it opens the path $E \to X_i \leftarrow Y$, so a leader who reads the stable blanket off $\G_0$ uses a set that is in general no longer invariant in $\G$.

\paragraph{Asm.~\ref{as:A1} and~\ref{as:A2}.}
Asm.~\ref{as:A1} and~\ref{as:A2} do not restrict the intervention model of \Cref{sec:problem_setup}; they require that $\Env$ contains, for each environment, a corresponding strong intervention, not that every environment is one.
Only \Cref{thm:worst_case} and its extensions in the appendix use them; Prop.~\ref{prop:asymptotic} uses a relaxed form.

Asm.~\ref{as:A1} requires, for each $e$, an environment $e'$ that acts like $e$ on $X_{\SB(Y)}$ (same marginal distribution) and in which $X_{\DEYforb}$ is uninformative about $Y$ given $X_{\DEYforb^C}$.
If $\Env$ contains, for example, only bounded soft or additive interventions, $X_{\DEYforb}$ can be informative about $Y$ given $X_{\DEYforb^C}$ in every environment, so no $e'$ satisfies condition~(ii) of Asm.~\ref{as:A1}.
Asm.~\ref{as:A2} additionally restricts the graph and is plausible when (a) every intervened child of $Y$ can be set to exogenous noise or a fixed value, (b) all other mechanisms are retained, and (c) no non-intervened child of $Y$ is downstream of an intervened child.
In a health setting, for example, Asm.~\ref{as:A2} holds if the children of a health outcome $Y$ are sensor readings of $Y$ or questionnaires answered by the patients themselves, manipulating one does not affect the others, and $\Env$ contains, for each $e$, the environment that sets all intervened children jointly to independent noise while leaving the other mechanisms of $e$ unchanged.

Prop.~\ref{prop:star_necessary} shows for regression that for every DAG violating~\eqref{eq:star}, there exists an SCM with this DAG for which $f_{\SB(Y)}$ is not worst-case optimal.
At the same time, for every DAG there is an SCM with that graph for which $f_{\SB(Y)}$ is worst-case optimal: let $Y$ be a deterministic function of $X_{\PA(Y)} \subseteq X_{\SB(Y)}$ (constant if $\PA(Y) = \varnothing$), so that $f_{\SB(Y)} = Y$ almost surely in every environment.
Thus, \eqref{eq:star} cannot be dropped from a worst-case guarantee for $f_{\SB(Y)}$ that holds for all SCMs with a given graph, but the graph alone cannot decide worst-case optimality for a given SCM.

\appsection{Stabilized regression and classification}

\subsection{Invariance tests for classification} \label{app:invariance_tests}

We now provide detailed descriptions of the invariance tests that can be used for the methods described in \Cref{sec:learning_subset_predictors}. All tests target the conditional independence $H_{0,S}\colon Y \indep E \given X_S$ and return a $p$-value used as the invariance score $s_\mathrm{inv}(S)$. Each test can be used with different base models (such as logistic regression, random forests, or gradient boosting).

\subsubsection{Residual-based tests}

Residual-based tests fit a model $\hat{f}$ for $\E[Y \given X_S]$ on data pooled across environments, compute residuals $R = Y - \hat{f}(X_S)$, and test whether the distribution of $R$ depends on $E$.

\paragraph{Invariant Residual Distribution Test (IRD).}
Tests whether the mean of $R$ varies across environments using a one-way ANOVA $F$-test \citep[adapted from][]{heinze-demlInvariantCausalPrediction2018}.

\paragraph{\textsc{tram}-GCM.}
Applies the generalized covariance measure \citep{shahHardnessConditionalIndependence2018} to the residuals of both predicting $Y$ from $X_S$ and $E$ from $X_S$, testing whether their expected product is zero \citep{kookModelbasedCausalFeature2023}.

\subsubsection{Predictive-performance-based tests}

Predictive-performance-based tests check whether $E$ carries additional predictive information about $Y$ beyond $X_S$, or vice versa.

\paragraph{Invariant Target Prediction Test (ITP).}
Compares the ROC AUC of a classifier predicting $Y$ from $(X_S, E)$ against one using $X_S$ and permuted observations of $E$ via DeLong's test \citep{delongComparingAreasTwo1988} for correlated ROC curves \citep{heinze-demlInvariantCausalPrediction2018, salas-porrasIdentifyingCausesPyrocumulonimbus2022}. A significant AUC decrease upon permuting $E$ rejects invariance.

\paragraph{Invariant Environment Prediction Test (IEP).}
Reverses the direction and tests whether $Y$ helps predict $E$ given $X_S$. It compares the cross-entropy loss of a model predicting $E$ from $(X_S, Y)$ against one using $X_S$ and permuted observations of $Y$ via a paired $t$-test \citep[adapted from][]{heinze-demlInvariantCausalPrediction2018}.

\subsection{Algorithm} \label{app:SC_alg}

Stabilized regression \citep{pfisterStabilizingVariableSelection2021} uses a two-stage procedure to determine the weights $\hat{\omega}_S$ for $S \subseteq \{1, \ldots, d\}$. 
We now propose
stabilized classification, which is directly ported from stabilized regression and proceeds in the same way.

\paragraph{Stage 1: invariance.} First, we estimate the collection of invariant subsets by computing an invariance score $s_\mathrm{inv}(S)$ for each subset $S \subseteq \{1, \ldots, d\}$. Then, we define the empirically invariant subsets as $\hat{\mathcal{I}}_{\Envtr} \coloneq \big\{S \subseteq \{1, \ldots, d\} \given s_\mathrm{inv}(S) \geq \alpha_\mathrm{inv} \big\}$, where the cutoff $\alpha_\mathrm{inv} \in 
(0,1)$ is a tuning parameter. \Cref{app:invariance_tests} discusses suitable scores $s_\mathrm{inv}(S)$ for this step. Because iterating over all $2^d$ subsets becomes computationally prohibitive even for moderate $d$, an optional variable screening step (e.g., using $\ell_1$-penalized logistic regression) can be applied to reduce the candidate covariate set before computing the invariance scores.

\paragraph{Stage 2: prediction performance.}
We estimate the subsets with best prediction performance among those in $\hat{\mathcal{I}}_{\Envtr}$. We evaluate the prediction performance of each $S \in \hat{\mathcal{I}}_{\Envtr}$ using a score $s_\mathrm{pred}(S)$, such as the negative binary cross-entropy loss on the training data. We then define the final set of selected subsets as $\hat{\mathcal{C}}_{\Envtr} \coloneq \big\{ S \in \hat{\mathcal{I}}_{\Envtr} \given s_\mathrm{pred}(S) \geq c(\alpha_\mathrm{pred}) \big\}$, where $c(\alpha_\mathrm{pred})$ is a cutoff threshold. Following \citet{pfisterStabilizingVariableSelection2021}, we determine $c(\alpha_\mathrm{pred})$ using a bootstrap procedure. Let $\hat{S}_\mathrm{max} \coloneq \argmax_{S \in \hat{\mathcal{I}}_{\Envtr}} s_\mathrm{pred}(S)$ denote the invariant subset with best prediction performance. We generate $B$ bootstrap samples from the training data, compute the prediction scores for $\hat{S}_\mathrm{max}$ on each sample, and set $c(\alpha_\mathrm{pred})$ to the $\alpha_\mathrm{pred}$-quantile of these bootstrap scores. Therefore, the hyperparameters $\alpha_\mathrm{inv}, \alpha_\mathrm{pred} \in (0,1)$ control the size of the ensemble. Finally, we use uniform weights $\hat{\omega}_S$ over the selected subsets in $\hat{\mathcal{C}}_{\Envtr}$: $\hat{\omega}_S \coloneq \frac{1}{|\hat{\mathcal{C}}_{\Envtr}|}$ for $S \in \hat{\mathcal{C}}_{\Envtr}$ and $\hat{\omega}_S \coloneq 0$ otherwise. Algorithm~\ref{alg:SC} summarizes the full procedure.

\begin{algorithm}[t]
\caption{Stabilized classification (SC) \label{alg:SC}}
\begin{algorithmic}[1]
\Require data $\mathcal{D} = (x_i, y_i, e_i)_{i=1}^n$ from $\Envtr$; hyperparameters $\alpha_\mathrm{inv}, \alpha_\mathrm{pred} \in (0,1)$, $B \in \mathbb{N}$
\State (optional) perform variable screening to reduce $d$ \label{line:var_screening}
\For{all $S \subseteq \{1, \ldots, d\}$}
    \State compute invariance score $s_\mathrm{inv}(S)$ using $\mathcal{D}$
\EndFor
\State $\hat{\mathcal{I}}_{\Envtr} \gets \big\{S \subseteq \{1, \ldots, d\} \given s_\mathrm{inv}(S) \geq \alpha_\mathrm{inv} \big\}$ \Comment{invariance}
\For{all $S \in \hat{\mathcal{I}}_{\Envtr}$}
    \State fit base classifier $\hat{f}_S$ on $\mathcal{D}$
    \State compute prediction score $s_\mathrm{pred}(S)$ on $\mathcal{D}$
\EndFor
\State $\hat{S}_\mathrm{max} \gets \argmax_{S \in \hat{\mathcal{I}}_{\Envtr}} s_\mathrm{pred}(S)$
\State compute cutoff $c(\alpha_\mathrm{pred})$ as the $\alpha_\mathrm{pred}$-quantile of $B$ bootstrap scores for $\hat{S}_\mathrm{max}$
\State $\hat{\mathcal{C}}_{\Envtr} \gets \big\{S \in \hat{\mathcal{I}}_{\Envtr} \given s_\mathrm{pred}(S) \geq c(\alpha_\mathrm{pred}) \big\}$ \Comment{prediction performance}
\State assign weights $\hat{\omega}_S \gets \frac{1}{|\hat{\mathcal{C}}_{\Envtr}|} \bm{1}\big\{S \in \hat{\mathcal{C}}_{\Envtr} \big\}$
\item[\textbf{Output:}] base estimators $\hat{f}_S$ and weights $\hat{\omega}_S$ defining the ensemble $\hat{f}^{\mathrm{SC}}$
\end{algorithmic}
\end{algorithm}

\appsection{Proofs}\label{app:proofs}

The following lemma serves as a generalization of the orthogonality of the conditional expectation (used for regression ($Y \in \mathbb{R}$) with the MSE) to regression and binary classification with loss functions strictly consistent for the mean functional. It is used for several results.

\begin{lemma}
\label{lem:bayes_restricted}
Let $Y$ and $W$ be random variables on a probability space
$(\Omega, \mathcal{F}, \P)$, with $Y$ taking values in $\mathcal{Y}$
and integrable.
Let $Z$ be a $\sigma(W)$-measurable random variable
taking values in $\mathcal{A}$. If $\ell$ is strictly consistent for the mean
functional, then
\begin{equation*}
  \mathbb{E}\bigl[\ell(Y,Z)\bigr]
  \;\geq\;
  \mathbb{E}\bigl[\ell\bigl(Y,\mathbb{E}[Y\given W]\bigr)\bigr],
\end{equation*}
where $\E[\cdot]$ denotes expectation under $\P$.
If both expectations are finite, equality holds if and only if
$Z=\mathbb{E}[Y\given W]$ almost surely.
\end{lemma}

\begin{proof}[Proof of Lem.~\ref{lem:bayes_restricted}]
For all $q \in \mathcal{A}$ and probability distributions $F$ on $\mathcal{Y}$ with finite mean $\E_{Y \sim F}[Y]$, define
$$
L(F, q) \coloneq \E_{Y \sim F}[\ell(Y, q)].
$$
Let $F_W$ denote the conditional distribution of $Y$ given $W$ \citep[Thm.~8.5]{kallenbergFoundationsModernProbability2021}. Its mean equals $\E[Y \given W]$ almost surely. Conditional on $W$, the random variable $Z$ acts as a constant since $Z$ is $\sigma(W)$-measurable, so almost surely, 
$$
\E[\ell(Y, Z) \given W] = L(F_W, Z).
$$

By strict consistency of $\ell$, for every distribution $F$ with finite mean,
$q \mapsto L(F, q)$ is uniquely minimized at $q = \E_{Y \sim F}[Y]$.
Since this holds for every $F$, it holds pointwise for the random
conditional distribution $F_W$ for almost every $\omega \in \Omega$.
Thus, almost surely, 
$$
  L(F_W, Z) \geq L(F_W, \E[Y \given W])
  =
  \E\big[\ell(Y, \E[Y \given W]) \given W\big],
$$
with equality if and only if $Z = \E[Y \given W]$ almost surely. Taking
expectations, $\E[\ell(Y,Z)] \geq \E[\ell(Y,\E[Y \given W])]$. For the equality
case, assume both expectations are finite and equal. Then the
nonnegative random variable
$D := \E[\ell(Y,Z)\given W] - \E[\ell(Y,\E[Y \given W])\given W]$ satisfies
$\E[D] = 0$, hence $D = 0$ almost surely, and pointwise strictness
gives $Z = \E[Y \given W]$ almost surely. The converse is immediate.
\end{proof}

\subsection{Proof of Lem.~\ref{lem:stable_region_opt}}
We first state and prove a lemma. The d-separation statement
for $\SB(Y)$ is also shown in the proof of
\citet[Thm.~3.5]{pfisterStabilizingVariableSelection2021}; we include
the argument for completeness and extend it to $\DEYforb^C$.

\begin{lemma}\label{lemma:d-sep}
The subsets $\DEYforb^C$ and $\SB(Y)$ satisfy the d-separations
$$
  Y \dsep{\G} E \given \DEYforb^C
  \qquad\text{and}\qquad
  Y \dsep{\G} E \given \SB(Y).
$$
\end{lemma}
\begin{remark} \label{rem:SB_in_different_DAGs}
    The statement and proof of Lem.~\ref{lemma:d-sep} only use that $E$ is a source node in $\G$ with no edge $E \to Y$, and apply unchanged to any DAG over $(X_1, \ldots, X_d, Y, E)$ with these properties, provided $\DEYforb$ and $\SB(Y)$ are computed in that DAG.
\end{remark}

\begin{proof}[Proof of Lem.~\ref{lemma:d-sep}]
We prove both statements simultaneously. Let
$S \in \{\DEYforb^C,\, \SB(Y)\}$ and recall that
$\PA(Y) \subseteq S \subseteq \DEYforb^C$ from \eqref{eq:stable_blanket_decomp}. Consider an arbitrary path from $E$
to $Y$ in $\G$. Since $E$ is a source node and not a parent of $Y$,
the path contains an intermediate node $i$ adjacent to $Y$. We
distinguish two cases.

\paragraph{Case 1. $E \rightarrow ~\cdots~ i \rightarrow Y$.}
Then $i \in \PA(Y) \subseteq S$, so the path is blocked.

\paragraph{Case 2. $E \rightarrow ~\cdots~ i \leftarrow Y$.}
Since $E$ is a source node, the path contains at least one collider.
\begin{itemize}
  \item If $i \in \DEYforb$: let $k$ be the collider closest to
    $Y$ on the path (possibly $k = i$). The segment from $Y$ to $k$ is
    directed $(Y \rightarrow i \rightarrow \cdots \rightarrow k)$, so $k \in \DE(i) \cup \{i\}$. Since
    $i \in \DEYforb$, and because descendants of nodes in
    $\DEYforb$ are again in $\DEYforb$, we have
    $\DE(k) \cap \DEYforb^C = \varnothing$. Hence neither $k$ nor any
    descendant of $k$ lies in $S \subseteq \DEYforb^C$, and the path
    is blocked.
  \item If $i \in \DEYforb^C$: then
    $i \in \CH(Y) \cap \DEYforb^C \subseteq S$. If $i$ is not a
    collider, the path is blocked because $i \in S$. If $i$ is a
    collider, the path has the form
    $E \to \cdots j \to i \leftarrow Y$ for some $j$ with $i \neq j \neq E$ (if
    $j = E$, then $i \in \CH(E) \cap \CH(Y) = \CH_\mathrm{int}(Y)$,
    contradicting $i \in \DEYforb^C$). Then
    $j \in \PA(\CH(Y) \cap \DEYforb^C)\setminus\{Y\} \subseteq \SB(Y)$ by
    \eqref{eq:stable_blanket_decomp}, and $j \in \DEYforb^C$ since
    $\DEYforb$ is closed under descendants. Thus $j \in S$, and since
    $j$ is a non-collider on the path, the path is blocked.
\end{itemize}
\end{proof}

\begin{proof}[Proof of Lem.~\ref{lem:stable_region_opt}]
By Lem.~\ref{lemma:d-sep} and the global Markov property,
$\DEYforb^C$ and $\SB(Y)$ are invariant subsets. In particular, for
all $e \in \Env$, the conditional mean functions representing $\E_e[Y \given X_{\DEYforb^C}]$ and
$\E_e[Y \given X_{\SB(Y)}]$ do not depend on~$e$. By definition of
$\SB(Y)$ \citep[Def.~3.4]{pfisterStabilizingVariableSelection2021}, for all $j \in \DEYforb^C \setminus \SB(Y)$ the
d-separation $j \dsep{\G} Y \given \SB(Y)$ holds. Thus,
$ \big(\DEYforb^C \setminus \SB(Y)\big) \dsep{\G} Y \given \SB(Y)$, so by the global Markov property,
$$
  X_{\DEYforb^C \setminus \SB(Y)} \indep Y \given X_{\SB(Y)}
$$
holds under $\P_e$ for all $e \in \Env$.
Doob's conditional independence property
\citep[Thm.~8.9]{kallenbergFoundationsModernProbability2021} then
yields that almost surely,
\begin{equation} \label{eq:Doob}
  \E[Y \given X_{\DEYforb^C}]
   = 
  \E[Y \given X_{\SB(Y)}, X_{\DEYforb^C \setminus \SB(Y)}]
   = 
  \E[Y \given X_{\SB(Y)}]
   = 
  f_{\SB(Y)}(X).
\end{equation}

Let $e \in \Env$ and let $f\colon \mathcal{X} \to \mathcal{A}$ be
$\sigma(X_{\DEYforb^C})$-measurable. Applying
Lem.~\ref{lem:bayes_restricted} under $\P_e$ with $W = X_{\DEYforb^C}$
and $Z = f(X)$ gives
$$
  R_e(f)
   \geq 
  \E_e \bigl[\ell(Y,  \E[Y \given X_{\DEYforb^C}])\bigr]
   \stackrel{\eqref{eq:Doob}}{=} 
  R_e(f_{\SB(Y)}).
$$
Since $e$ and $f$ were arbitrary, we obtain the result.
\end{proof}

\subsection{Proof of \texorpdfstring{\Cref{thm:PA_vs_SB_game}}{Theorem~\ref{thm:PA_vs_SB_game}}}
\begin{proof}[Proof of \Cref{thm:PA_vs_SB_game}]
For (i), $\Env^*(f) = \argmax_{e \in \Env} R_e(f)$, so $\sup_{e \in \Env^*(f)} R_e(f) = \sup_{e \in \Env} R_e(f)$ for every $f$. By Lem.~\ref{lem:stable_region_opt}, $R_e(f_{\SB(Y)}) \leq R_e(f_{\PA(Y)})$ for all $e \in \Env$, and taking suprema gives the claim.

For (ii), since $\PA(Y)$ and $\SB(Y)$ are invariant, the tower property gives $\E_e[f_{\PA(Y)}(X)] = \E_e[f_{\SB(Y)}(X)] = \E_e[Y]$ for all $e \in \Env$. Hence $\Vfoll_e(f_{\PA(Y)}) = \Vfoll_e(f_{\SB(Y)})$ for all $e$, which implies $\Env^*(f_{\PA(Y)}) = \Env^*(f_{\SB(Y)})$. Therefore
\[
    \sup_{e \in \Env^*(f_{\SB(Y)})} R_e(f_{\SB(Y)})  \leq  \sup_{e \in \Env^*(f_{\PA(Y)})} R_e(f_{\PA(Y)}),
\]
where the inequality uses Lem.~\ref{lem:stable_region_opt}.
\end{proof}

\subsection{Proof of Lem.~\ref{lem:augmentation}} \label{app:proof_augmentation}
\begin{proof}[Proof of Lem.~\ref{lem:augmentation}]
Since $A_j \subseteq \DEforbg{\G_0}^C \setminus \{Y, E\}$ for all $j$, we have $\CH_{\G}(Y) = \CH_{\G_0}(Y)$, $\CH_{\G}(E) = \CH_{\G_0}(E)$, and $\CH_\mathrm{int}(Y)$ is the same in $\G$ and $\G_0$.

First, we show $\DEforbg{\G} = \DEforbg{\G_0}$. The inclusion $\DEforbg{\G_0} \subseteq \DEforbg{\G}$ holds because adding edges can only enlarge the set of descendants. For the reverse, assume by contradiction that there exists $i \in \DEforbg{\G}$ with $i \notin \DEforbg{\G_0}$. Then, there exists a directed path in $\G$ from some $j \in \CH_\mathrm{int}(Y)$ to $i$ that uses at least one edge only existing in $\G$ (since this path does not exist in $\G_0$). Let $k \to l$ be the first such edge on the path (with $k \in A_l \subseteq \DEforbg{\G_0}^C$). The segment $j \to \cdots \to k$ uses only edges of $\G_0$, so $k \in \DE_{\G_0}(\CH_\mathrm{int}(Y)) \cup \CH_\mathrm{int}(Y) = \DEforbg{\G_0}$. This contradicts $k \in A_l \subseteq \DEforbg{\G_0}^C$.

Second, we show $\SBg{\G} = \SBg{\G_0}$.
By \eqref{eq:stable_blanket_decomp}, the stable blanket $\SBg{\G_0}$ is the union of $\PA_{\G_0}(Y)$, $\CH_{\G_0}(Y) \setminus \DEforbg{\G_0}$, and $\PA_{\G_0}(\CH_{\G_0}(Y) \setminus \DEforbg{\G_0})\setminus\{Y\}$.
Regarding the first set, $\PA_{\G}(Y) = \PA_{\G_0}(Y)$ because no intervention on $Y$ takes place. The second is unchanged in $\G$ since $\DEforbg{\G} = \DEforbg{\G_0}$ and $\CH_{\G}(Y) = \CH_{\G_0}(Y)$. The third is unchanged in $\G$ since every new edge points into $\CH_{\G_0}(E)$, and $\CH_{\G_0}(E) \cap (\CH_{\G_0}(Y) \setminus \DEforbg{\G_0}) = \varnothing$ since $\CH_{\G_0}(E) \cap \CH_{\G_0}(Y) = \CH_\mathrm{int}(Y) \subseteq \DEforbg{\G_0}$. Hence, no new edge enters $\CH_{\G_0}(Y) \setminus \DEforbg{\G_0}$, and $\PA_{\G}(\CH_{\G}(Y) \setminus \DEforbg{\G}) = \PA_{\G_0}(\CH_{\G_0}(Y) \setminus \DEforbg{\G_0})$.

Finally, $Y \dsep{\G} E \given \SBg{\G_0}$ follows since $\SBg{\G} = \SBg{\G_0}$ and by Lem.~\ref{lemma:d-sep} and Rmk.~\ref{rem:SB_in_different_DAGs}, $Y \dsep{\G} E \given \SBg{\G}$.
\end{proof}

\subsection{Proof of Thm.~\ref{thm:worst_case}}

We first state and prove a lemma. 

\begin{lemma}[Asm.~\ref{as:A2} implies
Asm.~\ref{as:A1}]
\label{lem:A2_implies_A1}
If Asm.~\ref{as:A2} holds, then for every $e \in \Env$ the
environment $e'$ provided by Asm.~\ref{as:A2} satisfies both
conditions of Asm.~\ref{as:A1}.
\end{lemma}
\begin{proof}[Proof of Lem.~\ref{lem:A2_implies_A1}]
Fix an arbitrary $e \in \Env$ and let $e' \in \Env$ be the corresponding environment provided by
Asm.~\ref{as:A2}, in which for all
$j \in \CH_\mathrm{int}(Y)$, we set $X_j \coloneq \Tilde{\varepsilon}_j$ with $\Tilde{\varepsilon}_j$ independent of all other exogenous noises, removing incoming edges to $\CH_\mathrm{int}(Y)$ in $\G$.
All other structural assignments are unchanged.
Let $\G'$ denote the resulting graph.
Unless stated otherwise, all graph-theoretic notation such as $\PA(\cdot)$ refers to the original graph $\G$.
Quantities defined with respect to $\G'$ are marked explicitly.

We first verify that Asm.~\ref{as:A1}~(i) holds.
Since $\DEYforb$ is closed under descendants, no node in
$\SB(Y) \subseteq \DEYforb^C$ is a descendant of any node in
$\CH_\mathrm{int}(Y)$. Because only the structural assignments of
$\CH_\mathrm{int}(Y)$ are modified in $e'$, the marginal distribution of
$X_{\SB(Y)}$ is the same under $\P_e$ and $\P_{e'}$.

To verify Asm.~\ref{as:A1}~(ii), we need to verify $Y \indep X_{\DEYforb} \given X_{\DEYforb^C}$ under $\P_{e'}$. Since $\P_{e'}$ is generated by the
modified SCM, it is Markov with respect to $\G'$. It therefore suffices
to show that $Y \dsep{\G'} \DEYforb \given \DEYforb^C$. Consider an arbitrary path in $\G'$ from $Y$ to some $j \in \DEYforb$, and let
$i$ be the node adjacent to $Y$ on this path.
\begin{itemize}
    \item If $Y \leftarrow i$, then $i \in \PA(Y) \subseteq \DEYforb^C$, so the path is blocked by $\DEYforb^C$.
    \item If $Y \to i$ with $i \in \CH(Y) \setminus \DEYforb^C$, then
    $i \in \DEYforb$, and
    condition~\eqref{eq:star} gives
    $i \in \CH_\mathrm{int}(Y)$. But in $\G'$ the edge $Y \to i$ is
    deleted, so this path does not exist.
    \item If $Y \to i$ with $i \in \CH(Y) \cap \DEYforb^C$, then if $i$ is a
    non-collider, the path is blocked since $i \in \DEYforb^C$. If
    $i$ is a collider, the path has the form
    $Y \to i \leftarrow k \cdots $ for some
    $k \in \PA_{\G'}(i) \setminus \{Y\}$. Since $i \in \CH(Y) \cap \DEYforb^C$, the node $i$ is not intervened on, hence $\PA_{\G'}(i) = \PA(i)$. Therefore,  
    $k \in \PA(\CH(Y) \cap \DEYforb^C)\setminus\{Y\} \subseteq \SB(Y)
    \subseteq \DEYforb^C$
    by~\eqref{eq:stable_blanket_decomp}. Since $k$ is a non-collider
    on the path, the path is blocked.
\end{itemize}
Since every path is blocked,
$Y \dsep{\G'} \DEYforb \given \DEYforb^C$, and the global Markov
property gives
$Y \indep X_{\DEYforb} \given X_{\DEYforb^C}$ under $\P_{e'}$.
\end{proof}

\begin{proof}[Proof of \Cref{thm:worst_case}]
By Lem.~\ref{lem:A2_implies_A1}, Asm.~\ref{as:A2} implies
Asm.~\ref{as:A1}, so it suffices to prove the claim under
Asm.~\ref{as:A1}. Let $f\colon \mathcal{X} \to \mathcal{A}$ be
an arbitrary measurable predictor. Fix $e \in \Env$ and let
$e' \in \Env$ be the corresponding environment provided by Asm.~\ref{as:A1}.
Applying Lem.~\ref{lem:bayes_restricted} under $\P_{e'}$ with $W = X$
and $Z = f(X)$,
\begin{equation} \label{eq:proper_step}
  R_{e'}(f)
   \geq 
  \E_{e'} \bigl[\ell(Y,  \E_{e'}[Y \given X])\bigr].
\end{equation}
By Asm.~\ref{as:A1}~(ii),
$Y \indep X_{\DEYforb} \given X_{\DEYforb^C}$ under $\P_{e'}$, so by Doob's conditional independence property
\citep[Thm.~8.9]{kallenbergFoundationsModernProbability2021},
$\E_{e'}[Y \given X] = \E_{e'}[Y \given X_{\DEYforb^C}]$ almost surely. Since
$\DEYforb^C$ is invariant by Lem.~\ref{lemma:d-sep} and the global Markov property,
$\E_{e'}[Y \given X_{\DEYforb^C}] = \E[Y \given X_{\DEYforb^C}]$, and
\eqref{eq:Doob} from the proof of Lem.~\ref{lem:stable_region_opt} gives $\E[Y \given X_{\DEYforb^C}] = f_{\SB(Y)}(X)$ almost surely.
Substituting into \eqref{eq:proper_step} yields
$$
  R_{e'}(f)  \geq  R_{e'}(f_{\SB(Y)}).
$$
It remains to show $R_{e'}(f_{\SB(Y)}) = R_e(f_{\SB(Y)})$. Since the marginal distribution of $X_{\SB(Y)}$ and the conditional distribution of $Y \given X_{\SB(Y)}$ are the same under $\P_e$ and $\P_{e'}$ by Asm.~\ref{as:A1}~(i) and the invariance of $\SB(Y)$, the distribution of $(Y, X_{\SB(Y)})$ is the same under $\P_{e'}$ and $\P_e$. Since $\ell(Y, f_{\SB(Y)}(X))$ is a measurable function of
$(Y, X_{\SB(Y)})$, the risks are equal.

Consequently, $\sup_{e'' \in \Env} R_{e''}(f) \geq R_{e'}(f) \geq R_e(f_{\SB(Y)})$.
Since $e$ was arbitrary, taking the supremum over $e$ on the
right-hand side yields \eqref{eq:worst_case_statement}. 
\end{proof}

\subsection{Proof of Prop.~\ref{prop:asymptotic}}

\begin{proof}[Proof of Prop.~\ref{prop:asymptotic}]
Fix an arbitrary predictor $f \colon \mathcal{X} \to \mathcal{A}$, $e \in \Env$, and let $(e'_j)_{j \geq 1}$ be the associated sequence from the assumption.
For all $j$, by Lem.~\ref{lem:bayes_restricted} applied under $\P_{e'_j}$ with
$W = X$ and $Z = f(X)$,
$$ R_{e'_j}(f) \geq R_{e'_j}(\mu_{e'_j}). $$
Since $\SB(Y)$ is invariant, the conditional distribution of
$Y \given X_{\SB(Y)}$ is the same in every environment. Together
with condition~(i), the joint distribution of $(Y, X_{\SB(Y)})$ is the
same under $\P_e$ and $\P_{e'_j}$. Because
$\ell(Y, f_{\SB(Y)}(X))$ is a measurable function of
$(Y, X_{\SB(Y)})$, it follows that for all $j$,
$$ R_{e'_j} \bigl(f_{\SB(Y)}\bigr) =
  R_e \bigl(f_{\SB(Y)}\bigr). $$
By condition~(ii), for all $\varepsilon>0$ there exists $N$ such that for all $j \geq N$, $R_{e'_j}\bigl(f_{\SB(Y)}\bigr)-R_{e'_j}(\mu_{e'_j})< \varepsilon$. Consequently,
\[
  R_{e'_j}(f)
  \geq
  R_e\bigl(f_{\SB(Y)}\bigr)
  -
  \Bigl(
    R_{e'_j}\bigl(f_{\SB(Y)}\bigr)-R_{e'_j}(\mu_{e'_j})
  \Bigr)
  \geq R_e\bigl(f_{\SB(Y)}\bigr) - \varepsilon.
\]
Therefore, for all $\varepsilon>0$, we have $\sup_{e'' \in \Env} R_{e''}(f) \geq R_e\bigl(f_{\SB(Y)}\bigr) - \varepsilon$. Letting $\varepsilon \to 0$ and taking the supremum over $e \in \Env$ ($e$ was arbitrary) yields \eqref{eq:worst_case_statement}.
\end{proof}

\subsection{Details for Ex.~\ref{ex:star_counterexample}}
\label{app:proof_ex:star_counterexample}

\begin{proof}[Details for Ex.~\ref{ex:star_counterexample}]
In this example, $X_1 \in \CH(Y) \cap \DE(\CH_\mathrm{int}(Y)) \setminus \CH_\mathrm{int}(Y)$, so condition~\eqref{eq:star} fails. Since $\SB(Y) = \varnothing$ and $Y \sim \mathrm{Ber}(1/2)$, the stable blanket classifier is constant: $f_{\SB(Y)}= \P(Y=1) = 1/2$. Consequently, for all $e \in \Env$, $R_e(f_{\SB(Y)}) = \E_e[(Y - 1/2)^2] = 1/4$.

Define $f\colon \{0,1\}^2 \to [0,1]$ by $f(0,0) = 0$, $f(1,0) = 1$, $f(0,1) = p$, $f(1,1) = 1-p$. Fix an arbitrary $e \in \Env$. On the event $\{X_2 = 0\}$, we have $X_1 = Y$, so $(Y - f(X_1, 0))^2 = 0$. On the event $\{X_2 = 1\}$, $X_1 = Y \oplus \varepsilon_1 $ with $\varepsilon_1  \sim \mathrm{Ber}(p)$ independent of $(Y, X_2, E)$.
With probability $1-p$, $\varepsilon_1  = 0$, so $X_1 = Y$ and $(Y - f(Y, 1))^2 = p^2$. With probability $p$, $\varepsilon_1  = 1$, so $X_1 = 1-Y$ and $(Y - f(1-Y, 1))^2 = (1-p)^2$. Hence $\E[(Y - f(X_1, 1))^2 \given X_2 = 1] = (1-p)\, p^2 + p\,(1-p)^2 = p(1-p)$, which does not depend on the environment. Combining both cases, for all $e \in \Env$,
$$
  R_e(f) = \E_e \big[(Y - f(X_1, X_2))^2\big] = \P_e(X_2 = 1) \cdot p(1-p) \leq p(1-p).
$$
Since $0 < p < 1/2$, we have $p(1-p) < 1/4 = R_e(f_{\SB(Y)})$.
\end{proof}

\subsection{Proof of Prop.~\ref{prop:star_necessary}}
\label{app:proof_prop:star_necessary}

\begin{proof}[Proof of Prop.~\ref{prop:star_necessary}]
Since \eqref{eq:star} fails, there exists
$i\in\CH(Y)\cap\DE(\CH_\mathrm{int}(Y))\setminus\CH_\mathrm{int}(Y)$. Consequently, $i\notin\CH(E)$, and $i \notin \SB(Y)$ since $i\in\DEYforb$.

Choose nonzero coefficients $b_{j,k}$ for every edge $k\to j$
in $\bar{\G}$, and let
$(\varepsilon_j)_{j\in\{Y,1,\ldots,d\}}$
be jointly independent standard Gaussian random variables.
Define a linear Gaussian SCM with DAG $\bar{\G}$:
\begin{equation}\label{eq:SCM_example}
    Y = \sum_{k\in\PA_{\bar{\G}}(Y)} b_{Y,k}  X_k + \varepsilon_Y,
\qquad
X_j = \sum_{k\in\PA_{\bar{\G}}(j)} b_{j,k}  X_k + \varepsilon_j
\quad (j=1,\ldots,d).
\end{equation}

By the faithfulness result for linear Gaussian SCMs
\citep[Thm.~3.2]{spirtesCausationPredictionSearch2001},
the coefficients can be chosen so that the induced distribution $\P_0$
is faithful to $\bar{\G}$; fix such a choice. 

Extend this to an SCM with DAG $\G$ by adding $E$ (independent of all noises) as a source node, taking three values $\Env = \mathrm{supp}(E) \coloneqq \{e_0, e_0', e_1'\}$, as described in \Cref{sec:problem_setup}.
Since we condition on $E$ to
obtain $\P_e$, its marginal distribution is not relevant.
Under $e_0$, we define all structural assignments to coincide with the SCM~\eqref{eq:SCM_example}, so that $\P_{e_0} = \P_0$ (thus, (i) holds). Under $e_0'$, we replace the structural assignment of every $j\in\CH_\mathrm{int}(Y)$ by $X_j\coloneqq\tilde\varepsilon_j$, where $(\tilde\varepsilon_j)_{j \in \CH_\mathrm{int}(Y)}$ are mutually independent standard Gaussian noises independent of the original noises, and all other structural assignments are unchanged. Consequently, $e' \coloneqq e_0'$ satisfies the requirement in the second part of Asm.~\ref{as:A2} both for $e = e_0$ and for $e = e_0'$ itself (in $e_0'$, every $X_j$, $j\in\CH_\mathrm{int}(Y)$, is already an exogenous noise independent of all other exogenous noises, and all other assignments trivially agree with those of $e_0'$); thus, (ii) holds for these two environments.

Under $e_1'$, use the assignments of $e_0'$, but replace $\varepsilon_j$ by an independent standard Gaussian copy for every $j\in\CH(E)\setminus\CH_\mathrm{int}(Y)$. All added noises are mutually independent and independent of the original noises. This makes the assignments of the remaining children of $E$ depend on $E$ and gives $\P_{e_1'}=\P_{e_0'}$.
For $e=e_1'$, choose $e'=e_1'$, so (ii) also holds.

We now construct a predictor with lower worst-case risk than $f_{\SB(Y)}$.
Since $i\notin\CH(E)$, the SCM has the assignment
$X_i=\sum_{k\in\PA(i)} b_{i,k} X_k+\varepsilon_i$ in every environment. Define
\[
\bar{X}_i\coloneq X_i - \sum_{k\in\PA(i)\setminus\{Y\}} b_{i,k}  X_k,
\]
so that $\bar{X}_i=b_{i,Y}  Y+\varepsilon_i$ in every environment.
We claim that none of $Y$, $E$, or $X_{\SB(Y)}$ can be descendants of $i$ in $\G$.
Since $i\in\CH(Y)$, the node $Y$ is a parent of $i$ and hence cannot
also be a descendant in a DAG.
$E$ is a source node and therefore has no ancestors.
Since $i\in\DEYforb$ and all descendants of nodes in $\DEYforb$ are in $\DEYforb$, we have that no node in $\SB(Y)\subseteq\DEYforb^C$ is a descendant of $i$. Unrolling the structural assignments along a topological ordering
of $\G$, each variable is a function of the noise variables of itself
and its ancestors only. Consequently, since $Y$, $E$, and every node in
$X_{\SB(Y)}$ are not descendants of $i$, none of them depend
on $\varepsilon_i$:
\begin{equation}\label{eq:eps_indep}
\varepsilon_i\indep(Y,  X_{\SB(Y)},  E).
\end{equation}

Under $\P_{e_0}$, the distribution of
$Y\given X_{\SB(Y)} = x_{\SB(Y)}$ is Gaussian with mean $m(x_{\SB(Y)})$ and
constant variance $\sigma^2$.
Moreover, since $\SB(Y)\subseteq\DEYforb^C$ and all descendants of nodes in $\DEYforb$ are in $\DEYforb$, no node in $\SB(Y)$ is a descendant of $\CH_\mathrm{int}(Y)$. As only the assignments of $\CH_\mathrm{int}(Y)$ differ between $e_0$ and $e_0'$, the marginal distribution of $X_{\SB(Y)}$ is the same under $\P_{e_0}$ and $\P_{e_0'}$, and hence also under $\P_{e_1'}$.
Since $\SB(Y)$ is invariant, it holds in every environment that, almost surely,
$$ Y\given X_{\SB(Y)}=x_{\SB(Y)} 
\sim \mathcal{N}\bigl(m(x_{\SB(Y)}), \sigma^2\bigr). $$
Since the SCM on $\bar{\G}$ has independent Gaussian noises with positive variances, the
covariance matrix of $(Y,X_{\SB(Y)})$ under $e_0$ is positive definite. In particular $\sigma^2>0$, since $\sigma^2=0$ would imply that $Y$ is an affine
function of $X_{\SB(Y)}$ almost surely, contradicting the positive definiteness.

The conditional distribution of $(Y, \bar{X}_i)$ given $X_{\SB(Y)}$ is
determined by 
$Y\given X_{\SB(Y)}\sim\mathcal{N}(m,\sigma^2)$ and
$\bar{X}_i\given Y\sim\mathcal{N}(b_{i,Y}Y,1)$, both
environment-independent by invariance of $\SB(Y)$ and \eqref{eq:eps_indep}, respectively. (Indeed, by \eqref{eq:eps_indep}, $\bar{X}_i = b_{i,Y}Y + \varepsilon_i$ is conditionally independent of $(X_{\SB(Y)}, E)$ given $Y$, so the two stated conditionals determine the conditional distribution of $(Y, \bar{X}_i)$ given $X_{\SB(Y)}$ in every environment.)
Hence $Y\given(X_{\SB(Y)}, \bar{X}_i)$ does not depend on $e$ (because of joint Gaussianity), and the following
predictor is well-defined independently of the environment:
$$f(x)\coloneq\E[Y\given X_{\SB(Y)}=x_{\SB(Y)}, \bar{X}_i=\bar{x}_i].$$

Since $(Y, \bar{X}_i)$ is jointly Gaussian conditionally on $X_{\SB(Y)}$,
\[
f(X)
= \E[Y\given X_{\SB(Y)},\bar{X}_i]
= f_{\SB(Y)}(X)
  + \frac{b_{i,Y} \sigma^2}{b_{i,Y}^2 \sigma^2+1}
    \bigl(\bar{X}_i-\E[\bar{X}_i\given X_{\SB(Y)}]\bigr),
\]
where $b_{i,Y} \sigma^2 /(b_{i,Y}^2 \sigma^2+1)$ is nonzero. Using \eqref{eq:eps_indep}, the residual $\bar{X}_i - \E[\bar{X}_i \given X_{\SB(Y)}]$ has variance at least $\Var(\varepsilon_i) = 1$. Hence, $f(X) \neq f_{\SB(Y)}(X)$ with positive probability under
every $\P_e$. Applying Lem.~\ref{lem:bayes_restricted} under $\P_e$ with
$W=(X_{\SB(Y)}, \bar{X}_i)$ and $Z=f_{\SB(Y)}(X)$ yields
$R_e(f_{\SB(Y)})>R_e(f)$ for all $e\in\Env$ (thus, (iii) holds).
\end{proof}

\subsection{Proof of Prop.~\ref{prop:adversarial}}
\label{app:adversarial}

\begin{proof}[Proof of Prop.~\ref{prop:adversarial}]
Write $x = (x_S, x_N)$ with $x_S \in \mathcal{X}_S$ and
$x_N \in \mathcal{X}_N$. Let $f\colon \mathcal{X} \to \mathcal{A}$ be
an arbitrary measurable predictor and $\Q \in \mathcal{Q}_S$. We construct a distribution $\tilde\Q \in \mathcal{Q}_S$ on $\mathcal{X} \times \mathcal{Y}$ by defining it as the product measure
$$
  \tilde\Q  \coloneq  \Q_{(X_S, Y)} \otimes \Q_{X_N},
$$
where $\Q_{(X_S, Y)}$ and $\Q_{X_N}$ denote the marginal distributions
of $(X_S, Y)$ and $X_N$ under $\Q$, respectively. We verify
$\tilde\Q \in \mathcal{Q}_S$:
\begin{enumerate}[label=(\roman*)]
  \item The marginal distribution of $(X_S, Y)$ under $\tilde\Q$ equals $\Q_{(X_S,Y)}$,
        so
        $\tilde\Q(Y\in\cdot\given X_S) = \Q(Y\in\cdot\given X_S) = \P(Y\in\cdot\given X_S)$ almost surely.
  \item Since $\Q_X \ll \mu$ and $\mu = \mu_S \otimes \mu_N$, we have
        $\Q_{X_S} \ll \mu_S$ and $\Q_{X_N} \ll \mu_N$, hence
        $\tilde\Q_X = \Q_{X_S} \otimes \Q_{X_N} \ll \mu$.
\end{enumerate}
By Tonelli's theorem,
$$ \E_{\tilde\Q} [\ell(Y, f(X))] = \int_{\mathcal{X}_N} \left( \int_{\mathcal{X}_S \times \mathcal{Y}} \ell(y, f(x_S, x_N)) \, d\Q_{(X_S, Y)}(x_S, y) \right) d\Q_{X_N}(x_N). $$
For each fixed $x_N \in \mathcal{X}_N$, define $g_{x_N}: \mathcal{X}_S \to \mathcal{A}$ by $g_{x_N}(z) \coloneq f(z, x_N)$. The inner integral can be written as
$$\int_{\mathcal{X}_S \times \mathcal{Y}} \ell(y, f(x_S, x_N)) \, d\Q_{(X_S, Y)}(x_S, y) = \E_{(X_S, Y) \sim \Q} \Big[ \ell\big(Y, g_{x_N}(X_S)\big) \Big].$$
For every fixed $x_N$, $g_{x_N}(X_S)$ is $\sigma(X_S)$-measurable. Applying
Lem.~\ref{lem:bayes_restricted} under $\Q$ with $W = X_S$ and
$Z = g_{x_N}(X_S)$, and using condition~(i),
\begin{equation*}
  \E_{(X_S,Y) \sim \Q} \Big[\ell \big(Y,\, g_{x_N}(X_S)\big)\Big]
   \geq 
  \E_{(X_S,Y) \sim \Q} \Big[
    \ell \big(Y,\, f_S(X_S)\big)
  \Big].
\end{equation*}
The right-hand side does not depend on $x_N$. Integrating over $x_N$ yields
\begin{equation*}
  \E_{\tilde\Q} \big[\ell(Y, f(X))\big]
   \geq 
  \E_\Q \big[\ell(Y, f_S(X_S))\big].
\end{equation*}
Since $\tilde\Q \in \mathcal{Q}_S$,
$$
  \sup_{\Q' \in \mathcal{Q}_S}
    \E_{\Q'} \big[\ell(Y, f(X))\big]
   \geq 
  \E_{\tilde\Q} \big[\ell(Y, f(X))\big]
   \geq 
  \E_\Q \big[\ell(Y, f_S(X_S))\big].
$$
As $\Q \in \mathcal{Q}_S$ was arbitrary, the right-hand side can be
replaced by its supremum over $\mathcal{Q}_S$. Since
$f$ was arbitrary, $f_S$ minimizes the worst-case risk.
\end{proof}

\subsection{Proof of Prop.~\ref{prop:frontier_new}}

\begin{proof}[Proof of Prop.~\ref{prop:frontier_new}]
Let $D(O) \coloneqq O \cup \DE(O)$, so that $N(O) = \{1, \ldots, d\} \setminus D(O)$.

\emph{Step 1: m-separation in $\mathcal{M}_O$ implies m-separation in $\mathcal{M}$.}
No node in $D(O)$ is $Y$ or an ancestor of $Y$, since otherwise some node in $O$ would have a directed path to $Y$.
Since $E$ is a source node, $E \notin D(O)$.
Hence $D(O) \subseteq \{1, \ldots, d\}$, and its complement $N(O) \cup \{Y, E\}$ in $\mathcal{M}$ is closed under taking ancestors because $D(O)$ is closed under taking descendants.
Now let $x, y \in N(O) \cup \{Y, E\}$, let $Z \subseteq N(O) \cup \{Y, E\}$, and let $\pi$ be a path between $x$ and $y$ in $\mathcal{M}$ that is not m-blocked by $Z$; we show that $\pi$ lies in $\mathcal{M}_O$ and is not m-blocked by $Z$ in $\mathcal{M}_O$ either.
Every node on $\pi$ is an endpoint, a node in $Z$, or an ancestor of one of these: a collider on $\pi$ is in $Z$ or has a descendant in $Z$, and a non-collider has an edge pointing away from it on $\pi$, which starts a directed subpath of $\pi$ that ends at a collider or an endpoint \citep[see also][proof of Lem.~3]{richardsonMarkovPropertiesAcyclic2003}.
Hence $\pi$ and every directed path from a collider on $\pi$ to $Z$ lie in $N(O) \cup \{Y, E\}$ and thus in $\mathcal{M}_O$, so $\pi$ is not m-blocked by $Z$ in $\mathcal{M}_O$ either.

\emph{Step 2: $(E, N(O) \setminus S(O)) \msep Y \given S(O)$ in $\mathcal{M}_O$, that is, every path in $\mathcal{M}_O$ between $Y$ and a node in $\{E\} \cup (N(O) \setminus S(O))$ is m-blocked by $S(O)$.}
Let $F$ denote the Markov blanket $\DIS(Y) \cup \PA(\DIS(Y)) \cup \DIS(\CH(Y)) \cup \PA(\DIS(\CH(Y)))$ computed in $\mathcal{M}_O$, so that $S(O) = F \setminus \{Y, E\} \subseteq N(O)$.
We use two observations about paths in $\mathcal{M}_O$ whose intermediate nodes are all colliders.
\begin{itemize}
    \item If such a path connects a node $w \neq Y$ to $Y$, then $w \in F$.
    If $w$ is adjacent to $Y$, it is a parent of $Y$, a child of $Y$, or joined to $Y$ by a bidirected edge, and all such nodes lie in $F$.
    Otherwise, the intermediate nodes are joined by bidirected edges and form a chain in $\DIS(Y)$ or in $\DIS(\CH(Y))$, depending on whether the edge at $Y$ is bidirected or points away from $Y$.
    The edge at $w$ has an arrowhead at the first intermediate node, so $w$ is a member or a parent of that district.
    \item No such path connects $E$ to $Y$.
    Since all edges at $E$ point away from $E$ and $E$ is not adjacent to $Y$, such a path would have one of the forms in \eqref{eq:two_forms_new} in $\mathcal{M}_O$ and hence in $\mathcal{M}$, so it would contain a node of $O$ by the definition of an admissible collider set, but $\mathcal{M}_O$ contains no node of $O$.
\end{itemize}
Now let $v \in \{E\} \cup (N(O) \setminus S(O))$ and let $\pi$ be a path between $v$ and $Y$ in $\mathcal{M}_O$; we show that $\pi$ is m-blocked by $S(O)$.
Let $w$ be the intermediate non-collider on $\pi$ closest to $Y$, or $w \coloneqq v$ if there is none.
The nodes strictly between $w$ and $Y$ on $\pi$ are colliders and $w \neq Y$, so the two observations give $w \in F$ and $w \neq E$, and thus $w \in F \setminus \{Y, E\} = S(O)$.
Since $v \notin S(O)$, we have $w \neq v$, so $w$ is a non-collider on $\pi$ in $S(O)$ and $\pi$ is m-blocked by $S(O)$.

\emph{Step 3: proof of (i).}
Step~2 and Step~1 give $(E, N(O) \setminus S(O)) \msep Y \given S(O)$ in $\mathcal{M}$, since all nodes involved lie in $N(O) \cup \{Y, E\}$.
The global Markov property yields $Y \indep (E, X_{N(O) \setminus S(O)}) \given X_{S(O)}$ under $\P$, and the decomposition and weak union properties \citep[Thm.~8.12]{kallenbergFoundationsModernProbability2021} give
\begin{equation} \label{eq:ci_frontier_new}
    Y \indep E \given X_{S(O)}, \qquad Y \indep X_{N(O) \setminus S(O)} \given X_{S(O)}, \qquad Y \indep E \given X_{N(O)},
\end{equation}
so $S(O)$ and $N(O)$ are invariant.
Moreover, $S(O) \subseteq N(O)$ by definition, and $\PA(Y) \subseteq N(O)$ because a parent of $Y$ in $D(O)$ would be an ancestor of $Y$ in $D(O)$ (which Step~1 excludes).
Hence $\PA(Y)$ is also the set of parents of $Y$ in $\mathcal{M}_O$, where $\PA(Y) \subseteq \PA(\DIS(Y))$, and since $E \notin \PA(Y)$, we have $\PA(Y) \subseteq S(O)$.

\emph{Step 4: proof of (ii).}
By the invariance of $N(O)$, the second independence in \eqref{eq:ci_frontier_new} together with Doob's conditional independence property \citep[Thm.~8.9]{kallenbergFoundationsModernProbability2021}, and the definition of $f_{S(O)}$, we have for all $e \in \Env$, $\P_e$-almost surely,
\begin{equation*}
    \E_e[Y \given X_{N(O)}] = \E[Y \given X_{N(O)}] = \E[Y \given X_{S(O)}] = f_{S(O)}(X).
\end{equation*}
For $e \in \Env$ and $\sigma(X_{N(O)})$-measurable $f \colon \mathcal{X} \to \mathcal{A}$, applying Lem.~\ref{lem:bayes_restricted} under $\P_e$ with $W = X_{N(O)}$ and $Z = f(X)$ yields
\begin{equation*}
    R_e(f) \geq \E_e\big[\ell(Y, \E_e[Y \given X_{N(O)}])\big] = R_e(f_{S(O)}),
\end{equation*}
where the equality uses the display above; see the proof of Lem.~\ref{lem:stable_region_opt} for the same argument.

\emph{Step 5: proof of (iii).}
Assume that $\PA(Y)$ is invariant.
By (i), $f_{\PA(Y)}$ is $\sigma(X_{N(O)})$-measurable, so (ii) gives $R_e(f_{\PA(Y)}) \geq R_e(f_{S(O)})$ for all $e \in \Env$; taking suprema proves \Cref{thm:PA_vs_SB_game} for the follower objective~(i).
For the follower objective~(ii), the tower property and the invariance of $S(O)$ and $\PA(Y)$ give $\E_e[f_{S(O)}(X)] = \E_e[Y] = \E_e[f_{\PA(Y)}(X)]$ for all $e \in \Env$, and the claim follows exactly as in the proof of \Cref{thm:PA_vs_SB_game}.
\end{proof}

\appsection{Additional details for the synthetic experiments}
\label{app:synthetic_details}

We now provide more details on the synthetic experiments described in \cref{sec:synthetic_experiments}. We describe the data-generating process (\cref{app:dgp}), how the leader and the follower optimize their objectives (\cref{sec:sdfds}), an additional analysis showing that predictors that are invariant in
population can be unstable under the follower's intervention at finite sample
sizes, with this instability shrinking as the leader receives more training data (\cref{app:synthetic_regression_error}), implementation details (\cref{app:synthetic_implementation}), and a comparison of the population follower of \Cref{sec:synthetic_experiments} with individual followers as in \Cref{sec:individuals} and \Cref{app:individual_followers_theory} (\cref{app:individual_followers}).

\subsection{Data-generating processes}
\label{app:dgp}

All synthetic experiments 
described in \Cref{sec:synthetic_experiments}
use the training graph $\G_0$ 
in \cref{fig:synthetic_dag}.
The parent set and stable blanket are
\[
    \PA(Y)=\{X_1,X_2\},
    \qquad
    \SB(Y)=\{X_1,X_2,X_3\}.
\]
The all-variable predictor uses $\{X_1,X_2,X_3,X_4,X_5\}$.
The follower can intervene on $X_1$ and $X_4$
by adding perturbations
$E_{\delta_1}$ and $E_{\delta_4}$,
which are functions of some of the random variables (detailed below) satisfying
$\max(|E_{\delta_1}|,|E_{\delta_4}|)
\leq b$, where
$b$ is the intervention
bound.

For the nonlinear SCM, we draw independent standard Gaussian noises and generate
\begin{align*}
    X_2 &= \varepsilon_2,\\
    X_1 &= 0.8\tanh(X_2)+0.6X_2+0.3\varepsilon_1+E_{\delta_1}(X_2,\varepsilon_1),\\
    Y &= \sin(X_1)+0.8X_2+0.5X_1X_2+0.4\varepsilon_Y,\\
    X_3 &= \tanh(1.2Y)+0.45Y^2+0.5\varepsilon_3,\\
    X_4 &= 2.1\tanh(Y)+0.9Y+0.55X_2+0.15YX_2+0.2\varepsilon_4
          +E_{\delta_4}(Y,X_2,\varepsilon_4),\\
    X_5 &= 1.5\tanh(X_4)+0.75X_4+0.35X_4^2+0.2\varepsilon_5.
\end{align*}
For the linear-Gaussian SCM, we use
\begin{align*}
    X_2 &= \varepsilon_2,\\
    X_1 &= 1.1X_2+0.3\varepsilon_1+E_{\delta_1}(X_2,\varepsilon_1),\\
    Y &= X_1+0.8X_2+0.4\varepsilon_Y,\\
    X_3 &= 1.2Y+0.5\varepsilon_3,\\
    X_4 &= 1.4Y+0.55X_2+0.2\varepsilon_4+E_{\delta_4}(Y,X_2,\varepsilon_4),\\
    X_5 &= 1.3X_4+0.2\varepsilon_5.
\end{align*}
In the additional experiment of \Cref{sec:synthetic_experiments} with the dashed edges of \Cref{fig:synthetic_dag}, the $X_4$ perturbation is allowed to
depend on $(Y,X_2,\varepsilon_4,X_1,X_3)$ instead of only
$(Y,X_2,\varepsilon_4)$.

\subsection{Leader's and follower's optimization}
\label{sec:sdfds}

For each subset $S\in\{\PA(Y),\SB(Y),\{1,\ldots,5\}\}$, the leader trains a
predictor $f_S$. In the linear-Gaussian SCM, we use the closed-form population
linear regression function $\mathbb E[Y\given X_S]$. In the nonlinear SCM, we fit
an MLP regressor on the leader's training data 
and evaluate it after the follower's
intervention.

The follower's intervention is parameterized by neural networks for
$E_{\delta_1}$ and $E_{\delta_4}$, with a final $\tanh$ layer enforcing the intervention
bound. We consider two follower objectives:
\begin{align*}
    \text{MSE:} \qquad
    &\max_{E_{\delta_1},E_{\delta_4}} \; \mathbb E\big[(Y-f_S(X_S))^2\big],\\
    \text{prediction:} \qquad
    &\min_{E_{\delta_1},E_{\delta_4}} \; \mathbb E\big[f_S(X_S)\big].
\end{align*}
After optimizing the follower objective, we always report the leader's
deployment MSE
\[
    \mathbb E\big[(Y-f_S(X_S))^2\big]
\]
under the induced intervention distribution.

\subsection{Finite-sample regression error}
\label{app:synthetic_regression_error}
The theory (for example, \Cref{thm:PA_vs_SB_game} and \Cref{thm:worst_case}) 
considers
population predictors such as
$f_S(x_S)=\mathbb E[Y\given X_S=x_S]$. For an invariant set $S$, this conditional
mean is the same in all environments. In the finite-sample nonlinear experiment (see \Cref{app:dgp}),
however, the leader deploys an estimate
\[
    \hat f_S = f_S + \Delta_S,
\]
where $\Delta_S$ is the regression error. Under squared loss, the risk decomposes into population risk and a regression-error term:
\begin{align*}
    R_e(\hat f_S)
    &= \mathbb E_e\big[(Y-\hat f_S(X_S))^2\big]\\
    &= R_e(f_S)
       + \mathbb E_e\big[\Delta_S(X_S)^2\big],
\end{align*}
where the cross term vanishes because 
of $f_S=\mathbb E_e[Y\given X_S]$ and the law of iterated expectations.
The population risk $R_e(f_S)$ can itself vary across environments: In both SCMs presented in \cref{app:dgp}, $Y=m(X_1,X_2)+0.4\varepsilon_Y$ for a function $m$, with independent standard Gaussian noise unaffected by interventions, so $R_e(f_{\PA(Y)})=0.16$ for every $e$; but 
in the nonlinear SCM, $X_3$ is more informative about $Y$ at some values of $\SB(Y) \setminus \{X_3\}$ 
than others, so interventions can change the average conditional variance $R_e(f_{\SB(Y)})=\E_e[\Var(Y\given X_{\SB(Y)})]$.
The regression-error term can
also change across interventions because the follower changes the marginal
distribution of $X_S$. In particular, the follower can move probability mass
toward regions where $|\Delta_S|$ is large.
Larger training samples can reduce regression error and the resulting sensitivity to interventions.
For \Cref{fig:synthetic_main}, we computed both terms of the decomposition using the exact conditional means, available in closed form for $\PA(Y)$ and by numerical integration over $\varepsilon_Y$ for $\SB(Y)$.
The increase of the deployment MSE from $b=0.25$ to $b=1$ is entirely regression error for $\hat f_{\PA(Y)}$, since $R_e(f_{\PA(Y)})$ is constant.
For $\hat f_{\SB(Y)}$ and under both follower objectives, about $80\%$ of it is due to the change of $\E_e[\Var(Y\given X_{\SB(Y)})]$ (the follower lowers $X_1$ for most observations, which moves $Y$ to values at which $X_3$ is less informative about $Y$).

Regression error alone does not explain the instability of the all-variable predictor. The all-variable set contains $X_4$ and $X_5$, which lie in the
descendant region of the intervened child of $Y$. Consequently, the relevant
conditional relation may change under the follower's intervention, and therefore an increase
in the leader's training sample size does not necessarily make the all-variable predictor more stable:
\cref{fig:synthetic_appendix} confirms this behavior. In the train-size
sweep, the deployment MSE of the parent and stable-blanket predictors becomes
flatter as the leader receives more data, whereas the all-variable predictor
remains sensitive to the intervention bound. The same figure also shows that
allowing the $X_4$ perturbation to depend additionally on $X_1$ and $X_3$ leaves the
qualitative comparison unchanged.

\begin{figure}[t]
    \centering
    \includegraphics[width=1\linewidth]{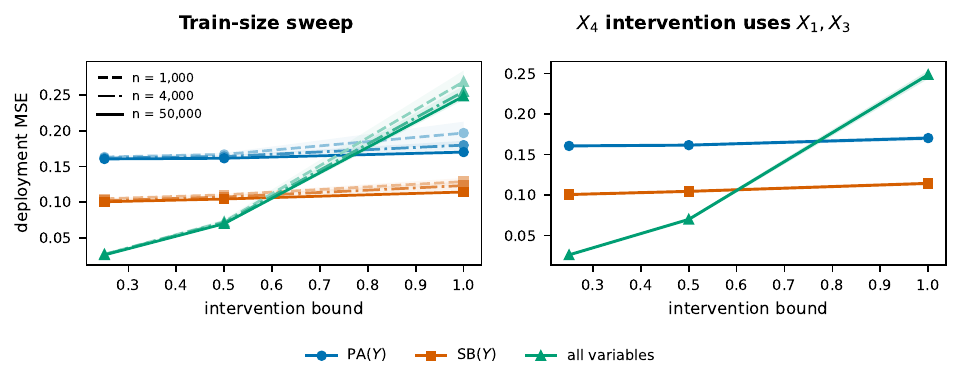}
    \caption{Additional nonlinear SCM results with a follower minimizing the expected prediction.
    Left: train-size sweep. 
Increasing the leader's sample size ($n \in \{1{,}000, 4{,}000, 50{,}000\}$)
    reduces the residual instability of $\PA(Y)$ and $\SB(Y)$, consistent with a decrease in
    finite-sample regression error. Right: allowing the $X_4$ intervention to
    additionally use $X_1$ and $X_3$ does not qualitatively change the results.}
    \label{fig:synthetic_appendix}
\end{figure}

\subsection{Implementation details of the leader's predictor} \label{app:synthetic_implementation}

In the linear-Gaussian SCM, the conditional expectation $\mathbb{E}[Y \given X_S]$ is computed in closed form from the SCM's joint covariance. In the nonlinear SCM, $\hat{f}_S$ is an MLP ($5$ hidden layers of width $256$, ReLU) trained by Adam \citep{kingmaAdamMethodStochastic2015} (learning rate $10^{-3}$, weight decay $10^{-5}$, batch size $512$) on $50{,}000$ training observations, with early stopping determined
on a $3{,}000$-observation validation set (patience $20$, up to $500$ epochs).
 
Each follower perturbation $E_{\delta_j}$ ($j \in \{1, 4\}$) is an MLP ($2$ hidden layers of width $64$, ReLU) followed by $b \cdot \tanh(\cdot)$ to enforce $|E_{\delta_j}| \leq b$; its inputs are as specified in \Cref{app:dgp}. With $\hat{f}_S$ frozen, $E_{\delta_1}$ and $E_{\delta_4}$ are optimized jointly by Adam (learning rate $10^{-3}$) for $2{,}000$ steps. Each step computes the follower's objective on a fresh batch of $512$ observations
drawn from the SCM with the current $E_{\delta_1}, E_{\delta_4}$ and backpropagates through the SCM forward pass. We use $5$ random restarts and pick the parameters with the best objective value on $10{,}000$ held-out observations.

Each curve averages $10$ independent runs with shaded $95\%$ confidence intervals; the reported deployment MSE is computed on $10{,}000$ 
fresh observations under the selected $E_{\delta_1}, E_{\delta_4}$. The experiments are run on an Apple M4 Pro CPU ($14$ cores) using $10$ parallel workers.

\subsection{Comparison with individual followers} \label{app:individual_followers}

\Cref{fig:individual_followers} compares the population follower of \Cref{sec:synthetic_experiments} with individual followers who each choose 
the shift intervention
given their own realization of the observed variables as in \Cref{app:individual_followers_theory}.
\Cref{fig:individual_followers} shows that
the results of \Cref{fig:synthetic_main} remain almost exactly the same when the population follower is replaced by the individual followers, which empirically supports the argument of \Cref{app:individual_followers_theory}.
Implementation details are given in the \texttt{README} of the code repository.\footnote{\url{https://github.com/LinusKuehne/prediction-intervention-games}}

\begin{figure}[t]
    \centering
    \includegraphics[width=\textwidth]{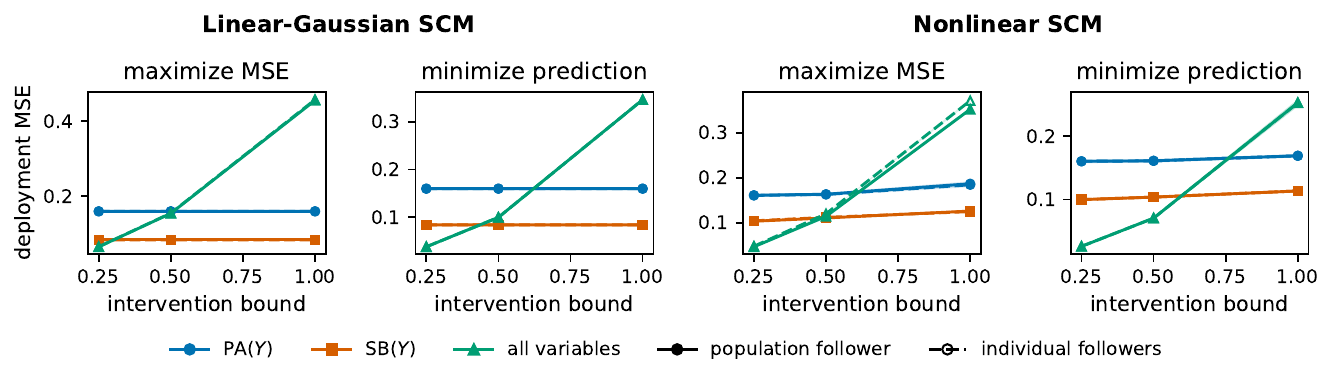}
    \caption{Population follower of \Cref{sec:synthetic_experiments} (solid lines, filled markers) and individual followers of \Cref{app:individual_followers_theory} (dashed lines, open markers).
    For each predictor, both followers are shown; the solid and dashed curves lie almost on top of each other.
    Left two panels: linear-Gaussian SCM with population regression predictors.
    Right two panels: nonlinear SCM with predictors estimated from finite data.
    In each pair, the follower maximizes the leader's MSE (left) or minimizes the expected prediction  (right).
    Curves average $10$ runs, with $95\%$ confidence intervals for the mean.}
    \label{fig:individual_followers}
\end{figure}

\appsection{Additional details for the Causal Chambers experiments} \label{app:additional_details_causal_chambers}

We now provide more details on the experiment described in \cref{sec:new_chamber_experiments}.
We describe the physical device 
and its variables (\Cref{app:causal_chambers_device}), the data-generating process (\Cref{app:causal_chambers_data}), the hidden-confounding hypothesis (\Cref{app:hidden_confounding}), implementation details for predictors deployed by the leader (\Cref{app:implementation_details_leader}), implementation details for the follower (\Cref{app:implementation_details_follower}), and evaluation details (\cref{app:evaluation_causal_chambers}).

\subsection{Causal Chambers: Physical device and variables} \label{app:causal_chambers_device}

We collect data from the Causal Chambers by \citet{gamellaCausalChambersRealworld2025} using the \texttt{causalchamber} API.\footnote{\url{https://github.com/juangamella/causal-chamber-package}} We use the light tunnel \texttt{lt\_mk2\_standard} containing an RGB light source (intensities $\mathtt{red}$, $\mathtt{green}, \mathtt{blue}$), two linear polarizers at adjustable angles ($\mathtt{pol\_1}, \mathtt{pol\_2}$), and three pairs of infrared and visible-light sensors ($\mathtt{ir\_j}$, $\mathtt{vis\_j}$ for $\mathtt{j} \in \{1, 2, 3\}$) at different positions along the tunnel. Sensors at positions $\mathtt{j} = 1, 2$ are unaffected by the polarizers, while sensors at position $\mathtt{j}=3$ are affected by the polarizers through Malus' law.\footnote{The intensity of light passing through both polarizers is proportional to $\cos^2(\mathtt{pol\_1} - \mathtt{pol\_2})$.} An infrared LED placed above the sensors at position $\mathtt{j} = 3$ ($\mathtt{led\_3\_ir}$) activates only during the measurement of $\mathtt{ir\_3}$ and $\mathtt{vis\_3}$, and does not affect the sensors at positions $\mathtt{j} = 1, 2$. \cref{fig:light_tunnel} shows a schematic of the chamber. Tab.~\ref{tab:variables} lists all variables relevant to our experiments. We refer to \citet{gamellaCausalChambersRealworld2025} for the complete documentation.
\begin{figure}[t]
    \centering
    \includegraphics[width=1.0\linewidth]{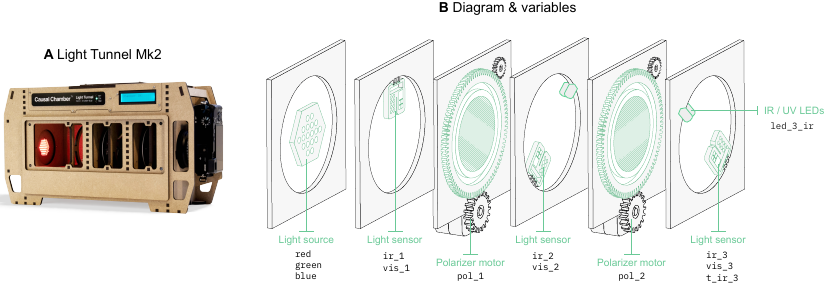}
    \caption{Overview of the Causal Chamber light tunnel (\texttt{lt\_mk2\_standard}). (A) Picture of the hardware. (B) Diagram of the hardware components. The displayed variables are relevant for our experiment. We use different subsets of them for classification tasks. Some are not used as covariates but as intervention targets to generate the different environments. Figure adapted from \citet{gamellaCausalChambersRealworld2025}, licensed under CC BY 4.0.}
    \label{fig:light_tunnel}
\end{figure}
\begin{table}[t]
    \centering
    \caption{Causal Chamber light tunnel variables relevant for our experiments. %
    }
    \label{tab:variables}
    \small
    \begin{tabular}{@{}llp{8.5cm}@{}}
        \toprule
        \textbf{Variable} & \textbf{Range} & \textbf{Description} \\
        \midrule
        $\mathtt{red}$, $\mathtt{green}$, $\mathtt{blue}$ & $\{0, \dots, 255\}$
            & Brightness settings for the main light source LEDs. \\
        $\mathtt{ir\_j}$ ($\mathtt{j} \in \{1,2,3\}$) & $[0, 65535]$
            & Uncalibrated infrared intensity measured by sensor at pos.\ $\mathtt{j}$. \\
        $\mathtt{vis\_j}$ ($\mathtt{j} \in \{1,2,3\}$) & $[0, 65535]$
            & Uncalibrated visible-light intensity measured by sensor at pos.\ $\mathtt{j}$. \\
        $\mathtt{pol\_1}$, $\mathtt{pol\_2}$ & $[-270, 270]$
            & Rotation angles for the first and second polarizer (degrees). \\
        $\mathtt{led\_3\_ir}$ & $\{0, \dots, 4095\}$
            & Brightness of the infrared LED placed above sensor at pos.\ $\mathtt{j}=3$. \\
        \bottomrule
    \end{tabular}
\end{table}
\subsection{Data-generating process} \label{app:causal_chambers_data}

The covariate vector is $X = (\mathtt{red},\, \mathtt{green},\, \mathtt{blue},\,  \mathtt{ir\_2},\, \allowbreak \mathtt{vis\_2},\,  \mathtt{ir\_3},\, \mathtt{vis\_3})$. The light intensities are drawn independently from normal distributions truncated to $[0,255] \subseteq \mathbb{R}$, rounded to integers:
\[
    \mathtt{red} \sim \mathcal{N}_{[0,255]}(64, 20^2),\quad
    \mathtt{green} \sim \mathcal{N}_{[0,255]}(32, 30^2),\quad
    \mathtt{blue} \sim \mathcal{N}_{[0,255]}(90, 12^2).
\]
We define $\mathtt{RGB} \coloneq (\mathtt{red}, \mathtt{green}, \mathtt{blue}) \in \mathbb{R}^3$. Data are generated in two stages.
\begin{enumerate}
    \item Target generation: The $\mathtt{RGB}$ values are sampled. The chamber measures the infrared reading $\mathtt{ir\_1}$, and the target is defined as $Y \coloneq \bm{1}\{\mathtt{ir\_1} > \tau\}$, where $\tau = 12{,}500$.
 
    \item Downstream interventions: Given the value of $Y$, we intervene on $\mathtt{ir\_3}$ and $\mathtt{vis\_3}$
    by setting
    \[
        \mathtt{led\_3\_ir} \coloneq \beta^{\mathrm{led}}_0 + Y \cdot \beta^{\mathrm{led}}_1,
        \qquad
        \mathtt{pol\_2} \coloneq \beta^{\mathrm{pol}}_0 + Y \cdot \beta^{\mathrm{pol}}_1,
    \]
    rounded to integers in $\{0, \dots, 4095\}$ and $\{0, \ldots, 270\}$, respectively. The coefficients $(\beta^{\mathrm{led}}_0, \beta^{\mathrm{led}}_1, \beta^{\mathrm{pol}}_0, \beta^{\mathrm{pol}}_1)$ vary across environments (Tab.~\ref{tab:interventions_d2}), parameterizing the exogenous variable $E$. The chamber then takes a second round of measurements with these updated settings (the exact values of $\mathtt{RGB}$ from the first stage are kept). Since $\mathtt{led\_3\_ir}$ and $\mathtt{pol\_2}$ influence $\mathtt{ir\_3}$ and $\mathtt{vis\_3}$, this creates an environment-dependent association between $Y$ and the downstream sensor readings at $\mathtt{j}=3$, see \cref{fig:causal_chambers_dags}.
\end{enumerate}

Tab.~\ref{tab:interventions_d2} lists the interventions corresponding to the environments. In all three training environments, $\beta^{\mathrm{led}}_0 = \beta^{\mathrm{pol}}_0 = 0$ and $\beta^{\mathrm{led}}_1, \beta^{\mathrm{pol}}_1 \geq 0$, so $Y = 1$ leads to higher (or equal) values of $\mathtt{led\_3\_ir}$ and $\mathtt{pol\_2}$ than $Y = 0$. The effects of $\mathtt{led\_3\_ir}$ (increasing $\mathtt{ir\_3}$) and $\mathtt{pol\_2}$ (decreasing $\mathtt{ir\_3}$ and $\mathtt{vis\_3}$) partially cancel, resulting in a positive correlation between $Y$ and $\mathtt{ir\_3}$, making $\mathtt{ir\_3}$ a highly predictive training feature. A choice of $\beta^{\mathrm{led}}_0, \beta^{\mathrm{pol}}_0 > 0$ and $\beta^{\mathrm{led}}_1, \beta^{\mathrm{pol}}_1 < 0$ by a follower 
could 
reverse the relationship: $Y = 1$ would yield lower values of $\mathtt{led\_3\_ir}$ and $\mathtt{pol\_2}$. A classifier making predictions based on $\mathtt{ir\_3}$ 
would 
perform badly
in such an environment. The follower's intervention class here is more restricted than in
the framework of \Cref{sec:problem_setup}. This makes the experiment a
realistic test of our framework, since in practice followers are typically
limited to a particular set of actions rather than able to intervene
arbitrarily.

\begin{table}[t]
    \centering
    \caption{%
      Interventions in different environments. The coefficients $(\beta^{\mathrm{led}}_0, \beta^{\mathrm{led}}_1, \beta^{\mathrm{pol}}_0, \beta^{\mathrm{pol}}_1)$ control the environment-dependent effect of $Y$ on $\mathtt{led\_3\_ir}$ and $\mathtt{pol\_2}$.%
    }
    \label{tab:interventions_d2}
    \small
    \begin{tabular}{@{}clcccc@{}}
        \toprule
        & $E$ & $\beta^{\mathrm{led}}_0$ & $\beta^{\mathrm{led}}_1$
            & $\beta^{\mathrm{pol}}_0$ & $\beta^{\mathrm{pol}}_1$ \\
        \midrule
        & 0  & 0 & 0 & 0 & 0 \\
        & 1  & 0 & 20 & 0 & 50 \\
        & 2  & 0 & 8 & 0 & 20 \\
        \bottomrule
    \end{tabular}
\end{table}

\subsection{Hidden confounding} \label{app:hidden_confounding}
\citet{gamellaCausalChambersRealworld2025} provide partial causal-graph
information for the light tunnel, where an edge $Z_1 \to Z_2$ means that
intervening on $Z_1$ changes the distribution of subsequent measurements of $Z_2$.
They note that the absence of an edge in their graphs does not preclude a
causal effect or confounding by unmeasured variables.
Under the DAG in \cref{fig:causal_chambers_dags} (middle), which we inferred from their provided information, we would expect conditional independencies of the form $\mathtt{ir\_i} \indep \mathtt{ir\_j} \given \mathtt{RGB}$ and $\mathtt{vis\_i} \indep \mathtt{vis\_j} \given \mathtt{RGB}$ to hold by the global Markov property. Tab.~\ref{tab:hidden_conf} reports $p$-values for the corresponding null hypotheses in the training environments described in \Cref{app:causal_chambers_data}, computed with the weighted GCM test \citep{scheideggerWeightedGeneralisedCovariance2022}. The conditional independencies are rejected in the majority of cases.
We hypothesize that this is because of hidden confounding (such as background lighting) and believe the graph in \cref{fig:causal_chambers_dags} (right) is a more accurate description. Indeed, this is in correspondence with the results presented in \cref{sec:new_chamber_experiments}.
\begin{table}[t]
    \centering
    \caption{%
      $p$-values of a conditional independence test (Weighted GCM test \citep{scheideggerWeightedGeneralisedCovariance2022}) for the null hypothesis $Z_1 \indep Z_2 \given \mathtt{RGB}$. Conditional independence is rejected in the majority of cases, and in particular when pooling the data from all training environments. %
    }
    \label{tab:hidden_conf}
    \small
\begin{tabular}{llrrrr}
\toprule
 & & \multicolumn{4}{c}{$p$-value} \\
\cmidrule(lr){3-6}
$Z_1$ & $Z_2$ & $E=0$ & $E=1$ & $E=2$ & $E \in \{0,1,2\}$ \\
\midrule
$\mathtt{ir\_1}$ & $\mathtt{ir\_2}$ & $<0.001$ & $<0.001$ & $<0.001$ & $<0.001$ \\
$\mathtt{ir\_2}$ & $\mathtt{ir\_3}$ & $<0.001$ & $0.688$ & $<0.001$ & $<0.001$ \\
$\mathtt{vis\_1}$ & $\mathtt{vis\_2}$ & $<0.001$ & $<0.001$ & $<0.001$ & $<0.001$ \\
$\mathtt{vis\_2}$ & $\mathtt{vis\_3}$ & $<0.001$ & $0.081$ & $<0.001$ & $<0.001$ \\
\bottomrule
\end{tabular}
\end{table}

\subsection{Implementation details (leader)} \label{app:implementation_details_leader}

The leader trains five predictors on data pooled from the three training environments $E \in \{0, 1, 2\}$ (Tab.~\ref{tab:interventions_d2}), with $10{,}000$ observations per environment. Four predictors are random forests (\texttt{scikit-learn} \citep{scikit-learn} defaults, $100$ trees) restricted to the covariate subsets indicated by their subscript: $\hat{f}_{\mathtt{RGB}}$, $\hat{f}_{\mathtt{RGB}, \mathtt{ir\_2}}$, $\hat{f}_{\mathtt{RGB}, \mathtt{ir\_2}, \mathtt{vis\_2}}$, and $\hat{f}_\mathrm{all}$. The fifth predictor is the stabilized-classification ensemble $\hat{f}^\mathrm{SC}$ (\Cref{sec:learning_subset_predictors}). All five predictors are held fixed after training.
 
For SC, we use $\alpha_\mathrm{inv} = \alpha_\mathrm{pred} = 0.05$ without tuning, enumerate all $2^d$ ($d=7$) feature subsets, and apply no screening. Predictiveness is scored by out-of-sample binary cross-entropy on the pooled training environments, using out-of-bag predictions; the predictive cutoff is set by $250$ bootstrap samples. Invariance is tested with the \textsc{tram}-GCM test \citep{kookModelbasedCausalFeature2023}, using a Python re-implementation of the non-parametric \texttt{rangerICP} variant ported from the \texttt{tramicp} package \citep{tramicp}.
The base classifiers inside SC
use $100$ trees, with \texttt{scikit-learn} defaults.
The invariance test uses random forests with $500$ trees, a minimum leaf size of $5$, and $\lfloor\sqrt{|S|}\rfloor$ candidate features per split for nonempty subsets $S$.

All computations run in minutes on an Apple M4 Pro CPU ($14$ cores) using $8$ parallel workers; the dominant cost of the experiment is data collection from the chamber via the API.

\subsection{Implementation details (follower)} \label{app:implementation_details_follower}

For a given deployed predictor $\hat{f}$, the follower picks one environment from a fixed candidate set $\Env$ of $49$ environments. We denote this choice by $e^*(\hat{f})$ to emphasize its dependence on $\hat{f}$; different predictors can lead to different choices. Each environment is specified by a coefficient tuple $(\beta^{\mathrm{led}}_0, \beta^{\mathrm{led}}_1, \beta^{\mathrm{pol}}_0, \beta^{\mathrm{pol}}_1)$, as in \Cref{app:causal_chambers_data}. The reference environment $e_\mathrm{ref}$ has coefficients $(0, 12, 0, 30)$; the magnitudes of its slope coefficients lie between those of the training environments $E = 1$ and $E = 2$ (Tab.~\ref{tab:interventions_d2}). The $49$ environments consist of the environment with coefficients $(0, 0, 0, 0)$, in which both interventions are switched off, and a $6 \times 8$ grid combining six intensity scales with eight non-trivial sign patterns for the (LED, polarizer) pair, to which $e_\mathrm{ref}$ belongs: each component takes one of three patterns: switched off, the training pattern (active for $Y = 1$), or reversed (active for $Y = 0$).
Tab.~\ref{tab:env49} lists all $49$ tuples.
\begin{table}[t]
    \centering
    \caption{The $49$ candidate environments of the follower, ordered by their distance $d$ to $e_\mathrm{ref}$ (rank $1$), defined in \Cref{app:implementation_details_follower}. For an intervention bound $b$, the choice set $\Env_b$ consists of the environments with rank at most $\max\{1, \mathrm{round}(49 b)\}$.}
    \label{tab:env49}
    \small
    \begin{tabular}{@{}rrrrrr@{\hspace{2.5em}}rrrrrr@{}}
        \toprule
        Rank & $\beta^{\mathrm{led}}_0$ & $\beta^{\mathrm{led}}_1$ & $\beta^{\mathrm{pol}}_0$ & $\beta^{\mathrm{pol}}_1$ & $d$ & Rank & $\beta^{\mathrm{led}}_0$ & $\beta^{\mathrm{led}}_1$ & $\beta^{\mathrm{pol}}_0$ & $\beta^{\mathrm{pol}}_1$ & $d$ \\
        \midrule
        $1$ & $0$ & $12$ & $0$ & $30$ & $0.00$ & $26$ & $12$ & $-12$ & $0$ & $30$ & $0.96$ \\
        $2$ & $0$ & $6$ & $0$ & $15$ & $0.25$ & $27$ & $12$ & $-12$ & $30$ & $-30$ & $1.00$ \\
        $3$ & $0$ & $19$ & $0$ & $45$ & $0.28$ & $28$ & $0$ & $12$ & $30$ & $-30$ & $1.00$ \\
        $4$ & $0$ & $0$ & $0$ & $15$ & $0.48$ & $29$ & $0$ & $0$ & $30$ & $-30$ & $1.00$ \\
        $5$ & $0$ & $0$ & $0$ & $30$ & $0.48$ & $30$ & $19$ & $-19$ & $0$ & $0$ & $1.24$ \\
        $6$ & $0$ & $0$ & $0$ & $45$ & $0.48$ & $31$ & $19$ & $-19$ & $0$ & $45$ & $1.24$ \\
        $7$ & $0$ & $0$ & $0$ & $0$ & $0.50$ & $32$ & $19$ & $-19$ & $45$ & $-45$ & $1.25$ \\
        $8$ & $0$ & $6$ & $0$ & $0$ & $0.50$ & $33$ & $0$ & $0$ & $45$ & $-45$ & $1.25$ \\
        $9$ & $0$ & $0$ & $0$ & $60$ & $0.50$ & $34$ & $0$ & $19$ & $45$ & $-45$ & $1.25$ \\
        $10$ & $0$ & $12$ & $0$ & $0$ & $0.50$ & $35$ & $25$ & $-25$ & $0$ & $0$ & $1.48$ \\
        $11$ & $0$ & $19$ & $0$ & $0$ & $0.50$ & $36$ & $25$ & $-25$ & $0$ & $60$ & $1.48$ \\
        $12$ & $0$ & $25$ & $0$ & $60$ & $0.52$ & $37$ & $25$ & $-25$ & $60$ & $-60$ & $1.50$ \\
        $13$ & $0$ & $25$ & $0$ & $0$ & $0.52$ & $38$ & $0$ & $25$ & $60$ & $-60$ & $1.50$ \\
        $14$ & $6$ & $-6$ & $0$ & $0$ & $0.72$ & $39$ & $0$ & $0$ & $60$ & $-60$ & $1.50$ \\
        $15$ & $6$ & $-6$ & $0$ & $15$ & $0.72$ & $40$ & $31$ & $-31$ & $0$ & $0$ & $1.72$ \\
        $16$ & $0$ & $0$ & $0$ & $75$ & $0.75$ & $41$ & $31$ & $-31$ & $0$ & $75$ & $1.72$ \\
        $17$ & $6$ & $-6$ & $15$ & $-15$ & $0.75$ & $42$ & $31$ & $-31$ & $75$ & $-75$ & $1.75$ \\
        $18$ & $0$ & $6$ & $15$ & $-15$ & $0.75$ & $43$ & $0$ & $31$ & $75$ & $-75$ & $1.75$ \\
        $19$ & $0$ & $0$ & $15$ & $-15$ & $0.75$ & $44$ & $0$ & $0$ & $75$ & $-75$ & $1.75$ \\
        $20$ & $0$ & $31$ & $0$ & $75$ & $0.76$ & $45$ & $35$ & $-35$ & $0$ & $0$ & $1.88$ \\
        $21$ & $0$ & $31$ & $0$ & $0$ & $0.76$ & $46$ & $35$ & $-35$ & $0$ & $84$ & $1.88$ \\
        $22$ & $0$ & $0$ & $0$ & $84$ & $0.90$ & $47$ & $0$ & $0$ & $84$ & $-84$ & $1.90$ \\
        $23$ & $0$ & $35$ & $0$ & $84$ & $0.92$ & $48$ & $0$ & $35$ & $84$ & $-84$ & $1.90$ \\
        $24$ & $0$ & $35$ & $0$ & $0$ & $0.92$ & $49$ & $35$ & $-35$ & $84$ & $-84$ & $1.90$ \\
        $25$ & $12$ & $-12$ & $0$ & $0$ & $0.96$ &  &  &  &  &  &  \\
        \bottomrule
    \end{tabular}
\end{table}
 
We order the $49$ environments by their distance $d \coloneqq \max\{|\beta^{\mathrm{led}}_1 - 12|/25,\, |\beta^{\mathrm{pol}}_1 - 30|/60\}$ to $e_\mathrm{ref}$%
; ties in $d$ are ordered as in the released code. Across the $49$ environments, the intercepts $\beta^{\mathrm{led}}_0$ and $\beta^{\mathrm{pol}}_0$ are functions of the slopes $\beta^{\mathrm{led}}_1$ and $\beta^{\mathrm{pol}}_1$, so we compute $d$ from the slopes only. For an intervention bound $b \in [0, 1]$, the choice set $\Env_b \subseteq \Env$ consists of the $\max\{1, \mathrm{round}(49 b)\}$ environments closest to $e_\mathrm{ref}$ (for $b \in \{0, 0.05, 0.1, 0.25, 0.5\}$, these are $1$, $2$, $5$, $12$, and $24$ environments); in particular, we set $\Env_{b=0} \coloneq \{e_\mathrm{ref}\}$.
 
For each candidate environment $e \in$ $\Env$, we collect $1{,}000$ observations of $X$ from the chamber under $e$ before the follower's selection.
Across the $49$ candidate environments, we reuse the same $1{,}000$ sampled $\mathtt{RGB}$ values and their corresponding target labels $Y$.
We randomly split the shared row indices into $500$ selection and $500$ evaluation observations and reuse this split across predictors and intervention bounds.
The follower estimates the population mean prediction $\E_e[\hat{f}(X)]$ from the $500$ observations reserved for environment selection and selects the environment $e^*(\hat{f}) \in \Env_b$ minimizing this estimate.

\subsection{Evaluation} \label{app:evaluation_causal_chambers}

After the follower's selection, we evaluate on the held-out $500$ observations of $X$ from $e^*(\hat{f})$. For each predictor $\hat{f}$ and each value of $b \in \{0, 0.05, 0.1, 0.25, 0.5\}$, we report on this evaluation set two quantities: (i) the shift in the population mean prediction relative to the reference, $\E_{e^*(\hat{f})}[\hat{f}(X)] - \E_{e_\mathrm{ref}}[\hat{f}(X)]$, and (ii) the leader's MSE under the chosen environment, $R_{e^*(\hat{f})}(\hat{f})$. Both are plotted against $b$ in \Cref{fig:follower_chambers}; for reference, the right panel also shows the leader's MSE under $e_\mathrm{ref}$ as a dotted horizontal line for each predictor.

%% file: dag_3.tex
\begin{tikzpicture}[>=Stealth, semithick, scale=1.05]
    \node (Y) at (0,0.7) {$Y$};
    \node (E) at (2,0.7) {$E$};
    \node (X1) at (-1,0) {$X_1$};
    \node (X2) at (1, 0) {$X_2$};
    
    \draw[->] (E) -- (X2);
    \draw[->] (Y) -- (X2);
    \draw[->] (Y) -- (X1);
    \draw[->] (X2) -- (X1);
\end{tikzpicture}

%% file: dag_2.tex
\begin{tikzpicture}[>=Stealth, semithick, scale=1.0]
    \node (E) at (-2.1, 0.5) {$E$};
    \node (RGB) at (0, 0) {$\mathtt{RGB}$};
    \node (Y) at (2.1, 0.5) {$Y$};
    \node (Bottom) at (0, -0.85) {$(\mathtt{ir\_2}, \mathtt{vis\_2})$};
    \node (Top) at (0, 0.85) {$\mathtt{(ir\_3, vis\_3)}$};

    \draw[->] (E) -- (Top);
    \draw[->] (RGB) -- (Top);
    \draw[->] (RGB) -- (Y);
    \draw[->] (Y) -- (Top);
    \draw[->] (RGB) -- (Bottom);
\end{tikzpicture}

%% file: dag_6.tex
\begin{tikzpicture}[>=Stealth, semithick, scale=1.0]
    \node (E) at (-2.1, 0.5) {$E$};
    \node (RGB) at (0, 0) {$\mathtt{RGB}$};
    \node (Y) at (2.1, 0.5) {$Y$};
    \node (H) at (2, -0.5) {$H$};
    \node (Bottom) at (0, -0.85) {$(\mathtt{ir\_2}, \mathtt{vis\_2})$};
    \node (Top) at (0, 0.85) {$\mathtt{(ir\_3, vis\_3)}$};

    \draw[->] (E) -- (Top);
    \draw[->] (RGB) -- (Top);
    \draw[->] (RGB) -- (Y);
    \draw[->] (Y) -- (Top);
    \draw[->] (RGB) -- (Bottom);
    \draw[->] (H) -- (Top);
    \draw[->] (H) -- (Bottom);
    \draw[->] (H) -- (Y);
\end{tikzpicture}

%% file: dag_7.tex
\begin{tikzpicture}[>=Stealth, semithick, scale=1.05]
    \node (E) at (-1.5,1) {$E$};
    \node (X1) at (0,1) {$X_1$};
    \node (Y) at (1.5,1) {$Y$};
    \node (X2) at (0.75,0) {$X_2$};

    \draw[->] (E) -- (X1);
    \draw[->] (X1) -- (Y);
    \draw[->] (Y) -- (X2);
    \draw[->] (X1) -- (X2);
\end{tikzpicture}

%% file: dag_5.tex
\begin{tikzpicture}[>=Stealth, semithick, scale=1.05]
    \node (X1) at (-1.5,1) {$X_1$};
    \node (Y) at (0,1) {$Y$};
    \node (X2) at (1.5,1) {$X_2$};
    \node (E) at (0,0) {$E$};

    \draw[->] (X1) -- (Y);
    \draw[->] (Y) -- (X2);
    \draw[->] (E) -- (X1);
    \draw[->] (E) -- (X2);
\end{tikzpicture}

%% file: dag_4.tex
\begin{tikzpicture}[>=Stealth, semithick, scale=1.05]
    \node (E) at (-1.5,1) {$E$};
    \node (X1) at (0,1) {$X_1$};
    \node (Y) at (1.5,1) {$Y$};
    \node (X2) at (0,0) {$X_2$};

    \draw[->] (E) -- (X1);
    \draw[->] (E) -- (X2);
    \draw[->] (Y) -- (X1);
\end{tikzpicture}